\documentclass{article}
\usepackage{kr}

\usepackage{times}
\usepackage{soul}
\usepackage{url}
\usepackage[hidelinks]{hyperref}
\usepackage[utf8]{inputenc}
\usepackage[small]{caption}
\usepackage{graphicx}
\usepackage{amsmath}
\usepackage{amsthm}
\usepackage{booktabs}
\usepackage{algorithm}
\usepackage{thm-restate}
\usepackage{algorithmic}
\usepackage{microtype}
\newtheorem{example}{Example}
\newtheorem{theorem}{Theorem}
\newtheorem{proposition}{Proposition}
\newtheorem{lemma}{Lemma}

\newtheorem{definition}{Definition}

\usepackage{thmtools}
\usepackage{thm-restate}
\usepackage{amsfonts}
\usepackage{xspace}
\usepackage{pifont}
\usepackage{color}
\usepackage{amsmath, upgreek, stmaryrd}
\usepackage{amssymb}
\usepackage{misc}   
\usepackage{xargs}
\usepackage[pdftex,dvipsnames]{xcolor}
\usepackage{graphicx}
\usepackage{mathtools}
\usepackage[capitalise]{cleveref}
\usepackage{enumitem}

\usepackage{compat-quentin}
\usepackage{mathdots}
\usepackage{tikz}
\usetikzlibrary{tikzmark}
\usetikzlibrary{calc}
\newcommand{\towerof}[2]{\mn{tower}(#1, #2)}

\newcommand{\sig}{\Sigma}
\newcommand{\const}{\mn{const}}
\newcommand{\sigmin}{\const_=}
\NewDocumentCommand\tp{od<>om}%
{%
	\IfNoValueTF{#1}%
	{%
		\IfNoValueTF{#2}%
		{%
			\IfNoValueTF{#3}%
			{}%
			{\mn{tp}_{#3}(#4)}%
		}%
		{%
			\IfNoValueTF{#3}%
			{}%
			{\mn{tp}^{#2}_{#3}(#4)}%
		}%
	}%
	{%
		\IfNoValueTF{#2}%
		{%
			\IfNoValueTF{#3}%
			{\mn{tp}_{#1}(#4)}%
			{\mn{tp}_{#3, #1}(#4)}%
		}%
		{%
			\IfNoValueTF{#3}%
			{\mn{tp}^{#2}_{#1}(#4)}%
			{\mn{tp}^{#2}_{#3,#1}(#4)}%
		}%
	}%
}

\newlist{properties}{enumerate}{1}
\setlist[properties,1]{label=(\alph*)}

\creflabelformat{propertiesi}{#2#1#3}

\crefname{propertiesi}{Property}{Properties}

\creflabelformat{observationi}{#2#1#3}

\crefname{observationi}{Observation}{Observations}

\title{Adding Circumscription to Decidable Fragments of First-Order
  Logic:\\[1mm] A Complexity Rollercoaster}

\author{%
Carsten Lutz$^1$\and
Quentin Mani\`ere$^{1, 2}$
\affiliations
$^1$Department of Computer Science, Leipzig University, Germany \\
$^2$Center for Scalable Data Analytics and Artificial Intelligence (ScaDS.AI), Dresden/Leipzig, Germany
\emails
\{carsten.lutz,
quentin.maniere\}@uni-leipzig.de
}

\begin{document}

\maketitle


\begin{abstract}
  We study extensions of expressive decidable fragments of first-order
  logic with circumscription, in particular the
  two-variable fragment FO$^2$, its extension C$^2$ with counting
  quantifiers, and the guarded fragment GF. We prove that if only
  unary predicates are minimized (or fixed) during circumscription,
  then decidability of logical consequence is preserved. For FO$^2$
  the complexity increases from \coNExp to $\coNExp^\NPclass$-complete,
  for GF it (remarkably!) increases from 2\Exp to \Tower-complete, and
  for C$^2$ the complexity remains open. We also consider querying circumscribed
  knowledge bases whose ontology is a GF sentence, showing that the
  problem is decidable for unions of conjunctive queries,
  \Tower-complete in combined complexity, and elementary in data
  complexity.  Already for atomic queries and ontologies that are sets
  of guarded existential rules, however, for every $k \geq 0$ there is
  an ontology and query that are $k$-\Exp-hard in data complexity.
\end{abstract}

\section{Introduction}
\label{sec-introduction}

There are various approaches to defining non-monotonic logics such as
default rules, autoepistemic operators, and circumscription. Most of
these are mainly used with propositional logic rather than with
first-order logic (FO), for two reasons. First, many of the
approaches such as default rules can yield non-intuitive results
when used with first-order logics, interacting in unexpected ways with
existential quantification; see for example
\cite{DBLP:journals/jar/BaaderH95} for a discussion of this
issue. And second, the undecidability of first-order logic of course
carries over to its non-monotonic variants.

Description logics (DLs) are decidable fragments of FO for which
non-monotonic variations have been studied extensively, see e.g.\
\cite{DBLP:journals/jar/BaaderH95,DBLP:journals/tocl/DoniniNR02,Bonatti2009,DBLP:journals/ai/GiordanoGOP13,DBLP:journals/ai/BonattiFPS15}. It
turned out that circumscription provides one of the most well-behaved
of such variations: it does not interact in dramatic ways with
existential quantification, has a simple and appealing semantics that
boils down to minimizing the interpretation of certain predicates, and
comes with a clean way to preserve the decidability of the base
logic. The latter is in fact achieved by permitting only unary
predicates to be minimized or fixed during minimization while binary
predicates must be allowed to vary \cite{Bonatti2009}. 
This still covers the main application of circumscription which is
reasoning about typical properties of objects that belong to a certain
class. To model the statement that KR papers are typically interesting, for example,
one may write
$$\mn{KRPaper}(x) \wedge \neg \mn{abKRpaper}(x) \rightarrow
\mn{Interesting}(x)$$ and then minimize the unary `abnormality
predicate' \mn{abKRpaper}. In this way, one may conclude that any
concrete KR paper is interesting unless there is concrete evidence
against that. For more information on DLs with circumscription, we
refer to
\cite{DBLP:conf/birthday/BonattiFLSW14,DBLP:conf/ijcai/Stefano0S23,DBLP:conf/kr/LutzMN23}

It is well-known that DLs are generalized by various decidable and
more expressive FO fragments, of which the two-variable fragment
FO$^2$, the guarded fragment GF, and the extension C$^2$ of FO$^2$
with counting quantifiers are the most important ones. In this paper,
we ask the following questions: \emph{Do expressive decidable
  fragments of FO remain decidable when extended with circumscription
  (when only unary predicates are minimized or fixed)?
  And if so, what is the impact on computational complexity?} The
answers are, in our opinion, somewhat surprising.

We study the reasoning problems of \emph{circumscribed consequence}
and \emph{circumscribed querying}. In the former, two sentences $\phi$
and $\psi$ are given along with a `circumscription pattern' \CP that
specifies which predicates are minimized, fixed, and varying. We are
then interested in deciding whether $\psi$ holds in every model that
is minimal in the sense specified by \CP, written
$\phi \models_\CP \psi$. Circumscribed querying is defined in the same
way, but now $\phi$ is a knowledge base that consists of a sentence
from the FO fragment under consideration (specifying an ontology) and
a database, and $\psi$ is a query. As query languages, we consider
single-atom queries (AQs), conjunctive queries (CQs), and
and unions thereof  (UCQs).

We start with studying FO$^2$. Similarly to the case of description
logic \cite{Bonatti2009}, a crucial step for proving decidability is
to show that circumscribed FO$^2$ has the finite model property (FMP)
in the sense that if $\phi \not\models_\CP \psi$, then there is a
\CP-minimal model \Amf of $\phi$ with $\Amf \not\models \psi$ that is
of bounded size. 
To prove this, we build
on a well-known construction from \cite{GKV}, used there to establish
the FMP of non-circumscribed FO$^2$, which converts a potentially
infinite model \Amf of an FO$^2$ sentence $\phi$ into a model \Bmf of
single exponential size. To apply this construction in the
circumscribed case, however, we need an additional condition to be
satisfied:
\begin{description}

\item[$(\heartsuit)$] \Bmf must not realize any 1-type more often
than~\Amf (for a suitable notion of 1-type).
  
\end{description}
The construction of \cite{GKV} does not satisfy this condition. We
remark that this is in contrast to filtration, the (much simpler)
finite model construction used for description logics such as \ALC.

We thus rework the construction of \cite{GKV} in a suitable way,
obtaining a version that satisfies Condition~$(\heartsuit)$. This
yields the FMP for circumscribed FO$^2$ and decidability as well as a
$\coNExp^\NPclass$ upper complexity bound for circumscribed
consequence. A matching lower bound is obtained from \ALC and thus
circumscribed consequence in FO$^2$ is of the same complexity as in
(the much less expressive) \ALC.  We obtain the same result for the
combined complexity of circumscribed AQ-querying and also show
$\Pi^p_2$-completeness for data complexity, again the same as in
\ALC. Querying with UCQs is undecidable already for non-circumscribed
FO$^2$, so we do not study it.

For GF, we follow the same general approach, with a remarkably
different outcome. There are two constructions that show the finite
model property of GF, both of them rather intricate.  The historically
first one was proposed by Gr\"adel, based on a combinatorial
construction due to Herwig \cite{DBLP:journals/jsyml/Gradel99}.
Later, Rosati introduced a different finite model construction while
studying certain integrity constraints for databases
\cite{DBLP:conf/pods/Rosati06}, and this construction, now known as
the \emph{Rosati cover}, has been adapted to GF in \cite{BGO}. Both
constructions fail to yield Property~$(\heartsuit)$ and modifying them to
achieve this property turns out to be much more difficult than in the
case of FO$^2$. We give a modified version of the Rosati cover
that yields finite models of non-elementary size, compared to single
exponential size as for the original Rosati cover.  This yields the
FMP for circumscribed GF. We then show that the non-elementary size of
finite models is unavoidable: circumscribed consequence in GF is
\Tower-complete! To us, this huge difference to the FO$^2$ case came
as a big surprise. We also show that circumscribed querying in GF is
decidable, generalizing recent work on DLs
\cite{DBLP:conf/kr/LutzMN23}. In combined complexity, it is
\Tower-complete with the lower bound applying to AQs and the upper
bound to UCQs. Regarding data complexity, it is elementary in the
sense that for each GF ontology \Omc, circumscription pattern \CP, and
UCQ $q$, querying is in $k$-\Exp for some $k$. We also show that there
is no uniform bound on $k$: for each $k \geq 1$ we identify an
ontology \Omc, circumscription pattern \CP, and AQ $q$ for which
querying is $k$-\Exp-hard. In fact, \Omc is a set of existential
rules, a `positive' fragment of GF that is important for querying. We
also show that with a single minimized predicate and all other
predicates varying, the data complexity of AQ-querying is \Exp-hard in
GF. Note that since CQs are sandwiched beween AQs and UCQs, this
also completely clarifies the (combined and data) complexity for
this query language.

In addition, we provide first results on circumscribed consequence and
AQ-querying in C$^2$. Using a reduction to Boolean
algebra with Presburger arithmetic, we show that these problems
are decidable. The complexity remains open. 

\section{Preliminaries}
\label{sec-preliminaries}

When speaking of first-order logic (FO), we generally mean the version
with equality and constants (unless otherwise noted) and without
function symbols. FO$^2$ is the two-variable fragment of FO, obtained
by fixing two variables $x$ and $y$ and disallowing the use of any
other variables \cite{Scott1962,DBLP:journals/mlq/Mortimer75,GKV}. C$^2$ is the extension of FO$^2$
with counting quantifiers of the form $\exists_{\leq n}$,
$\exists_{\geq n}$, and $\exists_{=n}$ for every
$n \geq 0$~\cite{DBLP:conf/lics/GradelOR97,DBLP:conf/lics/PacholskiST97,DBLP:journals/jolli/Pratt-Hartmann05}. In FO$^2$ and C$^2$, we generally only
admit predicates of arity at most two. In the guarded fragment of FO,
denoted GF, quantification is restricted to the pattern
$$
  \forall \bar y (\alpha(\bar x,\bar y) \rightarrow \varphi(\bar x,
  \bar y)) \qquad
    \exists \bar y (\alpha(\bar x,\bar y) \wedge \varphi(\bar x,
  \bar y))
$$
where $\varphi(\bar x,\bar y)$ is a GF formula with free variables
among $\bar x,\bar y$ and $\alpha(\bar x,\bar y)$ is a
relational atom $R(\bar x,\bar y)$ or an equality atom $x=y$ that in either
case contains all variables in $\bar x,\bar y$~\cite{DBLP:journals/jphil/AndrekaNB98,DBLP:journals/jsyml/Gradel99}. The
formula $\alpha$ is called the \emph{guard} of the quantified formula.

We use the standard notation of first-order logic, denoting
structures with $\Amf$ and $\Bmf$, their universes with $A$ and
$B$, and the interpretation of predicates $R$ with $R^\Amf$ and
$R^\Bmf$.  We reserve a countably infinite set of predicates of each
arity. We use $|\phi|$ to denote the \emph{length} of the
formula~$\phi$, that is, the length of $\phi$ when encoded as a word
over a suitable alphabet.

\medskip
\noindent
{\bf Circumscription.}
A \emph{circumscription pattern} is a tuple
$\CP = (\prec , \Msf, \Fsf, \Vsf)$, where $\Msf$,
$\Fsf$ and $\Vsf$ partition the unary predicates into \emph{minimized},
\emph{fixed} and \emph{varying} predicates, and $\prec$ is a strict partial
order on $\Msf$ called the \emph{preference relation}.
%
%
The order $\prec$ also induces a preference relation
$<_\CP$ on structures by setting $\structure{B} <_\CP \structure{A}$ if the
following conditions hold:
\begin{enumerate}
	\item
	$\domain{B} = \domain{A}$ and $c^\Amf = c^\Bmf$
          for all constants $c$,
	\item
	for all $P \in \fixedpredicates$, $P^\Bmf = P^\Amf$,
	\item
	for all $P \in \minimizedpredicates$ with $P^\Bmf \not \subseteq P^\Amf$, there
	is a $Q \in \minimizedpredicates$,
	$Q \prec P$, such that $Q^\Bmf \subsetneq Q^\Amf$,
	\item
	there exists a $P \in \minimizedpredicates$ such that $P^\Bmf \subsetneq P^\Amf$
	and for all $Q \in \Msf$, $Q \prec P$ implies $Q^\Bmf = Q^\Amf$.
\end{enumerate}
A \emph{\CP-minimal model} of an FO sentence $\phi$ is a model \Amf of
$\phi$ such that there is no $\structure{B} <_\CP \structure{A}$ that
is a model of $\phi$. Note that predicates of arity larger than
one always vary to avoid undecidability~\cite{Bonatti2009}. We also
assume
that nullary predicates always vary, which is w.l.o.g.\ as they can
be simulated by unary predicates.

For FO sentences $\phi$ and $\psi$, we write
$\phi \models_\CP \psi$ if every \CP-minimal model \Amf of $\phi$
satisfies $\Amf \models \psi$. Take any fragment $F$ of FO such as
FO$^2$. With \emph{circumscribed consequence} in $F$ we mean the
problem to decide, given sentences $\phi$ and $\psi$ from $F$ and a
circumscription pattern $\CP$, whether $\phi \models_\CP \psi$.

	

\medskip
\noindent
{\bf Ontology-mediated querying.}
Ontology-mediated querying with
circumscribed knowledge bases, as recently studied in \cite{DBLP:conf/kr/LutzMN23}, can
be seen as a version of circumscribed consequence where $\phi$ encodes
an ontology and a database and $\psi$ is a query. We next make this
precise.

A \emph{database} is a finite set of ground atoms, in this context
called \emph{facts}. We use $\mn{adom}(D)$ to denote the set of
constants that occur in $D$. A structure \Amf \emph{satisfies} a
database $D$ if (1) it satisfies all facts in it and (2)~interprets
all constant symbols $c$ in $\mn{adom}(D)$ as $c^\Amf = c$
(and thus no two such $c$ denote the same element of $A$).
%
We then also say that \Amf is a \emph{model} of $D$ and write
$\Amf \models D$.  Note that Point~(2) is the \emph{standard names
  assumption}, as usually made in the context of databases.
A
\emph{knowledge base (KB)} \Kmc takes the form
$\bigwedge \Omc \wedge D$ with $\Omc$ a finite set of FO sentences,
called the \emph{ontology}, and $D$ a database. We usually denote \Kmc
as a pair $(\Omc,D)$. We call \Kmc a GF-KB if all sentences in \Omc
fall into GF, and likewise for other FO fragments. 

A \emph{conjunctive query (CQ)} is an FO formula
of the form
$q = \exists \bar x \, \varphi(\bar x)$ where $\varphi$
is a
conjunction of relational atoms, possibly involving constants.
%
%
%
%
An \emph{atomic query (AQ)} is a CQ of the simple form $R(\bar
c)$ with $\bar
c$ a tuple of constants.  A {\em union of conjunctive queries (UCQ)}
$q(\bar
x)$ is a disjunction of CQs. 
Let $\Kmc$ be a KB and $q$ a UCQ. We write $\Kmc \models_\CP
q$ if $\Amf \models q$ for every \CP-minimal model \Amf of
$\Kmc$. The notion $\Kmc \models
q$ is defined analogously, except that all models of \Kmc are
considered, not only \CP-minimal ones. Take a fragment
$F$ of FO such as GF and a query language
$Q$ such as UCQs.  With \emph{circumscribed $Q$-querying in
  $F$}, we mean the problem to decide, given a knowledge base
$\Kmc=(\Omc,D)$ with \Omc a set of sentences from $F$ and a query
$q$ from $Q$, whether $\Kmc \models_\CP
q$.  When studying the combined complexity of this problem, all of
\Kmc, \CP, and
$q$ are treated as inputs. For data complexity, we assume
$\Omc$, \CP, and
$q$ to be fixed and thus of constant size. We remark that our queries
are Boolean, that is, they do not have answer variables. This is
without loss of generality since constants are admitted in queries.

We shall
also consider ontologies~\Omc that are sets of guarded existential
rules. An \emph{existential rule} is an FO sentence of the form
$$
\forall \bar x \forall \bar y \, (\phi(\bar x,\bar y) \rightarrow
\exists \bar z \, \psi(\bar x,\bar z))
$$
where $\phi$ and
$\psi$ are conjunctions of relational atoms. We call
$\phi$ the \emph{body} of the rule and
$\psi$ the \emph{head}.  The rule is \emph{guarded} if the body
contains an atom that contains all variables in it. When writing
existential rules, we usually omit the universal quantifiers. 
For every ontology \Omc that is a set of guarded existential rules, 
there is a GF ontology $\Omc'$ such that for all
databases~$D$ and UCQs $q$, we have $(\Omc,D) \models
q$ iff $(\Omc',D) \models
q$ \cite{DBLP:conf/pods/CaliGL09}. To construct
$\Omc'$, one simply adds a fresh predicate to the head of each rule in
\Omc that contains all variables in the head, and then translates the
resulting set of rules into an equivalent GF sentence in a
straightforward way. This proof also applies to circumscribed
querying, letting the fresh predicates vary.

\begin{example}
  \label{ex:first}
  Consider the database
  $$
  \begin{array}{r@{\;}c@{\;}l}
    D&=&\{ W(w_1), W(w_2), E(w_1), E(w_2), \mn{offers}(s,p) \}.
  \end{array}
  $$
  where $\mn{offers}(s,p)$ means that supplier $s$ offers product $p$,
  $W$ stands for warehouse, and $E$ for express.  Assuming that we
  have complete knowledge of all existing warehouses (e.g.\ in our
  company), we use a circumscription pattern \CP that minimizes predicate
  $W$ and lets all other predicates vary. Let
  the ontology \Omc contain the guarded existential rules
  $$
  \begin{array}{r@{\;}c@{\;}l}
    \mn{offers}(x,y) &\rightarrow& \exists z \,
    \mn{supplies}(x,y,z) \\[1mm]
  \mn{supplies}(x,y,z) &\rightarrow& W(z)
  \end{array}
  $$
  where $\mn{supplies}(x,y,z)$ expresses that supplier $x$ supplies
  product $y$ to warehouse $z$. Note that since $W$ is minimized,
  the existentially quantified variable $z$ in the upper rule can only
  bind to $w_1$ and $w_2$. We thus obtain
  $$(\Omc,D) \models_\CP \exists z \, (\mn{supplies}(s,p,z) \wedge E(z)).$$

  We now illustrate a basic trick that underlies the hardness proofs
  in Section~\ref{sect:lowerbounds}.  Extend the database with
  $\mn{mirror}(w_1,w_2)$ and $\mn{mirror}(w_2,w_1)$
  expressing that $w_1$
  and $w_2$ are supplied with the same products by the same
  suppliers. We wish to extend \Omc with
  $$
    \mn{supplies}(x,y,z_1) \wedge \mn{mirror}(z_1,z_2) \rightarrow
    \mn{supplies}(x,y,z_2)
  $$
  which yields
  \begin{equation*}
  (\Omc,D) \models_\CP \mn{supplies}(s,p,w_i) \text{ for all } i
  \in \{1,2\}.
\tag{$*$}
\end{equation*}
  However, the above rule is not
  guarded. We may work around this by using the guarded rules
  $$
  \begin{array}{@{}l}
    \mn{supplies}(x,y,z_1) \rightarrow \exists z_2 \,
    (\mn{mirror}(z_1,z_2) \wedge \mn{supplies}(x,y,z_2)) \\[1mm]
    \mn{mirror}(x,y) \wedge     \overline{\mn{mirror}}(x,y)
    \rightarrow \mn{false}
  \end{array}
  $$
  and extend the data with $\overline{\mn{mirror}}(w_i,w_i)$ for $i
  \in \{1,2\}$. Then if $z_1$ binds to $w_1$, the
  existentially quantified variable $z_2$ can only bind to $w_2$ and
  vice versa, and we again obtain $(*)$.
\end{example}

\medskip
\noindent
{\bf Substitutions, Signatures, Types.}
For a tuple $\bar a$, we generally use $a_i$ to denote the $i$-th
element of $\bar a$, for $1 \leq i \leq |\bar a|$.
A \emph{substitution} $\sigma$ is a function that maps variables to
variables. We typically write $\sigma x$ in place of $\sigma(x)$.  For
a tuple $\bar u$ of variables and constants, we write $\sigma \bar u$
to denote the tuple obtained by applying $\sigma$ componentwise,
treating it as the identity on constants.

A \emph{signature} is a set of constants and relation symbols.
For an
FO sentence $\phi$, we use $\mn{sig}(\phi)$ to denote the set of such
symbols in $\phi$, $\const(\phi)$ to denote the set of constants in
$\phi$, and $\sigmin(\phi)$ to denote the set of constants used in $\phi$ in an equality
atom.

Fix a signature~$\sig$.
A \emph{term} is a variable or a constant
from $\sig$.  An \emph{atom} is of the form $R(\bar u)$ or $v_1=v_2$
with $R$ a relation symbol from $\sig$, $\bar u$ a tuple of terms and
$v_1, v_2$ terms.  A \emph{literal} is an atom or a negated atom.  For
every $n \geq 1$, fix a sequence of variables $x_1,\dots,x_n$.  An
\emph{$n$-type} on $\sig$ is a maximal satisfiable 
set of literals that uses
exactly the variables $x_1, \dots, x_n$. Let \Amf be a structure. If $\bar a \in A^n$, then
the $n$-type on $\sig$ \emph{realized at}~$\bar a$ in \Amf, denoted
$\tp[\sig]<n>[\Amf]{\bar a}$, is the unique $n$-type $t$ on $\sig$ with
$\mathfrak{A} \models t(\bar a)$. 
We may drop superscript $n$ as $n$
is always identical to the length of $\bar a$. 
For a set $S \subseteq A$, we use
$\tp[\sig]<1>[\Amf]{S}$ to denote the set of $1$-types
$\{ \tp[\sig]<1>[\Amf]{a} \mid a \in S\}$. As an
abbreviation,
we may write $\tp[\sig]<1>{\Amf}$ in place of $\tp[\sig]<1>[\Amf]{A}$.

\section{The Two-Variables Fragment FO$^2$}
\label{sect:fotwo}

We show that circumscribed consequence is $\coNExp^\NPclass$-complete in
FO$^2$ and so is circumscribed AQ-querying, in combined
complexity. Note that querying with CQs or UCQs is undecidable already
for non-circumscribed FO$^2$. We remark that this section showcases the
general approach that we also use, in a more intricate form, for GF later
on.

An FO$^2$ sentence is in \emph{Scott normal form} if it
has the form
\begin{equation*}
  \phi=\forall x \forall y \, \varphi \wedge \bigwedge_{i=1..n_\exists}
  \forall x \exists y \, \psi_i
\tag{$*$}
\end{equation*}
with $\varphi$ and $\psi_i$ quantifier-free. It has been shown in
\cite{Scott1962,GKV} that every FO$^2$ sentence $\phi_0$ can be
converted in polynomial time into an FO$^2$ sentence $\phi$ in Scott
normal form that is a conservative extension of $\phi_0$: 
every model of $\phi$ is a model of $\phi_0$ and, conversely,
every model of $\phi_0$ can be extended to a model of $\phi$ by
interpreting the fresh predicates in $\phi$. 

We now establish an improved finite model property for
non-circumscribed FO$^2$ that satisfies 
Property~$(\heartsuit)$ from the introduction.
%




   
\begin{restatable}{proposition}{lemsmallmodelsfotwo}
	\label{lem:smallmodelsfotwo}
	Let $\phi$ be an FO$^2$ sentence of the form $(*)$, $\Sigma=\mn{sig}(\phi)$, \Amf a
        model of~$\phi$, and $k=|\Sigma \setminus \const(\phi)|$.
        Then there exists a model \Bmf of $\phi$ such that
	\begin{enumerate}
		
        \item 
  $|B| \leq 
    |\phi|^{n_\exists+1} \cdot 2^{n_\exists 4(k+6)}$;

		
		\item $\mn{tp}^1_\Sigma(\Amf) = \mn{tp}^1_\Sigma(\Bmf)$;
		
		\item 
		$|\{ a \in B \mid \mn{tp}_{\Bmf, \Sigma}^1(a)=t \}|
		\leq |\{ a \in A \mid \mn{tp}_{\Amf, \Sigma}^1(a)=t \}|$ for every 1-type $t$
		on $\Sigma$;
		
		\item $c^\Amf = c^\Bmf$ for all constants $c$ in $\phi$.
		
	\end{enumerate}
\end{restatable}
%
We remark that the construction from \cite{GKV} only yields 
Proposition~\ref{lem:smallmodelsfotwo} without Point~3, that is, it
may \emph{increase} the number of instances of some of the 
1-types realized in the original model. 
We next use Proposition~\ref{lem:smallmodelsfotwo} to establish
the following.
\begin{proposition}
  \label{lem:smallcircmodelfotwo}
  	Circumscribed
  FO$^2$ has the finite model property: if $\phi, \psi$ are
  FO$^2$-sentences with $\phi \not\models_\CP \psi$, 
  then there is a \CP-minimal model \Amf of $\vartheta = \phi \land \lnot \psi$
  with $|A| \leq |\vartheta|^{n_\exists+1} \cdot 4^{n_\exists (k+6)}$, 
  where $n_\exists$ is the number of existential quantifiers in the Scott normal form of $\vartheta$ and
  $k=|\mn{sig}(\vartheta) \setminus \const(\vartheta)|$.
%
\end{proposition}
\begin{proof}
  Assume that $\phi \not\models_\CP \psi$. Then there is a \CP-minimal
  model \Amf of $\phi$ with $\Amf \not\models \psi$. Thus \Amf is a
  model of $\phi \wedge \neg \psi$. By
  Proposition~\ref{lem:smallmodelsfotwo} there is a model \Bmf of
  $\phi \wedge \neg \psi$ that satisfies Points~1-4 of the
  proposition, with $\Sigma = \mn{sig}(\phi \wedge \neg \psi)$.  We
  show that \Bmf is a \CP-minimal model of $\phi$.

  Assume to the contrary that there is a model $\Bmf'$ of $\phi$
  such that $\Bmf' <_\CP \Bmf$. To obtain a contradiction, we
  construct a model $\Amf'$ of $\phi$ such that $\Amf' <_\CP
  \Amf$.

  Of course, $\Amf'$ must have the same universe as $\Amf$, thus
  we set \mbox{$A'=A$}. By Point~3 of Proposition~\ref{lem:smallmodelsfotwo}, we
  find an injection $f:B \rightarrow A$ such that
  $\mn{tp}^1_{\Bmf,\Sigma}(b)=\mn{tp}^1_{\Amf,\Sigma}(f(b))$ for all $b \in B$. We
  define $\Amf'$ so that its restriction to the range of $f$ is
  isomorphic to $\Bmf'$, with $f$ being an isomorphism. In particular,
  this restriction interprets all constants. To define the
  remaining part of $\Amf'$, we use cloning. By Point~2 of
  Proposition~\ref{lem:smallmodelsfotwo} and choice of $f$, for every
  $a\in A$ that is not in the range of $f$ we find a
  $\widehat a \in A$ that is in the range of $f$ and such that
  $\mn{tp}^1_{\Amf,\Sigma}(\widehat a)=\mn{tp}^1_{\Amf,\Sigma}(a)$. We then make $a$ a
  clone of $\widehat a$ in~$\Amf'$, that is, we set
  \begin{itemize}

  \item $\mn{tp}^1_{\Amf',\Sigma}(a)=\mn{tp}^1_{\Amf',\Sigma}(\widehat a)$; 
    
  \item $\mn{tp}^2_{\Amf',\Sigma}(a,b)=\mn{tp}^2_{\Amf',\Sigma}(\widehat a,b)$ for
    all $b$ in the range of $f$; 

  \item
    $\mn{tp}^2_{\Amf',\Sigma}(a,b)=\mn{tp}^2_{\Amf',\Sigma}(\widehat a,\widehat b)$
    for all $b$ not in the range of $f$.
    
  \end{itemize}
  It is easy to verify that $\Amf'$ is a model of $\phi$, since
  $\Bmf'$ is. Note in particular that the cloning does not touch on
  the constants, that is, if we make element $a$ a clone of
  $\widehat a$, then there is no constant $c$ with
  $c^{\Amf'}=\widehat a$. This can be seen as follows.  Assume that there
  was a constant $c$ with $c^{\Amf'}=\widehat a$. Then Point~4 of
  Proposition~\ref{lem:smallmodelsfotwo}, the definition of $<_\CP$,
  and the construction of the initial $\Amf'$ yields $c^\Amf=\widehat
  a$. But then   $\mn{tp}^1_{\Amf_\Sigma}(\widehat a)=\mn{tp}^1_{\Amf_\Sigma}(a)$
  implies $a = \widehat a$, which contradicts our initial assumption
  that $\widehat a$ is in the range of $f$, but $a$ is not.
  
  It can be verified that $\Amf' <_\CP \Amf$, since $\Bmf'
  <_\CP \Bmf$.
  %
  %
  %
  %
\end{proof}
It is now easy to derive the main result of this section.
\begin{theorem}
  \label{thm:circfo2dec}
  Circumscribed consequence in FO$^2$ is $\coNExp^{\NPclass}$-complete.
\end{theorem}
\begin{proof}
  The lower bound is inherited from the description logic \ALC \cite{Bonatti2009}.
  The upper bound is based on Proposition~\ref{lem:smallcircmodelfotwo},
  as follows.
  
  It is not hard to see that there exists an \NPclass algorithm that
  takes as input an FO$^2$ sentence $\phi$, a circumscription pattern
  \CP, and a finite structure \Amf and checks whether \Amf is
  \emph{not} a \CP-minimal model of $\phi$: the algorithm first checks
  in polynomial time whether \Amf is a model of $\phi$, answering
  “yes” if this is not the case. Otherwise, it guesses a structure
  $\Amf'$ with $A'=A$ and checks whether $\Amf'$ is a model of $\phi$
  and $\Amf' <_\CP \Amf$. It answers “yes” if both checks succeed, and
  “no” otherwise. Clearly, checking whether $\Amf' <_\CP \Amf$ can be
  done in time polynomial in the size of \Amf.

  This \NPclass algorithm may now be used as an oracle in a {\sc
    NExp}-algorithm for deciding $\phi \not\models_\CP \psi$: by
  Proposition~\ref{lem:smallcircmodelfotwo}, it suffices to guess a
  structure \Amf 
with $|A| \leq |\vartheta|^{n_\exists+1} \cdot 2^{n_\exists 4(k+6)}$,
check
  that it is not a model of $\psi$, and then use the
  \NPclass algorithm from above to check that \Amf is a
  \CP-minimal model of~$\phi$.
\end{proof}
From the above, we also obtain results on circumscribed AQ-querying.
\begin{restatable}{theorem}{thmfotwoconsvsaq}
  \label{thm:fotwoconsvsaq}
  Circumscribed AQ-querying in FO$^2$ is $\coNExp^{\NPclass}$-complete in combined complexity and
  $\Pi^p_2$-complete in data complexity. 
\end{restatable}
\begin{proof}
  For combined complexity, it suffices to show that circumscribed
  consequence and circumscribed AQ-querying mutually reduce to one
  another in polynomial time. First,
  \mbox{$(\Omc,D) \models_\CP R(\bar c)$} 
  is equivalent to $\phi \models_{\CP'} R(\bar c)$
  where
  $$
  \phi=\bigwedge \Omc \wedge \bigwedge D
\wedge \bigwedge_{\substack{c,c' \in \mn{adom}(D)\\c \neq c'}} c \neq c'.
$$
%
And second, $\phi \models_\CP \psi$ is equivalent to
$\phi' \models_{\CP'} P(c)$ where
$\phi'=\phi \wedge (\psi \rightarrow P(c))$, $P$ is a fresh unary
predicate that is varying in $\CP'$, and $c$ a fresh constant.

\smallskip The lower bound for data complexity is inherited from \ALC
\cite{DBLP:conf/kr/LutzMN23}. For the upper bound, we may argue
exactly as in the proof of Theorem~\ref{thm:circfo2dec}, where the
structure \Amf to be guessed is now of polynomial size since 
$n_\exists$ and $k$ are now constants in Propositions~\ref{lem:smallmodelsfotwo} and~\ref{lem:smallcircmodelfotwo}.
For~$k$, this depends on the assumption, which we may make w.l.o.g.,
that the database contains only predicates that occur also in the
ontology or query.
\end{proof}
%

\section{Upper Bounds for the Guarded Fragment}
\label{sec-gf}

We show that circumscribed consequence in GF is in \Tower and so is
circumscribed UCQ-querying, in combined complexity. We also show that
UCQ-querying is in \Elementary in data complexity, that is, 
for every GF ontology \Omc, circumscription pattern \CP, and AQ
$A(\bar x)$, there is a $k \geq 1$ such that given a database
$D$, it is in $k$-\Exp to decide whether
$(\Omc,D) \models_\CP A(\bar x)$. 

We remind the reader of the relevant complexity
classes, namely $\Elementary = \bigcup_{k \geq 1} \text{$k$-\Exp}$
and 
$$\Tower = \bigcup_{f \in \FElem} \textsc{Space}(\tower(f(n))$$ where
$\FElem$ is the class of all elementary functions and 
$\tower(x)$ denotes a tower of twos of height $x$.

\subsection{Circumscribed Consequence}
\label{sec-circcons}

We consider \GF sentences $\phi$ in Scott normal form. Such a sentence
takes the shape
$$
\bigwedge_{1 \leq i \leq n_\forall}\forall \bar x \, (\alpha_i \rightarrow \varphi_i)
\wedge \bigwedge_{1 \leq i \leq n_\exists} \forall \bar x \, (\beta_i \rightarrow
\exists \bar y \, ( \gamma_i \wedge \psi_i))
$$
where the $\alpha_i$, $\beta_i$ and $\gamma_i$ are atoms and the
$\varphi_i$ and $\psi_i$ are quantifier-free. It has been shown in
\cite{DBLP:journals/jsyml/Gradel99} that every GF sentence $\phi$ can be converted in
polynomial time into a GF sentence in Scott normal form that is a
conservative extension of $\phi$. 

%

%
We now state the improved finite model property for GF that satisfies
Property~$(\heartsuit)$ from the introduction.
Let $\towerof{0}{n} := n$ and, for every $k \geq 1$, define
$\towerof{k}{n} := 2^{\towerof{k-1}{n}}$, so that $\towerof{k}{n}$
refers to an exponentiation tower that consists of $k$ twos followed
by an~$n$.  A signature is \emph{unary} if it only contains constants
symbols and unary predicates.   
\begin{proposition}
  \label{prop:smallmodelsgf}
  Let $\phi$ be a GF sentence, $\sig$ a unary signature that contains $\sigmin(\phi)$, and \Amf a model of
  $\phi$.
  Then there exists a model \Bmf of $\phi$ 
  that satisfies
  the following properties:
  \begin{enumerate}

  \item $|B| \leq \towerof{4^{\sizeof{\sig} + 4}}{\sizeof{\phi}}$;

  \item $\tp[\sig]<1>{\Amf} = \tp[\sig]<1>{\Bmf}$;

  \item 
    $|\{ a \in B \mid \tp[\sig]<1>[\Bmf]{a}=t \}|
    \leq |\{ a \in A \mid \tp[\sig]<1>[\Amf]{a}=t \}|$ for every 1-type $t$ on $\sig$;
    
  \item 
  	$c^\Amf = c^\Bmf$ for all constants $c$ in $\phi$.
    
  \end{enumerate}
\end{proposition}
%
%
The reader should think of $\Sigma$ as the signature that contains,
apart from $\sigmin(\phi)$, all unary predicates that are minimized
and fixed in a circumscription pattern.

To establish Proposition~\ref{prop:smallmodelsgf}, we use a modified
version of the Rosati cover that leaves untouched a selected part
$\Delta$ from the original model.  In addition, if $\Delta$ contains
all the instances of some unary type in the original model, then so
does $\Delta$ in the modified Rosati cover. 
The precise formulation follows.

\begin{lemma}
	\label{lemma-non-uniform}
	Let $\phi$ be a GF sentence and $\sig$ a unary signature that contains $\sigmin(\phi)$.
	For all
	models $\mathfrak{A}$ of $\phi$ and all $\Delta \subseteq
	\domain{A}$ that contain $c^\Amf$ for every constant $c$ in~$\phi$,
	there exists a model $\structure{B}$ of $\phi$ that satisfies
	the following properties:
	\begin{enumerate}
		\item 
		$\sizeof{\domain{B}} \leq 2^{(\sizeof{\Delta} + \sizeof{\phi})^{\sizeof{\phi} + 11}}$;
		\item
		$\tp[\sig]<1>[\Amf]{a} = \tp[\sig]<1>[\Bmf]{a}$
		for all $a \in \Delta$;
		\item
		$\tp[\sig]<1>[\Amf]{\domain{A} \setminus \Delta} = \tp[\sig]<1>[\Bmf]{\domain{B} \setminus \Delta}$;
		\item 
		$\Delta \subseteq \domain{B}$ and $c^\Amf = c^\Bmf$ for all constants $c$ in $\phi$.
	\end{enumerate}
\end{lemma}
We now prove Proposition~\ref{prop:smallmodelsgf} by using
Lemma~\ref{lemma-non-uniform} and choosing an appropriate $\Delta$.
For any $\Delta \subseteq A$ that contains $c^\Amf$ for all constants
$c$ in~$\phi$, we use $\saveof{A}{\Delta}$ to denote the finite model
of $\phi$ produced by Lemma~\ref{lemma-non-uniform} (where `rc' stands
for `Rosati cover').

Let $\phi$, $\Sigma$, and \Amf be as in
Proposition~\ref{prop:smallmodelsgf}.  For every 1-type $t$ on $\sig$,
set
$$
\#_\Amf(t) := |\{ a \in A \mid \tp[\sig]<1>[\Amf]{a} = t \}|.
$$
The challenge is to choose $\Delta$ so that Point~3 of
Proposition~\ref{prop:smallmodelsgf} is satisfied. Call a 1-type $t$
\emph{stable} w.r.t.\ $\Delta \subseteq A$ if
$\#_{\saveof{A}{\Delta}}(t) \leq \#_\Amf(t)$ and call $\Delta$
\emph{stabilizing} if all 1-types are stable w.r.t.\ $\Delta$.  To
attain Point~3, it clearly suffices to choose a stabilizing $\Delta$.

We use a set $\Delta$ that contains \emph{all} instances
of 1-types realized only a certain number of times: for
$m \geq 1$, set
$$
\begin{array}{r@{\;}c@{\;}l}
\Delta_m &:=&
\{ a \in A \mid \#_\Amf(\tp[\sig]<1>[\Amf]{a}) \leq m \} \; \cup \\[1mm] 
&& \{ c^\Amf \mid c \in \const(\phi) \}. 
\end{array}
$$
Now consider $\Delta_m$, for some $m \geq 1$. For those 1-types $t$
that are realized in \Amf at most $m$ times, it is clear from Points~2
and~3 of Lemma~\ref{lemma-non-uniform} that $\saveof{A}{\Delta_m}$ has
the very same instances of~$t$, and thus $t$ is stable w.r.t.\
$\Delta_m$. Other types, however, may not be stable.

So can we find a value for $m$ to make $\Delta_m$ stabilizing? This
is trivially the case for
\[
m := \max(\{ \#_\Amf(t) \mid t \in \tp[\sig]<1>{\Amf}, \#_\Amf(t) < + \infty\}),
\]
but we would like to have an $m$ that is bounded from above
to comply with Point~1 in Proposition~\ref{prop:smallmodelsgf}.

\begin{restatable}{lemma}{lemmaseparator}
	\label{lemma-separator}
	There exists a stabilizing set $\Delta_m$ such that $2 \leq m  \leq \towerof{2^{2\sizeof{\sig} + 4}}{\sizeof{\phi}}$. 
\end{restatable}
\begin{proof}
  We start with $m_0 = 2$ (starting with
  $1$ would also work but using
  $2$ simplifies calculations as we are dealing with towers of
  $2$s).  If
  $\Delta_{m_0}$ is stabilizing, we are done.  Otherwise there must be
  a $1$-type $t$ on $\sig$ that is not stable
  w.r.t.~$\Delta_{m_0}$, i.e.\ $\#_{\saveof{A}{\Delta_{m_0}}}(t) >
  \#_\Amf(t)$. This implies in particular that
  $\#_\Amf(t)$ is no larger than the size of the universe of
  $\mn{rc}(\Amf,
  \Delta_{m_0})$.  Using the bound from Point~1 of
  Lemma~\ref{lemma-non-uniform}, we set $m_1 =
  2^{(\sizeof{\Delta_{m_0}} + \sizeof{\phi})^{\sizeof{\phi} +
      11}}$.  Now all instances of
  $t$ in \Amf are contained in
  $\Delta_{m_1}$ and by Points~2 and~3 of
  Lemma~\ref{lemma-non-uniform} we have
  $\#_{\saveof{A}{\Delta_{m_1}}}(t) = \#_\Amf(t)$, thus
  $t$ is stable w.r.t.\ $\Delta_{m_1}$, and in fact for any
  $\Delta_m$ with $m \geq m_1$. We proceed in this way, with
  $m_1$ in place of
  $m_0$, etc. This yields a sequence
  $m_0,m_1,m_2,\dots$ and for each $i \geq
  0$ the set of 1-types on
  $\sig$ that is stable w.r.t.\
  $\Delta_{m_{i+1}}$ is a strict superset of the set of 1-types on
  $\sig$ that is stable w.r.t.\ $\Delta_{m_{i}}$.  Since
  $\sig$ is unary and the 1-types of interest all come from the fixed
  interpretation
  $\Amf$, thus agreeing on the constant symbols, there are at most
  $2^\sizeof{\sig}$ many 1-types to consider.  Therefore, after at
  most $2^\sizeof{\sig}$ iterations we have found an
  $i$ such that
  $\Delta_{m_{i}}$ is stabilizing.  Let us argue that we have achieved
  the claimed bound on $m$. Take any $i \geq
  0$. Then $\sizeof{\Delta_{m_i}} \leq \sizeof{\phi} +
  2^\sizeof{\sig}m_i$.  Moreover, using the bound from Point~1 of
  Lemma~\ref{lemma-non-uniform} and $m_i \geq
  2$, we can show that $m_{i+1} \leq
  2^{m_i^{\sizeof{\phi}^7}}$.  This, in turn, gives $m_i \leq
  \towerof{i+1}{\sizeof{\phi}^{7i}}$.  Since we stop at the latest at
  $i = 2^\sizeof{\sig}$, we obtain $m_i \leq \towerof{2^\sizeof{\sig}
    + 1}{\sizeof{\phi}^{7\cdot
      2^\sizeof{\sig}}}$, which implies $m_i \leq
  \towerof{2^{2\sizeof{\sig} + 4}}{\sizeof{\phi}}$.

\end{proof}

To conclude the proof of Proposition~\ref{prop:smallmodelsgf}, it then suffices to let $m$ be as in Lemma~\ref{lemma-separator} and set $\Bmf := \saveof{\Amf}{\Delta_m}$.
Since $m$ is stabilizing, Point~3 of Proposition~\ref{prop:smallmodelsgf} is satisfied.
For Point~1, we may use Point~1 of
Lemma~\ref{lemma-non-uniform} and the fact that $\sizeof{\Delta_m} \leq
2^\sizeof{\sig}m + \sizeof{\phi}$. 


      
      %

\smallskip

We now lift the finite model property from
Proposition~\ref{prop:smallmodelsgf} to circumscribed consequence.
To apply Proposition~\ref{prop:smallmodelsgf}, we choose a unary
signature $\sig$ that contains the minimized and fixed predicates from
the circumscription pattern used.
The rest of the proof is similar to that of
      Proposition~\ref{lem:smallcircmodelfotwo}.
\begin{restatable}{theorem}{lemsmallcircmodelgf}
  \label{lem:smallcircmodelgf}
  Circumscribed GF has the finite model property. More precisely,
  every satisfiable GF sentence $\phi$ circumscribed by $\CP = ({\prec}, \Msf, \Fsf, \Vsf)$ has a
  model \Amf with $|A| \leq \towerof{4^{\sizeof{\sig} + 4}}{\sizeof{\phi}}$, where $\sig = \sigmin(\phi) \cup \Msf \cup \Fsf$.
\end{restatable}
Building on Theorem~\ref{lem:smallcircmodelgf}, we now obtain the
following using a brute-force enumeration procedure.
\begin{restatable}{theorem}{thmsatdec}
  \label{thm:circgfdec}
  Circumscribed consequence in GF is decidable and in $\Tower$.
\end{restatable}


\subsection{Circumscribed Querying}
\label{sec-circ-query}

We prove that UCQ-querying (and thus also AQ-querying) in
GF is decidable. 

\begin{theorem}
	\label{theorem:gf-aq-querying-upperbounds}
	Circumscribed UCQ-querying in \GF is in $\Tower$ w.r.t.\ combined complexity and in $\Elementary$ w.r.t.\ data complexity.
\end{theorem}

%

       


Without circumscription, decidability of UCQ-querying in GF is almost
immediate as one can replace the UCQ $q$ with the disjunction $q'$ of
all acyclic CQs that imply a CQ in $q$ (up to a certain size) and then
express $q'$ as a GF sentence, obtaining a reduction to
unsatisfiability \cite{BGO}. This does not work with circumscription.
\begin{example}
	\label{ex:second}
  Take the ontology \Omc that consists of the sentence
  $$
  \begin{array}{r}
    \forall x \, \big (A(x) \rightarrow \exists y (R(x,y) \wedge \exists z
    (R(y,z) \; \wedge \quad \\[1mm]
    \exists u \, (R(z,u) \wedge A(u)))) \big ),
  \end{array}
  $$
  the database $D=\{ A(a) \}$, and let \CP minimize $A$ and vary all
  other predicates. Then $(\Omc,D) \models_\CP q$ where:
  $$q := \exists x\ \exists y\ \exists z\,\ R(x,y) \wedge
  R(y,z) \wedge R(z,x),$$ but
  there is no acyclic CQ $q'$ that implies $q$ and satisfies
  $(\Omc,D) \models_\CP q'$.
\end{example}
We thus use a somewhat different, mosaic-based approach which exploits
the fact that if $(\Omc,D) \not\models_\CP q$, then this is witnessed by a
model \Bmf that, in a certain loose sense, has the shape of a forest.
More precisely, \Bmf can be obtained from any model $\Amf$ that
witnesses $(\Omc,D) \not\models_\CP q$ by a version of guarded unraveling
(see e.g.\ \cite{DBLP:books/daglib/p/Gradel014}) that leaves untouched
a `core' of $\Amf$ defined as 
\[
	\mn{core}_\Sigma(\Amf) := \{ a \in A \mid \#_\Amf(\mn{tp}^1_{\Amf,\Sigma}(a)) \leq \towerof{4^{\sizeof{\sig} + 4}}{\sizeof{\phi}} \}
      \]
for a suitable signature $\Sigma$.      
With `leaving untouched', we mean that elements from this core are not
duplicated during unraveling, but `reused' whenever needed. This is
required to guarantee minimality w.r.t.\ the circumscription pattern.
It is not important to make this unraveling explicit for the
subsequent development, but it may still guide the reader's intuition.

To prepare for the subsequent development, we give a central lemma
that establishes a sufficient condition for a model \Bmf to be
\CP-minimal, based on comparing it to a \CP-minimal `reference model'
\Amf. This is a version of the `core lemma' of
\cite{DBLP:conf/kr/LutzMN23}.
\begin{restatable}{lemma}{corelemma}
\label{lem-modelOfK}
\label{lem-lemma5} 
Let $\phi$ be a GF sentence, $\CP = (\prec, \Msf, \Fsf, \Vsf)$, and $\sig = \sigmin(\phi) \cup \Msf \cup \Fsf$.
Further let \Amf be a $\CP$-minimal model of~$\phi$ and let
\Bmf be a model of $\phi$ 
 such that
\begin{enumerate}
	\item 
	$
	\coreofsigma{A}
	\subseteq B
	$ and $c^\Bmf = c^\Amf$ for all $c \in \const(\phi)$;
	\item
	$\mn{tp}^1_{\Amf,\Sigma}(a)=\mn{tp}^1_{\Bmf,\Sigma}(a)$
	for all $a \in \coreofsigma{A}$;
	\item
	$\mn{tp}^1_{\Bmf,\Sigma}(B\setminus\coreofsigma{A}) = \mn{tp}^1_{\Amf,\Sigma}(A \setminus \coreofsigma{A})$,
\end{enumerate}
Then $\Bmf$ is a $\CP$-minimal model of $\phi$.
\end{restatable}
Intuitively, Lemma~\ref{lem-lemma5} says that the exact multiplicity
of types realized in \Amf outside of $\coreofsigma{A}$ is irrelevant for
$\CP$-minimality.  

Assume that we are given as an input a GF knowledge base $(\Omc,D)$,
a circumscription pattern $\CP$, and a Boolean UCQ $q$. 
We want to decide whether 
there is a countermodel \Imc against
$(\Omc,D) \models_\CP q(\bar a)$. This may be rephrased as
$\phi \models_\CP q(\bar a)$ for
$$\phi=\bigwedge \Omc \wedge \bigwedge D
\wedge \!\!\! \bigwedge_{c \in \mn{adom}(D)} \!\!\! \big ( P_c(c) \wedge
\!\!\!\!\! \bigwedge_{c' \in \mn{adom}(D) \setminus \{ c\}} \!\!\!\!\!\neg P_{c'}(c)\big).$$
%
We shall use the latter formulation. We may assume that $\phi$ is in
Scott normal form, that is, $\phi$ is 
$$
\bigwedge_{1 \leq i \leq n_\forall}\forall \bar x \, (\alpha_i \rightarrow \varphi_i)
\wedge \bigwedge_{1 \leq i \leq n_\exists} \forall \bar x \, (\beta_i \rightarrow
\exists \bar y \, ( \gamma_i \wedge \psi_i)).
$$

Let $\sig = \sigmin(\phi) \cup \Msf \cup \Fsf$. 
Set
$M = \sizeof{\const(\phi)} + 2^\sizeof{\sig} \cdot
\towerof{4^{\sizeof{\sig} + 4}}{\sizeof{\phi}} + 2^\sizeof{\sig}$, and
fix a set $U$ of size $M$. In an outer loop, our algorithm iterates
over all pairs $(\Amf_0,\typescandidate)$ with $\Amf_0$ a finite
structure that interprets all constants from $\phi$ and
$\typescandidate$ a set of 1-types such that the following conditions
are satisfied:
\begin{itemize}

\item $A_{0} \subseteq U$; 
\item $\typescandidate \subseteq \mn{tp}^1_\sig(\Amf_0)$.

  %
          
    %
\end{itemize}
We define $\Delta :=  \{ a \in A_0 \mid \mn{tp}^1_{\Amf_0, \sig}(a) \notin \typescandidate \}$.

For each pair $(\Amf_{0},\typescandidate)$, we then check
whether the following additional conditions are satisfied:
\begin{description}

\item[\textnormal{(I)}] $\Amf_{0}$ can be extended to a model \Amf of $\phi$
  such that
  \begin{enumerate}

  \item[(a)] $\Amf|_{A_{0}} = \Amf_{0}$,


  \item[(b)] $\mn{tp}^1_{\Amf, \sig}(A \setminus \Delta)
    = \typescandidate$,

  \item[(c)] $\Amf \not\models q$;

  \end{enumerate}

\item[\textnormal{(II)}] there exists a $\CP$-minimal model $\Bmf$ of
  $\phi$ such that
  %
  \begin{enumerate}
  	\item[(d)]
  	$\coreofsigma{B} = \Delta$;
  		
  \item[(e)] $\mn{tp}^1_{\Bmf, \sig}(a)=\mn{tp}^1_{\Amf, \sig}(a)$ for all $a \in \Delta$ and

    \item[(f)]
$ \mn{tp}^1_{\Bmf, \sig}(B \setminus \Delta) =
  \typescandidate$ 
  \end{enumerate}
  %
\end{description}
We return `yes' if all pairs fail the check and `no' otherwise.

\smallskip

If the checks succeed, then the model \Amf of $\phi$ from
Condition~(I) is a countermodel against $(\Omc,D) \models_\CP q$.  In
particular, we may apply Lemma~\ref{lem-lemma5}, using the model \Bmf
from Condition~(II) as the reference model, to show that \Amf is
$\CP$-minimal. Conversely, from any countermodel $\Amf$
against $(\Omc,D) \models_\CP q$, we can read off a pair $(\Amf_0, \typescandidate)$
%
by choosing $\typescandidate := \mn{tp}^1_{\Amf, \sig}(A \setminus
\coreofsigma{A})$ and $\Amf_0$ to be the restriction of \Amf to universe
$$
U_\Amf := \{ c^\Amf \mid c \in \const(\phi) \} \cup \coreofsigma{A} \cup \{ w_t \mid t \in \typescandidate \}
$$
where $w_t \in A$ is chosen arbitrarily such that \mbox{$\mn{tp}^1_{\Amf, \sig}(w_t) = t$}.
Then $\Amf$ witnesses Condition~(I) and choosing
$\Bmf=\Amf$ witnesses Condition~(II).

Of course, we have to prove that Conditions~(I) and~(II) are
decidable. For Condition~(II), we prove that the following is
a consequence of Lemma~\ref{lem-lemma5}:

\begin{restatable}{lemma}{smallreferencemodel}
	\label{lemma-small-reference-model}
	Let $\phi$ be a GF sentence, $\CP = (\prec, \Msf, \Fsf,
        \Vsf)$, and $\sig = \sigmin(\phi) \cup \Msf \cup \Fsf$.
	Let \Amf be a $\CP$-minimal model of~$\phi$.
	Then there exists a $\CP$-minimal model $\Bmf$ of $\phi$ 
	such that
	\begin{enumerate}
		\item 
		$\coreofsigma{B} = \coreofsigma{A}$;
		\item
		$\mn{tp}^1_{\Bmf,\Sigma}(a)=\mn{tp}^1_{\Amf,\Sigma}(a)$
		for all $a \in \coreofsigma{A}$;
		\item
		$\mn{tp}^1_{\Bmf,\Sigma}(B\setminus\coreofsigma{A}) = \mn{tp}^1_{\Amf,\Sigma}(A \setminus
		\coreofsigma{A})$;
		\item 
		$\sizeof{B} \leq 2^\sizeof{\sig} (1+ \towerof{4^{\sizeof{\sig} + 4}}{\sizeof{\phi}})$.
	\end{enumerate}
\end{restatable}
It follows that if a model as in (II) exists, then there exists one of size
at most $2^\sizeof{\sig} (1+ \towerof{4^{\sizeof{\sig} + 4}}{\sizeof{\phi}})$ and thus we can iterate
over all candidate structures \Bmf up to this size, check whether \Bmf
is a model of $\phi$ that satisfies Conditions~(d) to~(f), and then
iterate over all models $\Bmf'$ of $\phi$ with $B'=B$ to check that
$\Bmf$ is $\CP$-minimal.

Condition~(I) requires more work. We use a mosaic approach, that is,
we attempt to assemble the structure \Amf from Condition~(I) by
combining small pieces called \emph{mosaics}.
%
Fix a pair $(\Amf_{0},\typescandidate)$.  A mosaic for
$(\Amf_{0},\typescandidate)$ is a decorated finite structure whose
universe contains $\Amf_0$ and possibly elements from a fixed set
$U^+$ of $2 \cdot \mn{ar}$ elements where $\mn{ar}$ is the maximum arity of
predicates in $\phi$. 

We trace partial homomorphisms from CQs in $q$ through the mosaics, as
follows.
%
%
%
A \emph{match triple} for a structure \Bmf takes the form
$(p,\widehat p,h)$ such that $p$ is a CQ in~$q$,
$\widehat p \subseteq p$, and $h$ is a partial map from
$\mn{var}(\widehat p)$ to $B$ that is a homomorphism from
$\widehat p|_{\mn{dom}(h)}$ to~\Bmf where $\widehat p|_{\mn{dom}(h)}$
denotes the restriction of $\widehat p$ to the variables in the domain
of $h$. Intuitively, \Bmf is a mosaic and the triple
$(p,\widehat p,h)$ expresses that a homomorphism from $\widehat p$ to
\Amf exists, with the variables in $\mn{dom}(h)$ being mapped to the
current piece~\Bmf and the variables in
$\mn{var}(\widehat p) \setminus \mn{dom}(h)$ mapped to other pieces
of~\Amf. A match triple is \emph{complete} if $\widehat p = p$ and
\emph{incomplete} otherwise.  To make \Amf a countermodel, we must
avoid complete match triples.  A \emph{specification} for a structure
\Bmf is a set $S$ of match triples for \Bmf and we call $S$
\emph{saturated} if the following conditions are satisfied:
\begin{itemize}

\item if $p$ is a CQ in $q$, $\widehat p \subseteq p$, and $h$ is
  a homomorphism from $\widehat p$ to $\Bmf$, then $(p,\widehat p,h) \in S$;

\item if $(p,\widehat p,h),(p,\widehat p',h') \in S$ and $h(x)=h'(x)$
  is defined for all
  $x \in \mn{var}(\widehat p) \cap \mn{var}(\widehat p')$, then
  $(p,\widehat p \cup \widehat p',h \cup h') \in S$.
  

  
\end{itemize}
\begin{definition}
  A \emph{mosaic} for $(\Amf_{0},\typescandidate)$ is a pair
  $M=(\Bmf,S)$ where
  \begin{itemize}

  \item $\Bmf$ is a finite structure such that
    \begin{enumerate}
    \item 
    $B \subseteq A_0 \cup U^+$;

    \item 
      $\Bmf|_{A_0} = \Amf_0$;

    \item 
      $\mn{tp}^1_{\Bmf, \sig}(B \setminus \Delta) \subseteq \typescandidate$;

    \item \Bmf satisfies $\forall \bar x \, (\alpha_i \rightarrow
      \varphi_i)$, for $1 \leq i \leq n_\forall$;
      
    \end{enumerate}

  \item $S$ is a saturated specification for \Bmf that contains only
    incomplete match triples.

  \end{itemize}
  We use $\Bmf_M$ to refer to \Bmf and $S_M$ to refer
  to $S$.
\end{definition}
%
Let \Mmc be a set of mosaics for $(\Amf_{0},\typescandidate)$. We say
that $M \in \Mmc$ is \emph{good in} \Mmc if for
$1 \leq i \leq n_\exists$, the following condition is satisfied; if
$\beta_i=R(\bar z)$ and $\bar a \in R^\Bmf$, then we find a mosaic
$M' \in \Mmc$ 
such that 
  %
  \begin{enumerate}

  \item $\mn{tp}_{\Bmf_M, \sig}(\bar a)=\mn{tp}_{\Bmf_{M'}, \sig}(\bar a)$;

  \item $\Bmf_{M'} \models \exists \bar y\, (\gamma_i \wedge
    \psi_i)[\bar a]$;

  \item if $(p,\widehat p,h') \in S_{M'}$, then $(p,\widehat p,h) \in S_{M}$ where
    $h$ is the restriction of $h'$ to range $A_0
      \cup \bar a$.



      
    
  \end{enumerate}
  %
  To verify
  Condition~(I), we start
  with the set of all mosaics for
  $(\Amf_{0},\typescandidate)$
  and repeatedly and exhaustively eliminate mosaics that are not good.
  \begin{restatable}{lemma}{lemmamosaic}
  	\label{lemma-mosaic}
        $\Amf_0$ can be extended to a model \Amf of $\phi$ that
        satisfies Conditions~(a) to~(c) iff 
        at least one mosaic survives the elimination process.
\end{restatable}
At this point, we have established
Theorem~\ref{theorem:gf-aq-querying-upperbounds}. It
should be clear that the presented algorithm establishes membership in
\Tower in combined complexity. For data complexity, note that the size
$M$ of the stuctures $\Amf_0$ in pairs $(\Amf_{0},\typescandidate)$ is
now $k$-exponential for a constant $k$: it is essentially an
exponentiation tower of twos followed by $|\phi|$ whose height is
independent of $D$ (while $|\phi|$ depends linearly on $|D|$). The same
is true for the bound established by Lemma~\ref{lemma-small-reference-model}
and the size of mosaics.

\section{Lower Bounds for the Guarded Fragment}
\label{sect:lowerbounds}

We prove lower bounds that match the upper bounds given in
Section~\ref{sec-gf}. Our proofs are formulated in terms of the data
complexity of AQ-querying, but we also derive from them tight complexity
results for circumscribed consequence.

We start with an \Exp lower bound on the data complexity of
AQ-querying for the restricted yet natural case where only a single
predicate is minimized and no predicate is fixed. It is then of course
pointless to use a preference relation in the circumscription
pattern. The bound applies even for ontologies \Omc that are sets of
existential rules.
%
\begin{theorem}
  \label{theorem:exp-hardness}
  AQ-querying in GF is \Exp-hard in data complexity even for
  ontologies that are sets of existential rules, with a single
  minimized predicate and no fixed predicates, and with a
  fixed signature.

  For UCQ-querying, the same even holds
  for a fixed signature in which all predicates have arity at most two.
%
\end{theorem}
\newcommand{\config}{\mn{conf}}
\newcommand{\nonconfig}{\overline\config}
\newcommand{\acc}{\mn{acc}}
\newcommand{\rej}{\mn{rej}}
\newcommand{\universalstate}{\mn{univ}}
\newcommand{\existentialstate}{\mn{exist}}
\newcommand{\nonstate}{\overline{\state}}
\newcommand{\nonpos}{\overline{\pos}}
\newcommand{\nonmove}{\overline{\move}}
\newcommand{\nonacc}{\overline{\mn{acc}}}
\newcommand{\nextconfig}{\mn{nextConf}}
\newcommand{\transition}{\mn{trans}}
\newcommand{\state}{\mn{state}}
\newcommand{\pos}{\mn{pos}}
\newcommand{\move}{\mn{move}}
\newcommand{\symb}{\mn{symb}}
\newcommand{\nonsymb}{\overline{\symb}}
\newcommand{\MM}{\mn{M}}
\newcommand{\tape}{\mn{tape}}
\newcommand{\nextpos}{\mn{next}}
\newcommand{\nextsymb}{\mn{next}}
\newcommand{\start}{\mn{start}}
\newcommand{\nonstart}{\overline{\mn{start}}}
\newcommand{\nonend}{\overline{\mn{end}}}
\newcommand{\goleft}{\mn{left}}
\newcommand{\goright}{\mn{right}}
\newcommand{\same}{\mn{same}}
\newcommand{\copytape}{\mn{copy}}
\newcommand{\copylefttape}{\mn{copy_{\triangleleft}}}
\newcommand{\copyrighttape}{\mn{copy_{\triangleright}}}
\newcommand{\nextleft}{\mn{next_{\triangleleft}}}
\newcommand{\nextright}{\mn{next_{\triangleright}}}
\newcommand{\nontape}{\overline{\tape}}
\newcommand{\nosymbleft}{\mn{symb_{\prec}}}
\newcommand{\nextnosymbleft}{\mn{next_{\prec}}}
\newcommand{\nosymbright}{\mn{symb_{\succ}}}
\newcommand{\nextnosymbright}{\mn{next_{\succ}}}
\newcommand{\error}{\mn{error}}
\newcommand{\catcherror}{\mn{getError}}
\newcommand{\target}{\mn{goal}}
\newcommand{\nonnextsymb}{\overline{\nextsymb}}
\newcommand{\nonnextpos}{\overline{\nextpos}}
\newcommand{\nontransition}{\overline{\transition}}
\newcommand{\constants}{\mn{Const}}
\newcommand{\myend}{\mn{end}}

We invite the reader to verify the proof of
Theorem~\ref{theorem:exp-hardness}, provided in the appendix, as a
warmup for the proof of the main result of this section, which
is up next.

We show that, when using multiple
minimized predicates as well as the preference order, then the data
complexity is no longer in $k$-\Exp for any $k \geq 1$. In other
words, while for every fixed ontology \Omc, query $q$, and
circumscription pattern \CP querying is in $k$-\Exp in data complexity
for some $k$ (c.f.\ Theorem~\ref{theorem:gf-aq-querying-upperbounds}),
$k$ cannot be uniformly bounded by a constant from above for all \Omc,
$q$, and \CP. In combined complexity, AQ-querying is even \Tower-hard.


\newcommand{\lol}{\kappa}

%

\begin{theorem}
  \label{theorem:tower-hardness}
  AQ-querying in GF is
  \begin{enumerate}
  \item $\Tower$-hard in combined complexity
    (under logspace reductions) and

  \item $k$-\Exp-hard for every $k \geq 1$ in data
    complexity. 
    
  \end{enumerate}
        This holds already for circumscribed sets of
        guarded existential rules and without fixed predicates.
\end{theorem}

%

\newcommand{\void}{\mn{void}}
\newcommand{\cell}{\mn{ord}}
\newcommand{\bit}{\mn{bit}}
\newcommand{\firstzero}{\mn{fz}}
\newcommand{\ones}{\mn{ones}}
\newcommand{\zeros}{\mn{zeros}}
\newcommand{\nextfz}{\mn{nextfz}}
\newcommand{\nextcell}{\mn{succ}}
\newcommand{\head}{\mn{head}}
\newcommand{\nonhead}{\overline{\head}}
\newcommand{\blanks}{\mn{blanks}}
\newcommand{\binary}{\mn{bin}}
\newcommand{\diff}{\mn{diff}}
\renewcommand{\error}{\mn{err}}
\newcommand{\catch}{\mn{root}}
\newcommand{\senderror}{\mn{sendErr}}
\renewcommand{\nonnextpos}{\overline{\nextcell}}

\begin{figure*}
	\begin{align}
		\ones_\triangleleft(x, y), \nonstart(y), \cell_{k-1}(y) & \rightarrow \exists y'\; \ones_\triangleleft(x, y'), \nextcell_{k-1}(y', y), \bit_{k, 1}(x, y')
		\label{rule:ones-left}
		\\
		\zeros_\triangleleft(x, y), \nonstart(y), \cell_{k-1}(y) & \rightarrow \exists y'\; \zeros_\triangleleft(x, y'), \nextcell_{k-1}(y', y), \bit_{k, 0}(x, y')
		\label{rule:zeros-left}
		\\
		\copytape_\triangleright(x, x', y), \nonend(y), \cell_{k-1}(y) & \rightarrow \exists y'\; \copytape_\triangleright(x, x', y'), \nextcell_{k-1}(y', y), \copytape(x, x', y')
		\label{rule:copy-right}
	\end{align}
	\caption{Additional rules used in the proof of Theorem~\ref{theorem:tower-hardness} for every $k \in \{ 2, \dots, \lol\}$.}
	\label{figure:rules-tower}
        \vspace*{-4mm}
\end{figure*}



We prove Point~2 as follows. It is known that, for every
$\lol \geq 1$, there is a fixed $(\lol-1)$-exponentially space-bounded
alternating Turing machine (ATM)
whose word problem is $\lol$-\Exp-hard \cite{chandra-alternation}. We
provide a reduction from the word problem of each of these ATMs to
AQ-querying in GF.\footnote{We use ATMs for uniformity with the proof
    of Theorem~\ref{theorem:exp-hardness}. We could also work with
    deterministic Turing machines which, however, would only simplify
    the proof in a minor way.}
Our reductions are uniform accross all
$\lol$ and, as discussed in
\cite{schmitz-complexity-beyond-elementary}, this also yields
$\Tower$-hardness in combined complexity.

Let $\lol \geq 1$ and let \Mmc be a $(\lol-1)$-exponentially
space-bounded alternating Turing machine (ATM)
whose word problem is $\lol$-\Exp-hard.  We exhibit a set of
existential rules $\Omc$ and a circumscription pattern $\CP$ such that
given an input $e = e_1 \cdots e_n \in \Sigma^*$ to \Mmc, we can
construct in polynomial time a database $D$ such that $\Mmc$ accepts
$e$ iff $\Omc, D \models_\CP \target(a)$, where $\target$ is a unary
predicate and $a$ a dedicated constant
symbol. 
	
One main challenge is to generate a tape of the required length and we
first focus on achieving that.  To this end, we produce $\lol$ linear
orders, with the $k^\text{th}$ order being of length
$\towerof{k-1}{p(n)}$. In other words, the first order has length
$p(n)$, the second has length $2^{p(n)}$, the third $2^{2^{p(n)}}$,
and so on, until the $\lol^\text{th}$ order which has length
$\towerof{\lol-1}{p(n)}$ and will be used as the tape for the ATM
computation. The positions in the $(k+1)^\text{st}$ order will be
encoded in binary using elements of the $k^\text{th}$ order as bit
positions.  For each~$k$, the element of the $k^\text{th}$ order are
marked with the unary predicate $\cell_k$.  To guarantee that the
encoding of a position in the $(k+1)^\text{st}$ order indeed only uses
bit positions from the $k^\text{th}$ tape, the predicates
$\cell_1$, $\dots$, $\cell_\lol$ are minimized.

We also use other minimized predicates, arranged in a preference
order as follows:
\[ \catch \prec \error_1 \prec \cell_1 \prec \dots \prec \error_\lol
  \prec \cell_\lol \prec \error_{\lol + 1}.\] The predicate $\error_k$
is used to `report' errors in the $k$-th order by being
made true on the constant $a$. This shall then make the query
predicate $\target$ true on~$a$ and in this way rule out erroneous
models. The preferred minimization of $\error_k$ over
$\cell_k$ acts as an incentive to avoid such errors.  We use an
additional predicate $\error_{\lol+1}$ to detect errors in the ATM
computation.  To enforce that $\error_k$ is reported precisely on $a$,
we use $\catch$ and include in $D$
	\[ \catch(a). \]
	Any other predicate used is varying, which concludes the definition of $\CP$.
	We now clarify how error reporting works.
	Since the minimization of $\catch$ is preferred over that of
        all other predicates, in every $\CP$-minimal model $\structure{A}$ we have $\catch^\structure{A} = \{ a^\structure{A} \}$.
	When an error on the $k^\textrm{th}$ tape is
        detected at some element $x$, we generate an instance $y$ of $\error_k$.
	We then require $\error_k$ to be subsumed by $\catch$, so
        that, in every $\CP$-minimal model, $y$ is actually $a$. We
        also require $\error_k$ to be subsumed by $\target$ so that
        $\target$ holds at $a^\structure{A}$ whenever an error is
        detected in the representation of the $k^\textrm{th}$ order.
	Formally, we include in \Omc, for every $k \in \{ 1, \dots, \lol+1 \}$, the rules
	\begin{align*}
		\error_k(x) & \rightarrow \catch(x), \target(x).
	\end{align*}
	We do not want distinct orders to share elements and report an
        error if they do. We also require $a$ not to be used as an
        order element. For $1 \leq i < j \leq \lol$, add
        \begin{align*}
        	\cell_i(x), \cell_j(x) &\rightarrow \exists y\
        	\error_j(y) \\
        	\cell_k(x), \catch(x) &\rightarrow \exists y\ \error_k(y). 
        \end{align*}
        Elements of the first order are represented in the
        database~$D$ as constants $c^1_i$, $1 \leq i \leq
        p(n)$: 
%
\begin{align*}
	\cell_1(c_i^1) & \textrm{ for } 0 \leq i < p(n).
\end{align*}
%
The $k^\textrm{th}$ order is represented by the binary predicate
$\nextcell_k$, for $1 \leq k \leq \lol$. 
We use a unary predicate $\nonend$ to mark the elements of orders
that are not the final element. For all
$k \in \{ 1, \dots, \lol \}$, we add the rules
\begin{align*}
	\nextcell_k(x, x') & \rightarrow \cell_k(x), \cell_k(x')
	\\
	\cell_k(x), \nonend(x) & \rightarrow \exists x'\; \nextcell_k(x, x').
\end{align*}
For the first order, we ensure the intended interpretation of
$\nextcell_1$ via a binary predicate $\nonnextpos_1$, the following facts in $D$:
\begin{align*}
	\nonnextpos_1(c^1_i,c^1_j) & \textrm{ for } 0 \leq i,j < p(n) \text{
		with } j \neq i + 1,
	\\
	\nonstart(c^1_i) & \textrm{ for } 0 < i < p(n)
	\\
	\nonend(c^1_i) & \textrm{ for } 0 \leq i < p(n) - 1,
\end{align*}
and the rule
\[
\nextcell_1(x, y), \nonnextpos_1(x, y) \rightarrow \exists z\; \error_1(z).
\]
Note that we also introduced a $\nonstart$ predicate, for later use.

For the $k^\text{th}$ order, with $k \in \{ 2, \dots, \lol\}$, the
positions of elements are represented by the two binary predicates
$\bit_{k, 0}$ and $\bit_{k, 1}$ pointing to the elements of the
$(k-1)^\textrm{st}$ order, which serve as bit positions.  Intuitively,
$\bit_{k,b}(x, y)$ says that the $y^\text{th}$ bit in the binary
encoding of the position of element  $x$ in the $k^\text{th}$ order is $b$.
We add the following rules, for every $k \in \{ 2, \dots, \lol\}$ and $b \in \{ 0, 1\}$:
\begin{align*}
	\bit_{k, b}(x, y) & \rightarrow \cell_k(x), \cell_{k-1}(y)
	\\
	\bit_{k,0}(x, y) & \rightarrow \nonend(x)
	\\
	\bit_{k,1}(x, y) & \rightarrow \nonstart(x)
	\\
	\bit_{k,0}(x, y), \bit_{k,1}(x, y) & \rightarrow \exists z\; \error_k(z).
\end{align*}

We need to guarantee that the encoding of positions is incremented
when moving along the predicate $\nextcell_k$, generally assuming that
the least significant bit position is the first element in the order.
We use a binary predicate $\firstzero_k$ (for First Zero) and the
following rules, for all $k \in \{ 2, \dots, \lol \}$:
\begin{align*}
	\cell_k(x), \nonend(x) & \rightarrow \exists y\; \firstzero_k(x, y)
	\\
	\firstzero_k(x, y) & \rightarrow \bit_{k, 0}(x, y), \ones_\triangleleft(x, y).
\end{align*}
The second rule makes sure that the position represented by $y$ has
value $0$ and that all positions to the left of $y$ have value~$1$.  The latter
is enforced by the binary predicate $\ones_\triangleleft$
which propagates to every position strictly to the left of $y$,
enforcing a bit value of $1$; see
Rule~\ref{rule:ones-left} in Figure~\ref{figure:rules-tower}.

The following rules introduce a ternary predicate $\nextfz_k$
extending each instance of $\nextcell_k(x, x')$ to further include the
position of the first zero in the encoding of $x$. We use
$\nextfz_k$ to properly set up the bit values in the encoding of the
position of $x'$.  Add, for every $k \in \{ 2, \dots, \lol \}$,
\begin{align*}
\nextcell_k(x, x') & \rightarrow \exists y\; \nextfz_k(x, x', y), \firstzero_k(x, y)
\\
\nextfz_k(x, x', y) 
& \rightarrow \bit_{k, 1}(x', y), \zeros_\triangleleft(x', y), \copytape_\triangleright(x, x', y).
\end{align*}
Predicate $\zeros_\triangleleft(x',y)$ enforces that all 1 bits to the left
of the first zero in the encoding of the position of $x$, which is at
position~$y$, are flipped to $0$s in the encoding of the position
of~$x'$.  The $0$ in position $y$ for $x$ is flipped to a $1$
for $x'$. All other positions keep their bit values thanks to predicate
$\copytape_\triangleright$ which instantiates a $\copytape$ that, in
turn, complies with the following rule for $b = 0, 1$:
\begin{align*}
	\copytape(x, x', y), \bit_{k, b}(x, y) \rightarrow \bit_{k, b}(x', y).
\end{align*}
Details 
can be found as Rules~\ref{rule:zeros-left} and \ref{rule:copy-right} in Figure~\ref{figure:rules-tower}.


As explained above, the $\lol^\text{th}$ order has the desired length
and we use its elements as positions of tape cells in the ATM
computation. We show in the appendix how to encode that
computation. The challenging part is to ensure that the tape
symbols that are not under the head are preserved when the ATM makes a
transition. This is enforced by a mechanism similar to the
propagation of the predicate $\ones_\triangleleft$ above.

The (straightforward) polynomial time reduction from
circumscribed AQ-querying to circumscribed consequence given
in the proof of Theorem~\ref{thm:fotwoconsvsaq} also applies to
GF. Thus, 
Theorem~\ref{theorem:tower-hardness} also yields the following.
\begin{restatable}{corollary}{corrconshardness}
  Circumscribed consequence in GF is \Tower-hard.
\end{restatable}

\section{FO$^2$ with Counting: C$^2$}

We observe that in C$^2$, circumscribed consequence and circumscribed
AQ-querying are decidable. This is achieved by combining a result from
\cite{DBLP:conf/frocos/WiesPK09} with ideas from
\cite{DBLP:conf/birthday/BonattiFLSW14}.

Recall that Presburger arithmetic is the first-order theory of the
natural numbers with addition and equality. BAPA is a multisorted
theory that combines Presburger arithmetic with the theory of
(uninterpreted) sets and their cardinalities. We refer to
\cite{DBLP:journals/jar/KuncakNR06} for full details and only remark
that numerical variables are denoted with $x,y,z$, set variables
with $B$, and set cardinality with $|B|$.
For a structure \Amf and a 1-type $t$, we write $t^\Amf$ to denote the
set of elements $\{a \in A \mid \mn{tp}^1_\Amf(a)=t \}$.  The
following was proved in \cite{DBLP:conf/frocos/WiesPK09}, making
intense use of the results of \cite{DBLP:journals/jolli/Pratt-Hartmann05}.
\begin{theorem}
  Let $\phi$ be a C$^2$ sentence and let $t_1,\dots,t_n$ be the
  1-types for $\phi$. One can compute a formula $\chi_\phi(x_1,\dots,x_n)$
  of Presburger arithmetic such that
  \begin{enumerate}

  \item for every model \Amf of $\phi$,
    $\chi_\phi[|t_1^\Amf|,\dots, |t_n^\Amf|]$ is true;

  \item if $\chi_\phi[k_1,\dots,k_n]$ is true, then there is a model \Amf
    of $\phi$ with $|t_i^\Amf|=k_i$ for $1 \leq i \leq n$.
    
  \end{enumerate}
\end{theorem}
The above provides a reduction from consequence in C$^2$ to
unsatisfiability in BAPA: the C$^2$ consequence $\phi \models \psi$
holds iff the BAPA sentence
$ \exists B_1 \cdots \exists B_n \, \chi_{\phi \wedge \neg \psi}
[|B_1|/x_1,\dots,|B_n|/x_n]$ is unsatisfiable. We extend this to
circumscribed consequence.

Assume that we want to decide $\phi \models_\CP \psi$, with
$\phi,\psi$ two C$^2$ sentences and $\CP = ({\prec},\Msf,\Fsf,\Vsf)$.   
We may assume w.l.o.g.\ that $\phi$ and $\psi$ contain the same
predicates and thus have the same 1-types.
With $\vartheta$, we denote the
BAPA formula
%
$$
\chi_{\phi} [|B_1|/x_1,\dots,|B_n|/x_n]
\wedge   \bigwedge_{A \in \Msf \cup \Fsf} \big (B_A = \!\! \bigcup_{t_i \mid A(x_1) \in
  t_i} \!\! B_i \big )
$$
and we define $\vartheta'$ to be like $\vartheta$, but using set
variables $B'_i$ in place of $B_i$ and $B'_A$ in place of $B_A$.
Let $\overline B$ be the tuple of set
variables in $\vartheta$ and let $\overline B'$ be the corresponding
tuple for~$\vartheta'$. We write $\overline B <_\CP \overline B'$ to
denote the conjunction of 
\begin{itemize}

\item $B_A = B'_A$ for all $A \in \Fsf$;

\item for all $A \in \Msf$:
  $
  \displaystyle
  B'_A \not\subseteq B_A \rightarrow \bigvee_{A' \in \Msf \mid A'
    \prec A} B'_{A'} \subsetneq  B_{A'};
  $

\item 
  $ \displaystyle
  \bigvee_{A \in \Msf} \big (B'_A \subsetneq B_A \rightarrow \bigvee_{A' \in \Msf \mid A'
    \prec A} B'_{A'} = B_{A'} \big ).
  $

\end{itemize}
Now let $\chi$ denote the BAPA sentence
$$
\begin{array}{r@{\,}l}
  \exists \overline B \, \big ( & \vartheta \wedge \chi_\psi
                                  [|B_1|/x_1,\dots,|B_n|/x_n] \,\wedge
  \\[1mm]
&  \forall \overline B' 
  \, \big (\overline B <_\CP \overline B' \rightarrow \neg \vartheta'
  \big ) \big ).
\end{array}
$$
It can be verified that $\phi \models_\CP \psi$ iff $\chi$ is
unsatisfiable. Since satisfiability in BAPA is decidable
\cite{Feferman1959,DBLP:journals/jar/KuncakNR06}, we obtain
decidability of circumscribed consequence in C$^2$. This
carries over to circumscribed AQ-querying in the same straightforward
way as for FO$^2$.
\begin{theorem}
  \label{thm:Ctwomain}
  In C$^2$, circumscribed consequence and circumscribed AQ-querying are
  decidable. 
\end{theorem}
Since BAPA is also decidable over finite models, we also obtain the
version of Theorem~\ref{thm:Ctwomain} where circumscribed
consequence and querying are defined w.r.t.\ finite models.

\section{Conclusion}
\label{sec-conclusion}

We have studied the impact on computational complexity of adding
circumscription to decidable fragments of first-order logic, which
turns out to be remarkably varied: while FO$^2$ is very tame and does
not have higher complexity than \ALC in its circumscribed version, GF
suffers from a dramatic complexity explosion. We remark that there is
a close connection between circumscription and querying with closed
predicates as studied in \cite{Ngo2016,DBLP:journals/lmcs/LutzSW19},
see also Example~\ref{ex:first}. More details are in
\cite{DBLP:conf/kr/LutzMN23}. As an example,
Theorem~\ref{theorem:exp-hardness} also applies to AQ-querying of
guarded existential rules with a single unary closed predicate. This,
in turn, is related to results in \cite{DBLP:conf/lics/BenediktBCP16}.

Several interesting questions remain open. What is the exact
complexity of circumscribed consequence in C$^2$? We speculate that by making careful
use of the techniques in \cite{DBLP:journals/jolli/Pratt-Hartmann05},
one can bring it down to $\coNExp^\NPclass$. What is the complexity of
circumscribed consequence in GF with only a single minimized predicate
or with multiple such predicates but no preference order? Is
circumscribed UCQ-querying in GF finitely controllable? What is the
complexity of circumscribed querying with less expressive classes of
existential rules such as inclusion dependencies? Is circumscribed
consequence decidable in the unary / guarded negation
fragments of FO?
	Note that satisfiability in the latter fragment is known to be reducible to UCQ-querying in GF \cite{DBLP:journals/jacm/BaranyCS15}, but that this reduction relies on arguments based on treeifications of some subformulas of interest, a technique that cannot be applied in presence of circumscription as discussed with Example~\ref{ex:second}.

\cleardoublepage

\section*{Acknowledgments}
 The authors acknowledge the financial support by the Federal Ministry of Education and Research of Germany 
 and by the Sächsische Staatsministerium für Wissenschaft Kultur und Tourismus in the program Center of Excellence for AI-research 
 ``Center for Scalable Data Analytics and Artificial Intelligence Dresden/Leipzig'', 
 project identification number: ScaDS.AI

This work is partly supported by BMBF (Federal Ministry of Education and Research)
in DAAD project 57616814 (\href{https://secai.org/}{SECAI, School of Embedded Composite AI})
as part of the program Konrad Zuse Schools of Excellence in Artificial Intelligence.

\bibliographystyle{kr}
\bibliography{main}

\begin{thebibliography}{}

\bibitem[\protect\citeauthoryear{Andr{\'{e}}ka, N{\'{e}}meti, and van
  Benthem}{1998}]{DBLP:journals/jphil/AndrekaNB98}
Andr{\'{e}}ka, H.; N{\'{e}}meti, I.; and van Benthem, J.
\newblock 1998.
\newblock Modal languages and bounded fragments of predicate logic.
\newblock {\em J. Philos. Log.} 27(3):217--274.

\bibitem[\protect\citeauthoryear{Baader and
  Hollunder}{1995}]{DBLP:journals/jar/BaaderH95}
Baader, F., and Hollunder, B.
\newblock 1995.
\newblock Embedding defaults into terminological knowledge representation
  formalisms.
\newblock {\em J. Autom. Reason.} 14(1):149--180.

\bibitem[\protect\citeauthoryear{B{\'{a}}r{\'{a}}ny, Gottlob, and
  Otto}{2014}]{BGO}
B{\'{a}}r{\'{a}}ny, V.; Gottlob, G.; and Otto, M.
\newblock 2014.
\newblock Querying the guarded fragment.
\newblock {\em J.\ of Log.\ Methods Comput.\ Sci.} 10(2).

\bibitem[\protect\citeauthoryear{B{\'{a}}r{\'{a}}ny, ten Cate, and
  Segoufin}{2015}]{DBLP:journals/jacm/BaranyCS15}
B{\'{a}}r{\'{a}}ny, V.; ten Cate, B.; and Segoufin, L.
\newblock 2015.
\newblock Guarded negation.
\newblock {\em J. {ACM}} 62(3):22:1--22:26.

\bibitem[\protect\citeauthoryear{Benedikt \bgroup et al\mbox.\egroup
  }{2016}]{DBLP:conf/lics/BenediktBCP16}
Benedikt, M.; Bourhis, P.; ten Cate, B.; and Puppis, G.
\newblock 2016.
\newblock Querying visible and invisible information.
\newblock In {\em Proc.\ of {LICS}},  297--306.
\newblock {ACM}.

\bibitem[\protect\citeauthoryear{Bonatti \bgroup et al\mbox.\egroup
  }{2015a}]{DBLP:conf/birthday/BonattiFLSW14}
Bonatti, P.~A.; Faella, M.; Lutz, C.; Sauro, L.; and Wolter, F.
\newblock 2015a.
\newblock Decidability of circumscribed description logics revisited.
\newblock In {\em Advances in Knowledge Representation, Logic Programming, and
  Abstract Argumentation - Essays Dedicated to Gerhard Brewka on the Occasion
  of His 60th Birthday}, volume 9060 of {\em LNCS},  112--124.
\newblock Springer.

\bibitem[\protect\citeauthoryear{Bonatti \bgroup et al\mbox.\egroup
  }{2015b}]{DBLP:journals/ai/BonattiFPS15}
Bonatti, P.~A.; Faella, M.; Petrova, I.~M.; and Sauro, L.
\newblock 2015b.
\newblock A new semantics for overriding in description logics.
\newblock {\em Artif. Intell.} 222:1--48.

\bibitem[\protect\citeauthoryear{Bonatti, Lutz, and Wolter}{2009}]{Bonatti2009}
Bonatti, P.~A.; Lutz, C.; and Wolter, F.
\newblock 2009.
\newblock The complexity of circumscription in description logic.
\newblock {\em J.\ Artif.\ Intell.\ Res.} 35:717--773.

\bibitem[\protect\citeauthoryear{Cal{\`{\i}}, Gottlob, and
  Lukasiewicz}{2009}]{DBLP:conf/pods/CaliGL09}
Cal{\`{\i}}, A.; Gottlob, G.; and Lukasiewicz, T.
\newblock 2009.
\newblock A general datalog-based framework for tractable query answering over
  ontologies.
\newblock In {\em Proc.\ of {PODS}},  77--86.
\newblock {ACM}.

\bibitem[\protect\citeauthoryear{Chandra, Kozen, and
  Stockmeyer}{1981}]{chandra-alternation}
Chandra, A.~K.; Kozen, D.~C.; and Stockmeyer, L.~J.
\newblock 1981.
\newblock Alternation.
\newblock {\em J. ACM} 28(1):114–133.

\bibitem[\protect\citeauthoryear{Donini, Nardi, and
  Rosati}{2002}]{DBLP:journals/tocl/DoniniNR02}
Donini, F.~M.; Nardi, D.; and Rosati, R.
\newblock 2002.
\newblock Description logics of minimal knowledge and negation as failure.
\newblock {\em {ACM} Trans. Comput. Log.} 3(2):177--225.

\bibitem[\protect\citeauthoryear{Feferman and Vaught}{1959}]{Feferman1959}
Feferman, S., and Vaught, R.
\newblock 1959.
\newblock The first order properties of products of algebraic systems.
\newblock {\em Fundamenta Mathematicae} 47(1):57--103.

\bibitem[\protect\citeauthoryear{Giordano \bgroup et al\mbox.\egroup
  }{2013}]{DBLP:journals/ai/GiordanoGOP13}
Giordano, L.; Gliozzi, V.; Olivetti, N.; and Pozzato, G.~L.
\newblock 2013.
\newblock A non-monotonic description logic for reasoning about typicality.
\newblock {\em Artif. Intell.} 195:165--202.

\bibitem[\protect\citeauthoryear{Gr{\"{a}}del and
  Otto}{2014}]{DBLP:books/daglib/p/Gradel014}
Gr{\"{a}}del, E., and Otto, M.
\newblock 2014.
\newblock The freedoms of (guarded) bisimulation.
\newblock In {\em Johan van Benthem on Logic and Information Dynamics}.
  Springer.
\newblock  3--31.

\bibitem[\protect\citeauthoryear{Gr{\"{a}}del, Kolaitis, and Vardi}{1997}]{GKV}
Gr{\"{a}}del, E.; Kolaitis, P.~G.; and Vardi, M.~Y.
\newblock 1997.
\newblock On the decision problem for two-variable first-order logic.
\newblock {\em Bull.\ of Symb.\ Log.} 3(1):53--69.

\bibitem[\protect\citeauthoryear{Gr{\"{a}}del, Otto, and
  Rosen}{1997}]{DBLP:conf/lics/GradelOR97}
Gr{\"{a}}del, E.; Otto, M.; and Rosen, E.
\newblock 1997.
\newblock Two-variable logic with counting is decidable.
\newblock In {\em Proc.\ of {LICS}},  306--317.
\newblock {IEEE} Computer Society.

\bibitem[\protect\citeauthoryear{Gr{\"{a}}del}{1999}]{DBLP:journals/jsyml/Gradel99}
Gr{\"{a}}del, E.
\newblock 1999.
\newblock On the restraining power of guards.
\newblock {\em J. Symb. Log.} 64(4):1719--1742.

\bibitem[\protect\citeauthoryear{Kuncak, Nguyen, and
  Rinard}{2006}]{DBLP:journals/jar/KuncakNR06}
Kuncak, V.; Nguyen, H.~H.; and Rinard, M.~C.
\newblock 2006.
\newblock Deciding boolean algebra with presburger arithmetic.
\newblock {\em J. Autom. Reason.} 36(3):213--239.

\bibitem[\protect\citeauthoryear{Lutz, Mani{\`{e}}re, and
  Nolte}{2023}]{DBLP:conf/kr/LutzMN23}
Lutz, C.; Mani{\`{e}}re, Q.; and Nolte, R.
\newblock 2023.
\newblock Querying circumscribed description logic knowledge bases.
\newblock In {\em Proc.\ of {KR}},  482--491.

\bibitem[\protect\citeauthoryear{Lutz, Seylan, and
  Wolter}{2019}]{DBLP:journals/lmcs/LutzSW19}
Lutz, C.; Seylan, I.; and Wolter, F.
\newblock 2019.
\newblock The data complexity of ontology-mediated queries with closed
  predicates.
\newblock {\em Log. Methods Comput. Sci.} 15(3).

\bibitem[\protect\citeauthoryear{Mortimer}{1975}]{DBLP:journals/mlq/Mortimer75}
Mortimer, M.
\newblock 1975.
\newblock On languages with two variables.
\newblock {\em Math. Log. Q.} 21(1):135--140.

\bibitem[\protect\citeauthoryear{Ngo, Ortiz, and Simkus}{2016}]{Ngo2016}
Ngo, N.; Ortiz, M.; and Simkus, M.
\newblock 2016.
\newblock Closed predicates in description logics: Results on combined
  complexity.
\newblock In {\em Proc.\ of {KR}},  237--246.
\newblock AAAI Press.

\bibitem[\protect\citeauthoryear{Pacholski, Szwast, and
  Tendera}{1997}]{DBLP:conf/lics/PacholskiST97}
Pacholski, L.; Szwast, W.; and Tendera, L.
\newblock 1997.
\newblock Complexity of two-variable logic with counting.
\newblock In {\em Proc.\ of {LICS}},  318--327.
\newblock {IEEE} Computer Society.

\bibitem[\protect\citeauthoryear{Pratt{-}Hartmann}{2005}]{DBLP:journals/jolli/Pratt-Hartmann05}
Pratt{-}Hartmann, I.
\newblock 2005.
\newblock Complexity of the two-variable fragment with counting quantifiers.
\newblock {\em J. Log. Lang. Inf.} 14(3):369--395.

\bibitem[\protect\citeauthoryear{Rosati}{2006}]{DBLP:conf/pods/Rosati06}
Rosati, R.
\newblock 2006.
\newblock On the decidability and finite controllability of query processing in
  databases with incomplete information.
\newblock In {\em Proc.\ of {PODS}},  356--365.
\newblock {ACM}.

\bibitem[\protect\citeauthoryear{Schmitz}{2016}]{schmitz-complexity-beyond-elementary}
Schmitz, S.
\newblock 2016.
\newblock Complexity hierarchies beyond elementary.
\newblock {\em ACM Trans. Comput. Theory} 8(1).

\bibitem[\protect\citeauthoryear{Scott}{1962}]{Scott1962}
Scott, D.
\newblock 1962.
\newblock A decision method for validity of sentences in two variables.
\newblock {\em J.\ of Symb.\ Log.} 27.

\bibitem[\protect\citeauthoryear{Stefano, Ortiz, and
  Simkus}{2023}]{DBLP:conf/ijcai/Stefano0S23}
Stefano, F.~D.; Ortiz, M.; and Simkus, M.
\newblock 2023.
\newblock Description logics with pointwise circumscription.
\newblock In {\em Proc.\ of {IJCAI}},  3167--3175.
\newblock ijcai.org.

\bibitem[\protect\citeauthoryear{Wies, Piskac, and
  Kuncak}{2009}]{DBLP:conf/frocos/WiesPK09}
Wies, T.; Piskac, R.; and Kuncak, V.
\newblock 2009.
\newblock Combining theories with shared set operations.
\newblock In Ghilardi, S., and Sebastiani, R., eds., {\em Proc.\ of {FroCoS}},
  volume 5749 of {\em LNCS},  366--382.
\newblock Springer.

\end{thebibliography}

\appendix 

\section{Proofs for Section~\ref{sect:fotwo}}

\lemsmallmodelsfotwo*

\begin{proof}
  Let $\phi$, $\Sigma$, \Amf, and $k$ be as in the proposition.  
  We say that $a \in A$ is a \emph{king} if
  $$
     |\{ b \in A \mid \mn{tp}^1_{\Amf,\Sigma}(a)=\mn{tp}^1_{\Amf,\Sigma}(b) \}| \leq 3n_\exists.
  $$
  Note that this is different from the definition of kings in
  \cite{GKV} where $3n_\exists$ is replaced with 1. Let $K$ be
  the set of all kings in \Amf.

  For every $a \in A$ and $i \in \{1,\dots,n_\exists\}$, choose an
  $a^i \in A$ such that $\Amf \models \psi_i[a,a^i]$.

  An \emph{extended 1-type} is a pair $(t,S)$ where $t$ is a 1-type
  and $S$ is a set of pairs $(T,a)$ with $T$ a 2-type and $a \in K$.
  The cardinality of $S$ must be at most $n_\exists$. For $a \in A$, the extended 1-type \emph{realized
    at}~$a$, denoted $\mn{etp}^1_{\Amf,\Sigma}(a)$, is $(t,S)$ where
  $\mn{tp}^1_{\Amf,\Sigma}(a)=t$ and
  $$
    S= \{ (\mn{tp}^2_{\Amf,\Sigma}(a,a^i),a^i) \mid 1 \leq i \leq n_\exists \text{
      and } a^i \text{ is a king} \}.
  $$
  Let \Tmf be the set of all extended 1-types $(t,S)$ realized in \Amf
  such that
  $$|\{ a \in A \mid \mn{tp}^1_{\Amf,\Sigma}(a)=t \}| > 3n_\exists.$$
  The universe $B$ of \Bmf is
  $$
    K \cup \{ (t,S,i,j) \mid (t,S) \in \Tmf, 1 \leq i \leq n_\exists, j \in \{0,1,2\}\}.
  $$
  Note that, in contrast to the construction in \cite{GKV}, the model
  \Bmf does not comprise a royal court. This is possible since we work
  with extended 1-types.

  For every constant $c$ in $\phi$, there is a unique $a \in A$ with
  $x=c$ contained in $\mn{tp}^1_{\Amf,\Sigma}(a)$. It follows that $a$ is a
  king and thus we may set $c^\Bmf=a$.

  It remains to define the extension of the relation symbols in \Bmf.
  To this end, it suffices to define all 1-types and 2-types in \Bmf.
  We set $\mn{tp}^1_{\Bmf,\Sigma}(a)=\mn{tp}^1_{\Amf,\Sigma}(a)$ for
  all $a \in K$ and $\mn{tp}^1_{\Bmf,\Sigma}(a)=t$ for all
  $(t,S,i,j) \in B$. It is clear that this achieves Points~1-3 of the
  proposition.

  We further set $\mn{tp}^2_{\Bmf,\Sigma}(a,b)=\mn{tp}^2_{\Amf,\Sigma}(a,b)$
  for all \mbox{$a,b \in K$}, that is, the restrictions of \Amf and \Bmf to
  universe $K$ are identical. This leaves us with determining the 2-types
  of pairs $(a,b)$ that involve at least one non-king. This proceeds
  in two steps.

  In the first step, we connect kings to non-kings as recorded in the
  extended 1-types of the latter: if $(t,S,i,j) \in B$ and
  $(T,a) \in S$, then $\mn{tp}^2_{\Bmf,\Sigma }((t,S,i,j),a)=T$. It can be
  verified that after this step, kings have all required existential
  witnesses. If, in fact, $a \in K$ and $i \in \{1,\dots,n_\exists\}$,
  then either $a^i \in K$ and thus $\Bmf \models \psi_i[a,a^i]$, or
  $\Bmf \models \psi_i[a,(t,S,i,j)]$ for $(t,S)=\mn{etp}^1_{\Amf,\Sigma}(a^i)$
  and any $i \in \{1,\dots,n_\exists\}$ and $j \in \{0,1,2\}$.

  In the second step, we make sure that the existential demands of
  non-kings are also satisfied. Let $(t,S,i,j) \in B$ and
  $1 \leq i \leq n_\exists$. Choose any $a \in A$ with
  $\mn{etp}^1_{\Amf,\Sigma}(a)=(t,S)$. If $a^i$ is a king, then
  $\Bmf \models \psi_i[(t,S,i,j),a^i]$ and there is nothing to
  do. Otherwise, we set
  $\mn{tp}^2_{\Bmf,\Sigma}((t,S,i,j),(t',S',i,j+1 \mod
  3))=\mn{tp}^2_{\Amf,\Sigma}(a,a^i)$ for $(t',S')=\mn{etp}^1_{\Amf,\Sigma}(a^i)$.

  Now that all existential demands are satisfied, it remains to `fill
  up' the remaining 2-types. First assume that
  $\mn{tp}^2_{\Bmf,\Sigma}(a,(t,S,i,j))$ was not yet determined. Then choose
  any $b \in A$ with $\mn{etp}^1_{\Amf,\Sigma}(b)=(t,S)$ and set
  $\mn{tp}^2_{\Bmf,\Sigma}(a,(t,S,i,j)) = \mn{tp}^2_{\Amf,\Sigma}(a,b)$. Now assume that
  $\mn{tp}^2_{\Bmf,\Sigma}((t,S,i,j),(t',S',i',j'))$ was not yet
  determined. Then choose any $a,a' \in A$ with
  $\mn{etp}^1_{\Amf,\Sigma}(a)=(t,S)$ and $\mn{etp}^1_{\Amf,\Sigma}(a')=(t',S')$ and set
  $\mn{tp}^2_{\Bmf,\Sigma}((t,S,i,j),(t',S',i',j')) = \mn{tp}^2_{\Amf,\Sigma}(a,a')$.

  The reader may verify that no conflicts arise when
  determining the 2-types in \Bmf above, that is, that we do not
  assign two different 2-types to the same pair of elements. This
  relies on the use of extended types and of the $i$- and
  $j$-components of the non-kings in $B$. In fact, we use the latter
  in exactly the same way as in the construction in \cite{GKV}.

  It is
  now straightforward to show that \Bmf is indeed a model of~$\phi$,
  using the fact that \Amf is a model of $\phi$.

  \medskip It remains to analyze the size of \Bmf. The number of
  1-types realized in \Amf is bounded by $|\mn{const}(\phi)| + 2^k$:
  for every $c \in \mn{const}(\phi)$ there is exactly one 1-type $t$
  with $x_1=c \in t$ that is realized in \Amf, and the number of
  1-types that contain $x_1 \neq c$ for all $c \in \mn{const}(\phi)$
  is bounded by $2^k$. The number of kings is thus bounded by
  $$
  |\mn{const}(\phi)| + 2^k \cdot n_\exists \leq |\phi|+2^k
  \cdot |\phi| \leq |\phi| \cdot 2^{{k+1}}.
  $$
  We next analyze the size of \Tmf, that is, the number of extended
  1-types $(t,S)$ realized in \Amf. The number of options for $t$ is
  bounded by $|\phi| + 2^k \leq |\phi| \cdot 2^{k+1}$. Regarding
  $S$, we may choose at most $n_\exists$ many kings $a^i$ for which we
  include the pair $(\mn{tp}^2_{\Amf, \Sigma}(a,a^i),a^i)$. 
  Note that the two 1-types (on variable $x_1$ and on variable $x_2$)
  in the 2-type $\mn{tp}^2_{\Amf, \Sigma}(a,a^i)$ are already
  determined
  by the choice of $t$ and $a^i$, and thus only $2^{2k}$ choices
  for the 2-type remain.
  The number of
  choices for $S$
  is thus bounded by
  $$
  ((|\phi| \cdot 2^{{k+1}} \cdot 2^{2k})+1)^{n_\exists} \leq |\phi|^{n_\exists}
\cdot 2^{n_\exists 3(k+2)}
  $$
  and we obtain an overall upper bound on \Tmf of
  $$
  (|\phi| \cdot 2^{k+1}) \cdot (|\phi|^{n_\exists}
\cdot 2^{n_\exists 3(k+2)}) \leq |\phi|^{n_\exists+1} \cdot
2^{n_\exists 4(k+2)}.
  $$
  We may thus bound from above the number of elements in $B$ by
  $$
  |\phi|^{n_\exists+1} \cdot 2^{n_\exists 4(k+2)} \cdot 3n_\exists \leq
    |\phi|^{n_\exists+1} \cdot 2^{n_\exists 4(k+6)}.
    $$
\end{proof}

\section{Proofs for Section~\ref{sec-circcons}, Except Lemma~\ref{lemma-non-uniform}}

\lemsmallcircmodelgf*
\begin{proof}
  Let $\phi$ be a satisfiable GF sentence and $\CP = ({\prec}, \Msf, \Fsf, \Vsf)$ a circumscription pattern. 
  Take a $\CP$-minimal model \Amf of $\phi$. Then \Amf is a model
  of $\phi$ and by Proposition~\ref{prop:smallmodelsgf}, using $\sig = \sigmin(\phi) \cup \Msf \cup \Fsf$, there is a
  model \Bmf of $\phi$ that satisfies Points~1-3 of the proposition.
  We show that \Bmf is also a $\CP$-minimal model of $\phi$.

  Assume to the contrary that there is a model $\Bmf'$ of $\phi$
  such that $\Bmf' <_\CP \Bmf$. To obtain a contradiction, we
  construct a model $\Amf'$ of $\phi$ such that $\Amf' <_\CP
  \Amf$.

  Of course, $\Amf'$ must have the same universe as $\Amf$, thus
  \mbox{$A'=A$}. By Point~3 of Proposition~\ref{prop:smallmodelsgf},
  we find an injection $f:B \rightarrow A$ such that
  $\tp[\sig]<1>[\Bmf]{b} = \tp[\sig]<1>[\Amf]{f(b)}$ for all $b \in B$. We
  define $\Amf'$ so that its restriction to the range of $f$ is
  isomorphic to $\Bmf'$, with $f$ being an isomorphism. Note in
  particular that $\Amf'$ interprets all constants that occur in
  $\phi$ exactly as $\Bmf'$ does (up to the isomorphism $f$), thus, by Point~4 of Proposition~\ref{prop:smallmodelsgf} and the definition of $\Bmf' <_\CP \Bmf$, exactly as $\Amf$ does.
  To define the remaining part of $\Amf'$,
  we use cloning.

  By Point~2 of
  Proposition~\ref{prop:smallmodelsgf} and choice of $f$, for every
  $a\in A$ that is not in the range of $f$ we find a
  $\widehat a \in A$ that is in the range of $f$ and such that
  $\tp[\sig]<1>[\Amf]{\widehat a} = \tp[\sig]<1>[\Amf]{a}$. We then make $a$ a
  clone of $\widehat a$ in $\Amf'$, that is, for every relation symbol
  $R$, we set
$$
R^{\Amf'} = R^{\Amf'} \cup \{ \bar a[a/\pi(b')] \mid \bar a \in R^{\Amf'} \}
$$
where $\bar a[a/\pi(b')]$ denotes the tuple obtained from $\bar a$ by
replacing every occurrence of $\widehat a$ with $a$. This finishes the
construction of the structure $\Amf'$.

Using the fact that $\Bmf'$ is a model of $\phi$, it is now
straightforward to show that, as desired, also $\Amf'$ is a model
of~$\phi$. 
Note in particular that the cloning does not duplicate any constants occurring in some equality atom of $\phi$, that is, if we make element $a$ a clone of $\widehat a$,
then there is no constant $c \in \sigmin(\phi)$ with $c^{\Amf'}=\widehat a$. 
This can be
seen as follows.
By contradiction, assume that we can find such a constant $c$.
Since $\widehat{a}$ is in the range of $f$, the definition of $c^{\Amf'}$, gives $f(c^{\Bmf'}) = \widehat{a}$.
From $\Bmf' <_\CP \Bmf$, we obtain $c^{\Bmf'} = c^\Bmf$.
Now from the definition of $f$, we have $\tp[\sig]<1>[\Bmf]{c^\Bmf} = \tp[\sig]<1>[\Amf]{f(c^\Bmf)}$, \emph{i.e.}\ $\tp[\sig]<1>[\Bmf]{c^\Bmf} = \tp[\sig]<1>[\Amf]{\widehat{a}}$.
By virtue of $c \in \sigmin(\phi) \subseteq \sig$, we get $c^\Amf = \widehat{a}$ and in particular $\widehat{a}$ is the only element of $\Amf$ with type $t := \tp[\sig]<1>[\Amf]{\widehat{a}}$, since $t$ contains the atom $x_1 = c$.

It can also be verified that $\Amf' <_\CP \Amc$, since $\Bmf'
  <_\CP \Bmf$ and $\Msf, \Fsf \subseteq \sig$.
\end{proof}

\thmsatdec*
\begin{proof}
%
  Let $\sig = \sigmin(\phi) \cup \Msf \cup \Fsf$.  Our algorithm
  iterates over all structures \Amf (in the signature of $\phi$) with
  universe $A \subseteq \Delta$, where $\Delta$ is a fixed set of size
  $\towerof{4^{\sizeof{\sig} + 4}}{\sizeof{\phi}}$.  For each such
  \Amf, it checks whether
  (i)~$\Amf\models\phi$, (ii)~$\Amf \not\models\psi$,
  and (iii)~$\Amf$ is $\CP$-minimal.  If such an $\Amf$ is found, it
  returns `no', otherwise it returns `yes'.

  By Theorem~\ref{lem:smallcircmodelgf}, the algorithm returns `no'
  iff $\phi \not\models_\CP \psi$. It can be implemented by a
  $\towerof{4^{\sizeof{\sig} + 4}}{\sizeof{\phi} +
    \sizeof{\psi}}$-space bounded Turing machine.  In particular,
  Point~(iii) can be decided by iterating over all interpretations
  $\Bmf$ with \ $B = A$ and checking whether $\Bmf \models \phi$ and
  $\Bmf <_\CP \Amf$.
%
\end{proof}

\section{Proof of Lemma~\ref{lemma-non-uniform}}

\newcommand{\sigext}{\widehat{\sig}}

We now prove Lemma~\ref{lemma-non-uniform}. Let $\phi$ be a GF
sentence, $\sig$ a unary signature that contains $\sigmin(\phi)$, \Amf a model of $\phi$, and $\Delta \subseteq A$ with $c^\Amf \in \Delta$ for all $c \in \const(\phi)$.

In the following, we actually focus on signature $\sigext := \sig \cup \Delta \cup \signature(\phi)$.
We argue it is sufficient for our purpose: Points~1 and 4 are independent from $\sig$, and Points~2 and 3 being valid for $\sigext$ clearly implies the same properties for $\sig$ since $\sig \subseteq \sigext$.
Note in particular that $\sigext$ needs not to be unary and that a type on the $\sigext$ stores full information about the
restriction of \Amf to its subdomain made of $\Delta$ and
the constants in $\phi$.

A set of elements $S \subseteq A \setminus \Delta$ is \emph{guarded}
if it is a singleton or there is a relation symbol $R$ such that some
tuple in $R^{\mathfrak{A}}$ contains all elements of $S$. A tuple of
elements $\bar a$ is \emph{guarded} if the set of elements that occur
in $\bar a$ is guarded. A guarded tuple $\bar a$ is \emph{maximal} if
there is no guarded tuple $\bar b$ such that the set of elements in
$\bar a$ is a strict subset of the set of elements in~$\bar b$. We use
$\mn{mgt}(\Amf)$ to denote the set of all maximal guarded tuples in
\Amf that do not contain repeated elements and $\mn{gtp}(\Amf)$ to
denote the set of all types on $\sigext$ realized by maximal guarded tuples in
$\Amf$.

We now move to the actual proof of Lemma~\ref{lemma-non-uniform}.  Let
$\structure{A}$ be a model of $\phi$ with
$\Delta \subseteq \domain{A}$. Notice that if $\Delta = \domain{A}$,
then $\structure{A}$ itself already satisfies all conditions from
Lemma~\ref{lemma-non-uniform} and there is nothing to do.  In the
following, we thus assume $\Delta \subsetneq \domain{A}$, which
guarantees that both $\mn{mgt}(\structure{A})$ and
$\mn{gtp}(\structure{A})$ are non-empty.  The maximal arity of a
predicate in $\phi$ is denoted~$w$, and we assume w.l.o.g.\ that
$w > 1$. 

For two tuples $\bar a, \bar b$ without repeated elements,
 we use
$\rho(\bar a,\bar b)$ to denote the injective substitution
$\rho : \{ x_1, \dots, x_\sizeof{\bar a} \} \rightarrow \{ x_1, \dots,
x_\sizeof{\bar b} \}$ that maps $x_i$ to $x_j$ iff $a_i =
b_j$.
Note that
this is well-defined since $\bar b$ does not contain repeated
elements and injective since $\bar a$ does not contain repeated
elements.  We use $\mn{dom}(\rho)$ and $\mn{ran}(\rho)$ to denote
the domain and range of $\rho$.

Let $K = w^{4}$. For an $n$-type $t$ on $\sigext$, we use $\mn{var}(t)$ to
denote the set of variables $\{x_1,\dots,x_n\}$. We introduce the following constants and
function symbols, for $0 \leq j < K$:
\begin{itemize}
	
	\item for all $n$-types $t \in \mn{gtp}(\Amf)$ and $i \in [n]$, a
	fresh constant symbol $c^j_{t, i}$;
	
	\item for all types $t_1,t_2 \in \mn{gtp}(\Amf)$, partial functions
	$\rho: \mn{var}(t_1) \rightarrow \mn{var}(t_2)$, and
	$x_i \in \mn{var}(t_2) \setminus \mn{img}(\rho)$, a
	function symbol $f^j_{t_1,\rho,t_2,i}$ of arity $|\mn{dom}(\rho)|$.
	
\end{itemize}
%
For each term $\theta$, let
$J(\theta)$ denote the set of values that occur in the superscript of
a constant or function symbol in $\theta$. This notion
extends to tuples of terms in the expected way. The \emph{truncation}
of a term $\theta$ at depth $k \in \{0,1\}$, denoted $\theta/_k$, is defined as
follows:
$$
\begin{array}{rcll}
	c^j_{t,i}/_k &=& c^j_{t,i}
	\\[1mm]
	f^j_{t_1,\rho,t_2,i}(\bar\theta)/_0 &=& c^j_{t_2,i} \\[1mm]
	f^j_{t_1,\rho,t_2,i}(\bar\theta)/_1 &=&
	f^j_{t_1,\rho,t_2,i}(\bar\theta/_0).
\end{array}
$$
We use $\mn{trunc}(\theta)$ to denote the the term $\theta/_1$.

\medskip

For each type $t \in \mn{gtp}(\Amf)$, we define sets
$\mathcal{K}^r(t)$ 
of instances of $t$
at height $r$.


\paragraph*{Initialization $r = 0$.}
For each $n$-type $t \in \mn{gtp}(\Amf)$, we set:
\[
\mathcal{K}^0(t) = \{ \bar c_t^j \mid 0 \leq j < K \},
\]
where $\bar c_t^j$ stands for the tuple of constants $(c_{t, 1}^j, \dots, c_{t, n}^j)$.


\paragraph*{From $r$ to $r+1$.}

Assume that we have already constructed the sets $\mathcal{K}^r(t)$
for all $t \in \mn{gtp}(\Amf)$. We define the set
$\mathcal{K}^{r+1}(t)$, for all $t \in \mn{gtp}(\Amf)$, by starting
with the empty set and then applying
the following saturation step as long as possible.

For each type $s$, each partial bijection $\rho$, each
$\bar b \in \mathcal{K}^{r}(s)$, and each index
$j \in \{0,\dots,K \} \setminus J(\bar b |_{\textrm{dom}(\rho)})$ such
that there are $\bar u, \bar v \in \mn{mgt}(\structure{A})$ with
$\tp[\sigext][\Amf]{\bar u} = s$, $\tp[\sigext][\Amf]{\bar v} = t$, and
$\rho(\bar u,\bar v) = \rho$,
extend $\mathcal{K}^{r+1}(t)$ with the tuple $\bar a$ whose $i^\textit{th}$ component $a^j_i$ is defined as:
\[
a_i = 
\left\{
\begin{array}{l@{\;}l}
b_\ell & \textrm{ if } (x_\ell , x_i) \in \rho 
\\
f^j_{s, \rho, t, i}(\bar b|_{\mn{dom}(\rho)}/_1) & \textrm{ if } x_i \notin \textrm{img}(\rho)
\end{array}
\right.
\]

Note that, because of truncation, the maximum height of each term
$a_i$ is two.
	Furthermore, it is easily verified by induction on $r$ that each tuple $\bar a \in \Kmc^r(t)$, and even its truncation $\bar a /_1$, is without repetition.
	This property crucially relies on the condition $j \notin J(\bar b |_{\textrm{dom}(\rho)})$ dictating the permitted function symbol at the root of each term.

\begin{lemma}
	\label{lem:no-repetition}
		Let $\bar a \in \Kmc^m(t)$ for some $m \geq 0$ and $t \in \mn{gtp}(\structure{A})$.
		The tuple $\bar a/_0$, and a fortiori tuples $\bar a/_1$ and $\bar a$, do not contain repeated elements, that is $a_i/_0 = a_k/_0$ iff $i = k$. 
\end{lemma}
\begin{proof}
	We proceed by induction on $r$.
	For $r = 0$, this is immediate by definition of $\Kmc^0(t)$.
	For the induction step, assume the claim holds for all integers smaller than some $r > 0$.
	Let $\bar a \in \Kmc^r(t)$ for some $t$.
	Since $r > 0$, there is some type $s$
	and $\bar b \in \Kmc^{r-1}(s)$
	from which $\bar a$ has been obtained
	in the construction of $\Kmc^r(t)$, that is, there exist $\bar u, \bar v \in \mn{mgt}(\structure{A})$ with $\tp[\sigext][\Amf]{\bar u} = s$, $\tp[\sigext][\Amf]{\bar v} = t$ and for all $i = 1, \dots, \sizeof{\bar a}$:
	\[
	a_i = 
	\left\{
	\begin{array}{l@{\;}l}
		b_\ell & \textrm{ if } (x_\ell , x_i) \in \rho
		\\
		f^{j}_{s, \rho, t, i}(\bar b|_{\mn{dom}(\rho)}/_1) & \textrm{ if } x_i \notin \textrm{img}(\rho)
	\end{array}
	\right.
	\]
	where $\rho = \rho(u, v)$ and $j \notin J(\bar b|_{\textrm{dom}(\rho)})$.
	By induction hypothesis, it is clear that if $x_i, x_k \in \mn{img}(\rho)$, then $a_i/_0 = a_k/_0$ iff $i = k$.
	In the case $x_i, x_k \notin \mn{img}(\rho)$, then we have $a_i/_0 = a_k/_0$ iff $f^{j}_{s, \rho, t, i}(\bar b|_{\mn{dom}(\rho)}/_1)/_0 = f^{j}_{s, \rho, t, k}(\bar b|_{\mn{dom}(\rho)}/_1)/_0$ iff $c^j_{t, i} = c^j_{t, k}$ iff $i = k$.
	It remains to treat the case of $x_i \in \mn{img}(\rho)$ and $x_k \not\in \mn{img}(\rho)$, for which we clearly have $i \neq k$ and thus aim to prove $a_i/_0 \neq a_k/_0$.
	Note that $a_k/_0 = c^j_{t, k}$ while $a_i/_0$ has shape $c^{j'}_{t', i'}$ for some $j' \in J(b_\ell)$ with $\ell$ such that $\rho(x_\ell) = x_i$.
	From $b_\ell \in \bar b|_{\mn{dom}(\rho)}$ and $j \notin J(\bar b|_{\mn{dom}(\rho)})$, we obtain $a_i/_0  \neq a_k/_0$. 
\end{proof}

%
%


\smallskip

We
now define the
structure $\Bmf$. The universe $\domain{B}$ consists of all elements
in $\Delta$ and all elements that appear in a tuple in one of the sets
$\Kmc^r(t)$. Set 
$$
\begin{array}{r@{}c@{}l}
R^\Bmf &=& \displaystyle \bigcup_{r\geq 0} \bigcup_{t \in
	\mn{gtp}(\Amf)} \{ [\bar a/\bar e] \mid
R(\bar e) \in t
\text { and } \bar a \in \Kmc^r(t) \}.
\end{array}
$$
for each relation symbol $R$ where $[\bar a/\bar e]$ denotes the
tuple obtained from $\bar e$ by replacing every variable $x_i$
with the $i$-th element $a_i$ of $\bar a$. 
%
Note that 
$\bar e$ may involve elements from $\Delta$ which are left unchanged.
For each constant~$c$, 
set $c^\Bmf = c^\Amf$.

\smallskip

We remark that the above model construction is exactly the one from
\cite{BGO} except that we have defined types and the set
$\mn{gtp}(\Amf)$ in a different way, giving a special role to the
elements of $\Delta$. Moreover, we truncate terms at height~1
while~\cite{BGO}, which additionally deals with conjunctive queries,
truncates at larger heights. Without truncation, the above
construction falls back to being a standard unraveling
construction. Such a construction, however, does not deliver finite
models and truncation is used as a way to reuse elements so as to
achieve finiteneness. One has to be careful, though, because reusing
elements in a too agressive way will result in different guarded types
to overlap, which effectively generates guarded types that are not
present in the original structure \Amf and may result in the
constructed structure \Bmf to not be a model of $\phi$. An extreme
case of such aggressive reuse is the filtration technique from modal logic
which is well-known to fail for the guarded fragment.  The indexes
$\cdot^j$ result in several copies of each element to be introduced
which makes it possible to avoid too aggressive reuse when truncating.

\smallskip

We have to show that \Bmf is a model of $\phi$ and satisfies
Conditions~1-4 from Lemma~\ref{lemma-non-uniform}.

\medskip

The following lemma establishes a central `compatibility property' for
the sets $\Kmc^m(t)$. If we find tuples $\bar a \in \Kmc^m(t)$ and
$\bar a' \in \Kmc^n(t')$ that overlap, then $t$ and $t'$ agree
regarding the positions in which $\bar a$ and $\bar a'$ overlap. If
that property would fail, then our construction of \Bmf could generate types
that did not exist in the original model \Amf, and this would lead to
\Bmf not necessarily being a model of $\phi$.
\begin{lemma}
  \label{lem:formerclaim}
  Let $\bar a \in \Kmc^m(t)$, $\bar a' \in \Kmc^n(t')$ and let $\sigma = \rho(\bar a, \bar a')$.
  If $R(\bar e) \in t$ and all variables in $\bar e$ are in $\mn{dom}(\sigma)$, then $R(\sigma \bar e) \in t'$.

%
\end{lemma}
\begin{proof}
  The proof is by a nested induction on $m$ and $n$. 

  \smallskip

  For the outer
  induction start, assume that $m=0$.

\smallskip

				For the inner induction start, assume that $n=0$.
				Let $\bar a \in \Kmc^0(t)$ and $\bar a' \in \Kmc^0(t')$ for some $t$ and $t'$, and let $\sigma = \rho(\bar a, \bar a')$.
				By definition of $\Kmc^0(t)$ and
                                $\Kmc^0(t')$, there are $j$ and $j'$
                                such that $\bar a = \bar c^j_{t}$ and
                                $\bar a' = \bar c^{j'}_{t'}$. 
				Let $R(\bar e) \in t$ such that all variables in $\bar e$ are in $\mn{dom}(\sigma)$.

                                If $\bar e$ indeed contains a variable
                                $x_i$, say  with $\sigma x_i = x_{i'}$, then from $a_i=a'_{i'}$ we obtain $c^j_{t,i} = c^{j'}_{t', i'}$.
				Therefore, $t = t'$, $j = j'$, $i = i'$, thus $\bar a = \bar a'$ and $\sigma$ is the identity.
				It follows trivially that $R(\sigma
                                \bar e) \in t'$.
                                
				Otherwise there is no variable in
                                $\bar e$, and thus $\bar e$ only
                                involves elements from $\Delta$. 
				Since $t$ is a $\Delta$-extended
                                type from $\mn{gtp}(\structure{A})$,
                                by definition of such types we have
                                $\structure{A} \models R(\bar e)$, which in turn yields $R(\bar e) \in t'$.


                                \smallskip
                                Now for the (inner) induction step.
				Assume that the claim holds for $m=0$
                                and all integers smaller than some $n >0$.
				Let $\bar a \in \Kmc^0(t)$ and $\bar a' \in \Kmc^n(t')$ for some $t$ and $t'$ and let $\sigma = \rho(\bar a, \bar a')$.
				As $m=0$, there is again a $j$ such that $\bar a = \bar c^j_{t}$.
				Since $n > 0$, there is some type $s'$
                                and $\bar b' \in \Kmc^{n-1}(s')$
                                from which $\bar a'$ has been obtained
                                in the construction of $\Kmc^n(t')$, that is, there exist $\bar u, \bar v \in \mn{mgt}(\structure{A})$ with $\tp[\sigext][\Amf]{\bar u} = s'$, $\tp[\sigext][\Amf]{\bar v} = t'$ and for all $i = 1, \dots, \sizeof{\bar a'}$:
				\[
				a'_i = 
				\left\{
				\begin{array}{l@{\;}l}
					b'_\ell & \textrm{ if } (x_\ell , x_i) \in \rho
					\\
					f^{j'}_{s', \rho, t', i}(\bar b'|_{\mn{dom}(\rho)}/_1) & \textrm{ if } x_i \notin \textrm{img}(\rho)
				\end{array}
				\right.
				\]
				where $\rho = \rho(u, v)$ and $j' \notin J(\bar b' |_{\textrm{dom}(\rho)})$.
				Since $\bar a$ does not contain any function symbol, it must be that every variable in the image of $\sigma$ is also in the image of $\rho$, that is $\mn{img}(\sigma) \subseteq
				\mn{img}(\rho)$.
				Since $\rho$ is injective, we can consider $\rho^{-1} \circ \sigma$, which is the restriction to $\mn{dom}(\sigma)$ of $\sigma' := \rho(\bar a, \bar b')$.
				Let $R(\bar e) \in t$ such that all variables in $\bar e$ are in $\mn{dom}(\sigma)$ and, {a fortiori}, in $\mn{dom}(\sigma')$.
				By induction hypothesis, we obtain $R(\sigma' \bar e) \in s'$, thus $R(\rho^{-1} \sigma \bar e) \in s'$.
%
				Now, from $\typeinof{A}{\bar u} =
                                s'$, it follows that $\Amf
                                \models R(\rho^{-1} \sigma \bar e)[\bar u]$. 
                                Recalling $\rho = \rho(\bar u, \bar
                                v)$, this
                                yields $\structure{A} \models
                                R(\sigma \bar e)[\bar v]$.
				Recalling that $\typeinof{A}{\bar v} = t'$ now guarantees $R(\sigma \bar e) \in t'$.


\medskip
                                Now for the outer induction step. 
		Assume that the claim holds for all integers smaller than some $m >0$.

                \smallskip
		For the inner induction start, where $n = 0$, the arguments are similar to those used in the case $m = 0$ and $n > 0$.
		For the inner induction step, assume that the claim holds for all integers smaller than some $n >0$.
		Let $\bar a \in \Kmc^m(t)$ and $\bar a' \in \Kmc^n(t')$ for some $t$ and $t'$ and $\sigma = \rho(\bar a, \bar a')$.
		Let $\bar b \in \Kmc^{m-1}(s)$ and $\bar b' \in
                \Kmc^{n-1}(s')$ be tuples from which $\bar a$ and
                $\bar a'$ have been obtained in the construction of
                $\Kmc^m(t)$ and $\Kmc^n(t')$,
		meaning there exist $\bar u, \bar v, \bar u', \bar v' \in \mn{mgt}(\structure{A})$ with $\tp[\sigext][\Amf]{\bar u} = s$, $\tp[\sigext][\Amf]{\bar v} = t$, $\tp[\sigext][\Amf]{\bar u'} = s'$, $\tp[\sigext][\Amf]{\bar v'} = t'$ such that for all $i = 1, \dots, \sizeof{\bar a}$:
		\[
		a_i = 
		\left\{
		\begin{array}{l@{\;}l}
			b_\ell & \textrm{ if } (x_\ell , x_i) \in \rho 
			\\
			f^{j}_{s, \rho, t, i}(\bar b|_{\mn{dom}(\rho)}/_1) & \textrm{ if } x_i \notin \textrm{img}(\rho)
		\end{array}
		\right.
		\]
		where $\rho = \rho(u, v)$ and $j \notin J(\bar b|_{\textrm{dom}(\rho)})$,
		and for all $i = 1, \dots, \sizeof{\bar a'}$:
		\[
		a'_i = 
		\left\{
		\begin{array}{l@{\;}l}
			b'_\ell & \textrm{ if } (x_\ell , x_i) \in \rho'
			\\
			f^{j'}_{s', \rho', t', i}(\bar b'|_{\mn{dom}(\rho')}/_1) & \textrm{ if } x_i \notin \textrm{img}(\rho')
		\end{array}
		\right.
		\]
		where $\rho' = \rho(u', v')$ and $j' \notin J(\bar b' |_{\textrm{dom}(\rho')})$.
		Let $R(\bar e) \in t$ such that all variables in $\bar e$ are in $\mn{dom}(\sigma)$.
		
		If $\mn{dom}(\sigma) \subseteq \mn{img}(\rho)$,
                %
                %
                %
                then we can form $\sigma \circ \rho$, which restricts $\sigma' = \rho(\bar a', \bar b)$ to $\mn{dom}(\rho)$.
                We may then conclude by
                applying the induction hypothesis on $m$, with
                $\bar a := \bar b$, $t=s$, $m=m-1$, and $R(\bar
                e)=R(\rho^{-1}\bar e)$.
		
		Similarly, if $\mn{img}(\sigma) \subseteq
                \mn{img}(\rho')$, 
                then we can form
                $\rho^{-1} \circ \sigma$ which restricts $\sigma' = \rho(\bar a, \bar b')$ to $\mn{dom}(\sigma)$.
                In that case, the induction hypothesis
                on $n$ lets us conclude. 

		In the remaining case, that is
                $\mn{dom}(\sigma) \not\subseteq \mn{img}(\rho)$ and
                $\mn{img}(\sigma) \not\subseteq \mn{img}(\rho')$,
                there are
                $x_\iota \in \mn{dom}(\sigma) \setminus \mn{img}(\rho)$
                and
                $x_{\kappa'} \in \mn{img}(\sigma) \setminus
                \mn{img}(\rho')$.
                De facto,
                $a_\iota = f^{j}_{s, \rho, t, \iota}(\bar b|_{\mn{dom}(\rho)}/_1)$ and
                $a'_{\kappa'} = f^{j'}_{s', \rho', t', \kappa'}(\bar b'|_{\mn{dom}(\rho')}/_1)$.
                We let $\iota'$ and $\kappa$ be the indexes such that $\sigma(x_\iota) = x_{\iota'}$ and $\sigma(x_\kappa) = x_{\kappa'}$
               	
                We now prove $\sigma = \mn{Id}$ and $t = t'$, which will directly conclude the proof.
              	Let $x_{i} \in \mn{dom}(\sigma)$ and let $x_{i'} := \sigma x_i$.
              	We want to prove $i' = i$.
              
              	Consider first the case of $x_{i'} \not\in \mn{img}(\rho')$, that is $a_i = a'_{i'} = f^{j'}_{s', \rho', t', i'}(\bar b'|_{\mn{dom}(\rho')}/_1)$.
              	Notably, we have $j' \in J(a_i/_1)$.
              	This latter remark enforces $x_i \not\in \mn{img}(\rho)$.
              	Indeed, if $x_i \in \mn{img}(\rho)$, then $a_i \in \bar b|_{\mn{dom}(\rho)}$, thus $j' \in J(\bar b|_{\mn{dom}(\rho)}/_1)$.
              	In particular, $j' \in J(a_{\iota})$.
              	If $a'_{\iota'} \in \mn{img}(\rho')$, it yields $j' \in J(\bar b'|_{\mn{dom}(\rho')})$, which is a contradiction.
              	Similarly, if $a'_{\iota'} \not\in \mn{img}(\rho')$, then we have $j = j'$, thus $j \in J(\bar b|_{\mn{dom}(\rho)}/_1)$ which is a contradiction.
              	Therefore, we know $x_i \not\in \mn{img}(\rho)$, which yields $a_i = f^{j}_{s, \rho, t, i}(\bar b|_{\mn{dom}(\rho)}/_1)$.
              	Coupled with $a_i = f^{j'}_{s', \rho', t', i'}(\bar b'|_{\mn{dom}(\rho')}/_1)$, we obtain $i' = i$ as desired.
              	
              	Notice the above case covers $i' = \iota'$, and in particular gives $f^{j}_{s, \rho, t, \iota'}(\bar b|_{\mn{dom}(\rho)}/_1) = f^{j'}_{s', \rho', t', \iota'}(\bar b'|_{\mn{dom}(\rho')}/_1)$.
              	We therefore just proved that $t = t'$, $\rho = \rho'$, $s = s'$, $j = j'$ and $\bar b'|_{\mn{dom}(\rho')}/_1 = \bar b|_{\mn{dom}(\rho)}/_1$.
              	
              	We now move to $x_{i'} \in \mn{img}(\rho')$, thus $a_i = a'_{i'} \in \bar b'|_{\mn{dom}(\rho')}$.
              	This enforces $x_i \in \mn{img}(\rho)$.
             	Indeed, if $x_i \not\in \mn{img}(\rho)$, then $a_i = a'_{i'} = f^{j}_{s, \rho, t, i}(\bar b|_{\mn{dom}(\rho)}/_1)$ and therefore $j \in J(\bar b'|_{\mn{dom}(\rho')})$.
             	Recalling $j = j'$, this would contradict $j' \not\in J(\bar b'|_{\mn{dom}(\rho')})$.
             	Therefore, we know $x_i \in \mn{img}(\rho)$ and recall we have $x_{i'} \in \mn{img}(\rho')$.
             	Thus, there is $\ell$ and $\ell'$ such that $\rho(x_\ell) = x_i$ and $\rho'(x_{\ell'}) = x_{i'}$, that is $\rho(x_{\ell'}) = x_{i'}$ since $\rho = \rho'$, and we have $a_i = b_\ell$ and $a'_{i'} = b'_{\ell'}$.
             	From $a_i$ = $a'_{i'}$, it gives $b_\ell = b'_{\ell'}$ and thus, in particular, $b_\ell/_1 = b'_{\ell'}/_1$.
             	Recalling $\bar b'|_{\mn{dom}(\rho')}/_1 = \bar b|_{\mn{dom}(\rho)}/_1$, and since both $\ell$ and $\ell'$ clearly belong to $\mn{dom}(\rho)$, we get $b_\ell/_1 = b_{\ell'}/_1$.
             	Lemma~\ref{lem:no-repetition} yields $\ell = \ell'$, which in turn gives $i = i'$ as desired.

	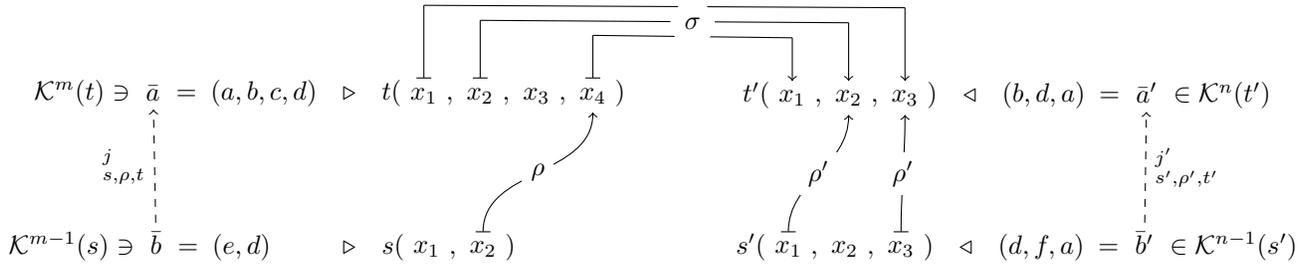
\begin{figure*}
		\centering
		\begin{tikzpicture}[remember picture]
			\node (aa) at ( 0, 0) {$
				\Kmc^m(t) \ni
				\raisebox{-.8ex}{\tikz[remember picture] \node (a) at (0, 0) {$\bar a$};}
				= ~ (a, b, c, d) ~~~\triangleright~~~ t(
				\raisebox{-1.05ex}{\tikz[remember picture] \node (ax1) at (0, 0) {$x_1$};},
				\raisebox{-1.05ex}{\tikz[remember picture] \node (ax2) at (0, 0) {$x_2$};}, 
				\raisebox{-1.05ex}{\tikz[remember picture] \node (ax3) at (0, 0) {$x_3$};},
				\raisebox{-1.05ex}{\tikz[remember picture] \node (ax4) at (0, 0) {$x_4$};}
				)$};
			\node (bb) at ( -.9,-2) {$
				\Kmc^{m-1}(s) \ni
				\raisebox{-.8ex}{\tikz[remember picture] \node (b) at (0, 0) {$\bar b$};}
				= ~ (e, d) ~~~~~~~~~~\triangleright~~~ s(
				\raisebox{-1.05ex}{\tikz[remember picture] \node (bx1) at (0, 0) {$x_1$};},
				\raisebox{-1.05ex}{\tikz[remember picture] \node (bx2) at (0, 0) {$x_2$};}
				)$};
			\node (aa') at ( 9, 0) {$
				t'(
				\raisebox{-1.05ex}{\tikz[remember picture] \node (a'x1) at (0, 0) {$x_1$};},
				\raisebox{-1.05ex}{\tikz[remember picture] \node (a'x2) at (0, 0) {$x_2$};}, 
				\raisebox{-1.05ex}{\tikz[remember picture] \node (a'x3) at (0, 0) {$x_3$};}
				) ~~~\triangleleft~~~ (b, d, a) ~ = 
				\raisebox{-.8ex}{\tikz[remember picture] \node (a') at (0, 0) {$\bar a'$};}
				\in \Kmc^{n}(t')
				$};
			\node (bb') at ( 9.13,-2) {$
				s'(
				\raisebox{-1.05ex}{\tikz[remember picture] \node (b'x1) at (0, 0) {$x_1$};},
				\raisebox{-1.05ex}{\tikz[remember picture] \node (b'x2) at (0, 0) {$x_2$};},
				\raisebox{-1.05ex}{\tikz[remember picture] \node (b'x3) at (0, 0) {$x_3$};}
				) ~~~\triangleleft~~~ (d, f, a) ~ =
				\raisebox{-.8ex}{\tikz[remember picture] \node (b') at (0, 0) {$\bar b'$};}
				\in \Kmc^{n-1}(s')
				$};
			\path
			(bx2) edge [draw, out=90, in=-90, |->] node [fill=white] {$\rho$} (ax4)
			
			(b'x1) edge [draw, out=90, in=-90, |->] node [fill=white] {$\rho'$} (a'x2)
			(b'x3) edge [draw, out=90, in=-90, |->] node [fill=white] {$\rho'$} (a'x3)
			
			(ax1) edge [draw, |-] ($(ax1) +(0,1.2)$) 
			($(ax1) +(0,1.2)$)  edge [draw, -] ($(a'x3) +(0,1.2)$) 
			($(a'x3) +(0,1.2)$)  edge [draw, ->] (a'x3)
			
			(ax2) edge [draw, |-] ($(ax2) +(0,1)$) 
			($(ax2) +(0,1)$)  edge [draw, -] ($(a'x2) +(0,1)$) 
			($(a'x2) +(0,1)$)  edge [draw, ->] (a'x2)
			
			(ax4) edge [draw, |-] ($(ax4) +(0,.8)$) 
			($(ax4) +(0,.8)$)  edge [draw, -] node [circle, fill=white, yshift={5}] {$\sigma$} ($(a'x1) +(0,.8)$) 
			($(a'x1) +(0,.8)$)  edge [draw, ->] (a'x1)
			
			(b) edge [draw, ->, dashed] node [left] {$ _{s, \rho, t}^j$} (a)
			(b') edge [draw, ->, dashed] node [right] {$ _{s', \rho', t'}^{j'}$} (a')
			;
		\end{tikzpicture}
		\caption{Inductive case on both $m$ and $n$ in the proof of Lemma~\ref{lem:formerclaim}. If $R(x_2, \delta, x_1, x_1, x_4) \in t$, we need to prove $R(x_2, \delta, x_3, x_3, x_1) \in t'$.}
	\end{figure*}
	
\end{proof}

%
\begin{lemma}
\label{lem:typelem}
If $\bar a \in \Kmc^m(t)$,
then $\tp[\sigext][\Bmf]{\bar a} = t$.
\end{lemma}
\begin{proof}
%
%
%
%
%

    Let $t \in \mn{gtp}(\structure{A})$ be an $n$-type on $\sigext$ and $\bar a \in \Kmc^m(t)$ for some $m$. We have to show that for any atom
    $\alpha$ using as terms only variables among $x_1,\dots,x_n$ and constants
    from $\Delta$, we have $\alpha \in t$ iff
    $\Bmf \models \alpha[\bar a]$. 
    From the maximality of $t$ and $\tp[\sigext][\Bmf]{\bar a}$, it will
    follow that the same equivalence holds also for literals, concluding the proof.
    Now for an atom $\alpha$, we distinguish the following cases:
  \begin{itemize}

  \item $\alpha$ has the form $x_i = c$, with $1 \leq i \leq n$ and $c \in \Delta$.
  
  	Assume that $x_i = c \in t$.
  	Since $t \in \mn{gtp}(\structure{A})$, there exists a $\bar b \in \mn{mgt}(\structure{A})$ with type $t$.
  	However, by definition, $\bar b$ does not contain any element from $\Delta$, hence a contradiction.
  	Conversely, notice that, by definition of $\Kmc^m(t)$, none of its tuples contains an element from $\Delta$.
  	Therefore, assuming $\structure{B} \models (x_i = c)[\bar a]$ already yields a contradiction.

  \item $\alpha$ has the form $x_i = x_j$, with $1 \leq i, j \leq n$.
 	
  	Recall that, from Lemma~\ref{lem:no-repetition}, $\bar a$ is without repetition.
  	Now, if $x_i = x_j \in t$, recalling that $t$ is the type of a maximally guarded tuple in $\structure{A}$, hence of a repetition-free tuple, it must be that $i = j$.
  	Naturally, $\structure{B} \models a_i = a_i$.
  	For the other direction, assume $\structure{B} \models a_i = a_j$.
  	Since $\bar a$ is without repetition, it must be that $i = j$. 
  	Every type from $\structure{A}$ trivially contains $x_i = x_i$.

  \item $\alpha$ has the form $c_1 = c_2$, with $c_1, c_2 \in \Delta$. 
  	
  	Since elements of $\Delta$ are treated as constants in $\structure{B}$ and we defined $c^\Bmf = c^\Amf$ for each constant $c$, we immediately have $\structure{B} \models c_1 = c_2$ iff $\structure{A} \models c_1 = c_2$ iff $c_1 = c_2 \in t$ (by definition of $\mn{gtp}(\structure{A})$).

  \item $\alpha$ has the form $R(\bar e)$.
  
  	It is immediate by construction of \Bmf that $R(\bar e) \in t$
  	implies $\Bmf \models R(\bar e)[\bar a]$. Conversely, assume that
  	$\Bmf \models R(\bar e)[\bar a]$. Then by construction of \Bmf there is
  	some $n'$-type $t'$, $m' \geq 0$, $\bar a' \in \Kmc^{m'}(t')$, and
  	atom $R(\bar e') \in t'$ such that $R(\bar e')[\bar a'] = R(\bar e)[\bar a]$, that is:
	$e'_\ell = x_i$ and $e_\ell=x_j$ implies $a'_i=a_j$.
	Letting $\sigma = \rho(\bar a', \bar a)$, the above yields $\sigma \bar e' = \bar e$ and,
	in particular, all variables occurring in $\bar e'$ are contained in $\mn{dom}(\sigma)$.
	Lemma~\ref{lem:formerclaim} thus delivers
	$R(\sigma \bar e') \in t$, that is $R(\bar e) \in t$ as desired.
%
%
  \end{itemize}

%

\end{proof}

\begin{lemma}
\Bmf is a model of $\phi$.
\end{lemma}
\begin{proof}
Let
$$
\phi = \bigwedge_i \forall \bar x \, (\alpha_i \rightarrow \varphi_i)
\wedge \bigwedge_j \forall \bar x \, (\beta_j \rightarrow
\exists \bar y \, ( \gamma_j \wedge \psi_j)).
$$
It is a direct consequence of Lemma~\ref{lem:typelem} that all
conjuncts of the form $\forall \bar x \, (\alpha \rightarrow \varphi)$
are satisfied in \Bmf.

Thus consider a conjunct
$\forall \bar x \, (\beta \rightarrow \exists \bar y \, ( \gamma
\wedge \psi))$ and assume that $\Bmf \models \beta[\bar a']$.  By
definition of \Bmf, there is an $s \in \mn{gtp}(\Amf)$, an $n\geq0$,
and an $\bar a \in \Kmc^n(s)$ such that all elements of $\bar a'$
occur in~$\bar a$. Let $v_1,\dots,v_k$ be the variables in $\beta$ and
let $\beta'$ be obtained from $\beta$ by replacing each variable $v_i$
with $x_j$ if $a'_i=a_j$ (this is well-defined because $\bar a$ does
not have repeated elements). By Lemma~\ref{lem:typelem}, we have
$\beta' \in s$. Take any tuple $\bar u \in \mn{mgt}(\Amf)$ such that
$\tp[\sigext][\Amf]{\bar u}=s$. Since \Amf is a model of $\phi$, we have
$\Amf \models \exists \bar y \, ( \gamma \wedge \psi)[\bar u]$.  Let
$\bar v \in \mn{mgt}(\Amf)$ be a tuple witnessing this and let
$\tp[\sigext][\Amf]{\bar v}=t$ and $\rho=\rho(\bar u,\bar v)$.  Choose any
$j \in \{0,\dots,K \} \setminus J(\bar a |_{\textrm{dom}(\rho)})$,
which must exist since the length of the tuples in $\Kmc^n(s)$ is
bounded from above by the maximum arity of relation symbols $w$,
and $K > w$. By
definition, $\Kmc^{n+1}(t)$ contains a tuple $\bar b$ whose
$i^\textit{th}$ component $b_i$ is defined as:
\[
b_i = 
\left\{
\begin{array}{l@{\;}l}
a_\ell & \textrm{ if } (x_\ell , x_i) \in \rho 
\\
f^j_{s, \rho, t, i}(\bar a|_{\mn{dom}(\rho)}/_1) & \textrm{ if } x_i \notin \textrm{img}(\rho).
\end{array}
\right.
\]
By Lemma~\ref{lem:typelem}, $\tp[\sigext][\Bmf]{\bar b}=t$. Thus,
$\bar b$ witnesses that $\Bmf \models \exists \bar y \, ( \gamma
\wedge \psi)[\bar a]$,
as required.
\end{proof}

\begin{lemma}
\Bmf satisfies Conditions~1-4 from Lemma~\ref{lemma-non-uniform}.
\end{lemma}
\begin{proof}
	Notice that in the case $\Delta = \domain{A}$, we defined $\structure{B} = \structure{A}$ and the conditions are trivially satisfied.
	We verify each condition for the case $\Delta \subsetneq \domain{A}$:
\begin{enumerate}
		\item 
		We exhibit an upper bound on the size of $\domain{B}$.
		We recall that $w \geq 2$ is the maximal arity of predicates from $\phi$ and we let $p \geq 1$ be the number of predicates occurring in $\phi$.
		Let us denote $T$ the number of types on $\sigext$.
		Since each $n$-type on $\sigext$ is essentially an interpretation over domain $\Delta \cup \{ x_1, \dots, x_n \}$, there are at most $2^{{p (\sizeof{\Delta} + n)}^{w}}$ distinct $n$-types on $\sigext$.
		The type of a maximally guarded tuple from $\structure{A}$ is a $n$-type on $\sigext$ for some $n \leq w$, thus we obtain:
		\[
		T \leq \sum_{n = 0}^{w} 2^{{p (\sizeof{\Delta} + n)}^{w}} \leq 2^{{pw(\sizeof{\Delta} + w)}^{w}}.
		\]
		Now, notice $\domain{B}$ is split between elements from $\Delta$ and those composed of constant symbols $c^j_{t, i}$ and of functions symbols $f^j_{s, \rho, t, i}$.
		The number of such partial bijection $\rho$ from $\{ x_1, \dots, x_m \}$ to $\{ x_1, \dots, x_n \}$, assuming $s$ is a $m$-type and $t$ an $n$-type, is bounded by $(w+1)^{w}$ thus by $2^{w^2}$.
		Recall that superscript $j$ ranges over $1, \dots, K$, with $K = w^{4}$, and $i$ over $1, \dots, w$.
		Therefore, the number $C$ of constant symbols $c^j_{t, i}$ is at most $T w^{5}$, and the number $F$ of function symbols at most $T^22^{w^6}$.
		Because of the truncation mechanism, every non-$\Delta$ element $e$ of $\domain{B}$ can be determined by a function that assigns to each node of the complete tree with depth $2$ and branching degree $w$, either a constant symbol, a function symbol, or a special value indicating that this node is not needed to describe $e$.
		It follows that there are at most $(C + F + 1)^{w^{3}}$ non-$\Delta$ elements in $\domain{B}$.
		Therefore, unraveling the bounds on $C$ and $F$, we get:
		\[
		\sizeof{B} \leq \sizeof{\Delta} + (T w^{5} + T^22^{w^6} + 1)^{w^{3}},
		\]
		which we can abruptly bound with:
		\[
		\sizeof{B} \leq \sizeof{\Delta} + 2^{w^9}T^{w^{5}}.
		\]
		Now, unraveling the bound on $T$, we obtain:
		\[
		\sizeof{B} \leq \sizeof{\Delta} + 2^{w^9 + {p w^{6}(\sizeof{\Delta} + w)}^{w}},
		\]
		which we can further bound by:
		\[
		\sizeof{B} \leq 2^{p (\sizeof{\Delta} + w)^{w + 10}}.
		\]
		This exhibits a double-exponential dependency w.r.t.\ the maximal arity, but only single-exponential dependency in terms of the number of predicates and the size of $\Delta$.
		Now, using $p, w \leq \sizeof{\phi}$, we can slightly simplify this bound as:
		\[
		\sizeof{B} \leq 2^{(\sizeof{\Delta} + \sizeof{\phi})^{\sizeof{\phi} + 11}}.
		\]

\item 
		Let $a \in \Delta$. 
		From $\Delta \subsetneq \domain{A}$, we have that $\mn{gtp}(\structure{A})$ is non-empty.
		In particular, there is a type $t \in \mn{gtp}(\structure{A})$.
		Since $a$ is treated as a constant symbol by types on $\sigext$, we have $\alpha \in \tp[\sigext]<1>[\Amf]{a}$ iff $\alpha[a] \in t$.
		In particular, if $\alpha$ is positive literal, that is an atom, then $\alpha \in \tp[\sigext]<1>[\Bmf]{a}$.
		Furthermore, the above equivalence holds for all type $t$ on $\sigext$ from $\mn{gtp}(\structure{A})$, which guarantees that no other atom only involving $a$ occurs in $\tp[\sigext]<1>[\Bmf]{a}$, thus  $\tp[\sigext]<1>[\Amf]{a} = \tp[\sigext]<1>[\Bmf]{a}$.
	\item 
		{\bf($\subseteq$).}
		Let $t \in \tp[\sigext]<1>[\Amf]{\domain{A} \setminus \Delta}$.
		By definition, there exists $e \in \domain{A} \setminus \Delta$ with $\tp[\sigext]<1>[\Amf]{e} = t$.
		Since $e \in \domain{A} \setminus \Delta$, there must exist some $\bar a \in \mn{mgt}(\structure{A})$ such that $a_i = e$ for some $i$.
		Let $s = \tp[\sigext][\Amf]{\bar a}$.
		We now prove that the element $c^1_{s, i}$, belonging to $\domain{B} \setminus \Delta$, has unary type $\tp[\sigext]<1>[\Bmf]{c^1_{s, i}} = t$.
		Notice $s$ fully determines the unary type of its variable $x_i$ variable, that is $t$.
		Thus, since $(c^1_{s, 1}, \dots, c^1_{s, \sizeof{\bar a}}) \in \Kmc^0(s)$ and from the definition of $\cdot^\structure{B}$, every atom of $t$ is in $\tp[\sigext]<1>[\Bmf]{c^1_{s, i}}$.
		Lemma~\ref{lem:typelem} guarantees there are no other atom in $\tp[\sigext]<1>[\Bmf]{c^1_{s, i}}$, thus $\tp[\sigext]<1>[\Bmf]{c^1_{s, i}} = t$.
		
		{\bf($\supseteq$).}
		Let $t \in \tp[\sigext]<1>[\Bmf]{\domain{B} \setminus \Delta}$.
		By definition, there exists $e \in \domain{B} \setminus \Delta$ with $\tp[\sigext]<1>[\Bmf]{e} = t$.
		By definition of $\domain{B}$, there must exist a type $s \in \mn{gtp}(\structure{A})$, an integer $m$ and a tuple $\bar a \in \Kmc^m(s)$ s.t.\ $a_i = e$ for some $i$.
		Since $s \in \mn{gtp}(\structure{A})$, there exists $\bar b \in \mn{mgt}(\structure{A})$ such that $\tp[\sigext][\Amf]{\bar b} = s$.
		In particular, $b_i \in \domain{A} \setminus \Delta$, and it is easily verified that $\tp[\sigext]<1>[\Amf]{b_i} = t$.

                	\item 
		From the very definition of $B$, resp.\ of $\Bmf$, we have $\Delta \subseteq \domain{B}$, resp.\ $c^\Bmf = c^\Amf$ for all $c \in \const(\phi)$, as desired.
		

\end{enumerate}
\end{proof}

\section{Proofs for Section~\ref{sec-circ-query}}

We start with explaining why it is not straightfoward to prove finite
controllability of circumscribed (U)CQ-querying in GF using the Rosati
cover.  The whole point of using the Rosati cover in \cite{BGO} is to
start from a model \Amf that satisfies a GF sentence $\varphi$ while
not satisfying a Boolean CQ $q$, and to convert it into a finite model
\Bmf with the same properties. This only works, however, under the
assumption that the original model \Amf does not satisfy any
`treeifications' of $q$, that is, there is no homomorphism $h$ from
$q$ to \Amf such that the homomorphic image of $q$ under $h$ is
acyclic. In the case without circumscription as in \cite{BGO}, this
assumption is for free since we can simply take $\Amf$ to be a model
of the GF sentence $\varphi \wedge \neg \chi_q$ where $\chi_q$ is the
disjunction of all treeifications of $q$. This works, because,
trivially, a model of $\varphi \wedge \chi_q$ is also a model
of~$\varphi$.  With circumscription, however, we would need the same
statement for \emph{minimal} models and it is easy to find examples
which falsify this.

\subsection{Proof of Lemma~\ref{lem-lemma5}}

\corelemma*

\begin{proof}
	Assume to the contrary of what is to  be shown that \Bmf is not a $\CP$-minimal model of
	$\phi$.  As \Bmf is a model of $\phi$ by
	assumption, there is thus a model $\Bmf'$ of $\phi$ with
	$\Bmf' <_\CP \Bmf$ (and thus $B'=B$).  To derive a
	contradiction, we construct a model $\Amf'$ of $\phi$ with
	$\Amf' <_\CP \Amf$.
	
	
	Applying Proposition~\ref{prop:smallmodelsgf} to $\Bmf'$ and $\Sigma$, we obtain a model $\Bmf''$ of $\phi$ such that
	the following conditions are satisfied:
	\begin{enumerate}
		\item[(a)] $|B''| \leq \towerof{4^{\sizeof{\sig} + 4}}{\sizeof{\phi}}$;
		
		\item[(b)] $\tp[\sig]<1>{{\Bmf''}} = \tp[\sig]<1>{\Bmf'}$;
		
		\item[(c)] 
		$\#_{\Bmf''}(t) \leq \#_{\Bmf'}(t)$ for every 1-type $t$ on $\sig$;
		
		\item[(d)] 
		$c^{\Bmf''} = c^{\Bmf'}$ for all constants $c$ in $\phi$.
	\end{enumerate}
	%
	
	To assemble the desired model $\Amf'$, we first determine the
	intended 1-type of each element $a$ of $\Amf'$.
	To this end, we define a surjection from $A$ to $B''$ which is obtained as the union of several mappings exposed as follows.
	
	For each $t \in \mn{tp}^1_\sig(\Amf)$ s.t.\ $\#_\Amf(t) \leq \towerof{4^{\sizeof{\sig} + 4}}{\sizeof{\phi}}$, we set:
	\begin{align*}
	D_t & := \{ t \in A \mid \mn{tp}^1_{\Amf, \sig}(a) = t \}
	\\
	S_t & := \{ b \in \Bmf'' \mid \exists b' \in B', \mn{tp}^1_{\Bmf, \sig}(b') = t  
	\\ & \hspace*{2cm}\wedge \mn{tp}^1_{\Bmf', \sig}(b') = \mn{tp}^1_{\Bmf'', \sig}(b) \}
	\end{align*}
	Note that by definition $\sizeof{D_t} = \#_\Amf(t) > \towerof{4^{\sizeof{\sig} + 4}}{\sizeof{\phi}}$ while $\sizeof{S_t} \leq \sizeof{B''} \leq \towerof{4^{\sizeof{\sig} + 4}}{\sizeof{\phi}}$ from Condition~(a) in the definition of $\Bmf''$.
	We can thus chose $\pi_t : D_t \rightarrow S_t$ a surjective mapping.
	
	We now form the union $\pi_0$ of all mappings $\pi_t$ s.t.\ $\#_\Amf(t) \leq \towerof{4^{\sizeof{\sig} + 4}}{\sizeof{\phi}}$.
	Note that the domain of $\pi_0$ is $A \setminus \coreofsigma{A}$, and that its range is exactly the elements from $B''$ whose type in $\Bmf''$ is realized in $\Bmf'$ at least once outside of $\coreofsigma{A}$.
	Thus, the remaining elements that need to be reached to cover all of $B''$ are those whose type in $\Bmf''$ is realized in $\Bmf'$ exclusively in $\coreofsigma{A}$.
	Let us denote the set of those latter types $S$.
	For each $t$ in $S$, it follows from Point~(c) in the definition of $\Bmf''$, that we can pick a surjection $\sigma_t : \{ b \in B' \mid \mn{tp}^1_{\Bmf', \sig}(b) = t \} \rightarrow \{ b \in B'' \mid \mn{tp}^1_{\Bmf'', \sig}(b) = t \}$.
	We complete $\pi_0$ into $\pi_1$ by adding to $\pi_0$ the union of all mappings $\sigma_t$ for $t \in S$.
	Note that the range of $\pi_1$ is already the whole $B''$, but that it's domain is still missing those elements of $\coreofsigma{A}$ whose type in $\Bmf'$ is not among $S$.
	To complete $\pi_1$ we simply pick, for each such $a \in \coreofsigma{A}$ s.t.\ $\mn{tp}^1_{\Bmf', \sig}(a) \notin S$ an element $b \in B''$ s.t.\ $\mn{tp}^1_{\Bmf', \sig}(a) = \mn{tp}^1_{\Bmf'', \sig}(b)$.
	We denote $\pi$ the resulting surjection from $A$ to $B''$.
	
	Since $\pi$ is surjective, we can pick a partial inverse function: $\rho : B'' \rightarrow A$ being s.t.\ $\rho \circ \pi = \identity$.

	We can now define the restriction of $\Amf'$ to $\mn{ran}(\rho)$:
	\begin{itemize}
		
		\item 
		$
		R^{\Amf'} = \{ \rho(\bar a) \mid \bar a \in R^{\Bmf''} \}
		$ for every relation symbol $R$;

		\item
		$c^{\Amf'} = \rho(c^{\Bmf'})$ for every constant $c$.
		
	\end{itemize}
	As a consequence, the restriction of $\Amf'$ to $\mn{ran}(\rho)$ is simply an
	isomorphic copy of $\Bmf''$. 
	By definition of $\pi$ and due to Points~(b) in the definition of $\Bmf''$ and Condition~2 on $\Bmf$, it
	additionally satisfies the condition that for all $a \in \mn{ran}(\rho)$,
	\begin{itemize}
		
		\item[($*$)] there is a $b \in B$
		such that $\mn{tp}^1_{\Amf', \sig}(a)=\mn{tp}^1_{\Bmf', \sig}(b)$ and
		$\mn{tp}^1_{\Amf, \sig}(a) = \mn{tp}^1_{\Bmf, \sig}(b)$
		
	\end{itemize}
	We want to extend
	the definition of $\Amf'$ to the elements
	$a \in A \setminus \mn{ran}(\rho)$ so that Condition~($*$) is
	satisfied. Informally, we do this by `cloning' already existing
	elements. It is then easy to use Condition~($*$) to show that
	$\Amf' <_\CP \Amf$ as desired, details are left to the reader.
	
	\smallskip
	
	We can do the cloning for any $a \in A \setminus \mn{ran}(\rho)$ in
	isolation.  
	We want to assign to $a$ the type $\mn{tp}^1_{\Bmf'', \sig}(\pi(a))$ by making it a clone of the element $a' := \rho(\pi(a))$ that already satisfies this property and is already present in $\Amf'$.
	This is done by setting, for every relation symbol $R$,
	$$
	R^{\Amf'} = R^{\Amf'} \cup \{ \bar a[a/a'] \mid \bar a \in R^{\Bmf''} \}
	$$
	where $\bar a[a/a']$ denotes the tuple obtained from $\bar a$ by
	replacing every occurrence of $a'$ with $a$. This finishes the
	construction
	of the structure $\Amf'$.
	
	Using the fact that $\Bmf''$ is a model of $\phi$, it is now
	straightforward to show that, as desired, also $\Amf'$ is a model
	of~$\phi$. 
\end{proof}

\subsection{Proof of Lemma~\ref{lemma-small-reference-model}}

\smallreferencemodel*

\begin{proof}
	We apply Proposition~\ref{prop:smallmodelsgf} on $\Amf$ and $\sig$ to obtain a model $\Amf'$ whose size is at most $\towerof{4^{\sizeof{\sig} + 4}}{\sizeof{\phi}}$.
	By Point~2 of Proposition~3, for each type $t \in \mn{tp}^1_{\Amf, \sig}(A)$, we can find an element $a_t \in A'$ whose type is $t$.
	We now extend $\Amf'$ into a model $\Bmf$ by cloning those elements $a_t$ until we have:
	\begin{itemize}
		\item $\#_\Bmf(t) = \#_\Amf(t)$ if $t \in \mn{tp}^1_{\Amf, \sig}(\Delta)$;
		\item $\#_\Bmf(t) = 1 + \towerof{4^{\sizeof{\sig} + 4}}{\sizeof{\phi}}$ if $t \in \mn{tp}^1_{\Amf, \sig}(A \setminus \Delta)$.
	\end{itemize}
	For types $t \in \mn{tp}^1_{\Amf, \sig}(\Delta)$, this is well-defined as we know that $\#_{\Amf'}(t) \leq \#_\Amf(t)$ from Point~3 of Proposition~\ref{prop:smallmodelsgf}.
	For types $t \in \mn{tp}^1_{\Amf, \sig}(A \setminus \Delta)$, this is also well-defined as we know that $\#_{\Amf'}(t) \leq \sizeof{\Amf'} \leq \towerof{4^{\sizeof{\sig} + 4}}{\sizeof{\phi}}$.
	It is readily verified that $\Bmf$ satisfies all the desired properties.
	In particular, Lemma~\ref{lem-lemma5} using $\Amf$ as the reference model applies and guarantees the $\CP$-minimality of $\Bmf$.
	Regarding the size of $\Bmf$, it is exactly $\sizeof{\Delta} + M (1 + \towerof{4^{\sizeof{\sig} + 4}}{\sizeof{\phi}})$ where $M$ is the number of types in $\mn{tp}^1_{\Amf, \sig}(A \setminus \Delta)$.
	Using $\sizeof{\Delta} \leq (2^\sizeof{\sig} - M) \towerof{4^{\sizeof{\sig} + 4}}{\sizeof{\phi}}$ and $M \leq 2^\sizeof{\sig}$, we obtained the claimed bound.
\end{proof}

\subsection{Proof of Lemma~\ref{lemma-mosaic}}

\lemmamosaic*

\renewcommand{\countermodel}{{\Amf}}
\newcommand{\patchoice}[3][\pattern]{\mathsf{ch}^{#2}_{#1, #3}}
\newcommand{\patref}{\mathsf{ref}}
\newcommand{\patrefof}[1]{\patref(#1)}
\newcommand{\targetof}[1]{\target(#1)}
\newcommand{\relevantpat}{\mathsf{MTree}}

To prove the ``$\Leftarrow$'' direction of Lemma~\ref{lemma-mosaic}, assume we are given 
a good mosaic $\candidatepattern$.
For each good mosaic $\pattern$, each $1 \leq i \leq n_\exists$, if $\bar a \in R_i^{\Bmf_{\pattern}}$, then we choose a good mosaic $\pattern '$ such that:
\begin{enumerate}
	\item $\mn{tp}_{\Bmf_M, \sig}(\bar a)=\mn{tp}_{\Bmf_{M'}, \sig}(\bar a)$;
	
	\item $\Bmf_{M'} \models \exists \bar y\, (\gamma_i \wedge
	\psi_i)[\bar a]$;
	
	\item if $(p,\widehat p,h') \in S_{M'}$, then $(p,\widehat p,h) \in S_{M}$ where
	$h$ is the restriction of $h'$ to range $A_0
	\cup \bar a$.
\end{enumerate}
We use $\patchoice{i}{\bar a}$ to denote this chosen $M'$.
Then, starting from $\candidatepattern$, we build a forest-shaped set of words which witnesses the acceptance of $\candidatepattern$.

\begin{definition}
	The mosaic forest $\relevantpat$ is the smallest set of words such that:
	\begin{itemize}
		\item $(\candidatepattern, \emptyset, \emptyset) \in \relevantpat$;
		\item If $w \in \relevantpat$ ends with $(\pattern, x, y)$ with $M$ a mosaic, and $\bar a \in R^{\Bmf_{\pattern}}$ for some $1 \leq i \leq n_\exists$ with $\beta_i = R(\bar z)$, then $w \cdot (\patchoice{i}{\bar a}, i, \bar a) \in \relevantpat$.
	\end{itemize}
\end{definition}



It remains to `glue' together 
the interpretations $\interpof$ according to the structure of $\relevantpat$.
Since a mosaic $\pattern$ may occur more than once, 
we create a copy of $\interpof$ for each 
node in $\relevantpat$ of the form $w \cdot (\pattern, x, y)$.
We do not duplicate elements from $A_0$ as they precisely are those we want to reuse.
Hence only elements from $U^+$ may be duplicated.
We also take into consideration the overlap between successive mosaics in the tree and introduce a copy of an element only when necessary.
Formally, this is achieved by the 
following duplicating functions, defined inductively on $\relevantpat$.
For $(\candidatepattern, \emptyset, \emptyset)$, define:
\[
\begin{array}{r@{~}c@{~}l}
	\caster[(\candidatepattern, \emptyset, \emptyset)] : B_{M_0} & \rightarrow & A_0 \cup \{ x^{(\candidatepattern, \emptyset, \emptyset)} \mid x \in B_{M_0} \}
	\\
	b & \mapsto &
	\left\{
	\begin{array}{ll}
		b & \text{if } b \in A_0
		\\
		b^{(\candidatepattern, \emptyset, \emptyset)} & \text{otherwise}
	\end{array}
	\right.
\end{array}
\]
If $w \cdot (M, i, \bar a) \in \relevantpat$ with $w$ not empty, define:
\[
\begin{array}{r@{~}c@{~}l}
	\caster[w \cdot (\pattern, i,  \bar a)] : \domain{\interpof} & \rightarrow & A_0 \cup \{ u^m \mid u \in U^+, m \in \relevantpat \}
	\\
	b & \mapsto &
	\left\{
	\begin{array}{ll}
		b & \text{if } b \in A_0
		\\
		\caster[w \cdot (\pattern, i,  \bar a)](b) & \text{if } b \in \bar a \setminus A_0
		\\
		b^{w \cdot (\pattern, i, \bar a)} & \text{otherwise}
	\end{array}
	\right.
\end{array}
\]
The desired model $\countermodel $ 
can then be defined as:
\[
\countermodel := \bigcup_{\substack{ w \cdot (\pattern, i, \bar a) \in \relevantpat \\ w \text{ possibly empty}}} \casterof[w \cdot (\pattern, i, \bar a)]{\interpof},
\]
that is the domain (resp.\ the interpretation of each predicate) of $\countermodel$ is the union across all $w \cdot (\pattern, i, \bar a) \in \relevantpat$ of the image by $\caster[w \cdot (\pattern, i, \bar a)]$ of the domain (resp.\ the interpretation of each predicate) of $\interpof$.

By definition, each $\caster[w \cdot (\pattern, i, \bar a)]$ is a homomorphism from $\interpof$ to $\countermodel$.
It is also easily verified that this homomorphism is injective.
Due to Point~1 in the definition of good mosaics, if the range of two duplicating functions overlap, then the common element agree on the interpretation of all predicates. 
More formally, the following lemma is an immediate consequence of the definition of functions $\caster[w \cdot (\pattern, i, \bar a)]$ and Point~1 in the definition of good mosaics.

\begin{lemma}
	\label{lemma:concepts-in-patterns}
	\label{lemma:predicate-transfer}
	For all $w \cdot (\pattern, i, \bar a) \in \relevantpat$ and all $\bar b \in B_M^{\sizeof{\bar b}}$, we have $\mn{tp}_{\countermodel, \sig}(\casterof[w \cdot (\pattern, i, \bar a)]{\bar b}) = \mn{tp}_{\interpof, \sig}({\bar b})$.%
\end{lemma}%

Note that, in the above, exact values of $i$ and $\bar a$ do not affect the statement.
Henceforth, we often use $w \cdot (M, \_, \_)$ to refer to an element of $\relevantpat$ in which values of the second and third components of its last triple are not important.
We are ready to show that 
$\countermodel$ is a model of~$\phi$.

\begin{lemma}
	\label{lemma:countermodel-is-model}
	$\countermodel$ is a model of $\phi$.
\end{lemma}

\begin{proof}
	Let $1 \leq i \leq n_\forall$. 
	We prove $\countermodel$ satisfies $\forall \bar x \, (\alpha_i \rightarrow \varphi_i)$.
	Assume there is $\bar a \in R^\countermodel$, where $\alpha_i = R(\bar z)$.
	By definition of $\bar a$, there exists $w \cdot (\pattern, \_, \_)$ and $\bar b \in R^{\interpof}$ s.t.\ $\bar a = \casterof[w \cdot (\pattern, \_, \_)]{\bar b}$.
	From Point~4 of the definition of a mosaic, we get $\interpof \models \varphi_i[\bar b]$.
	By Lemma~\ref{lemma:predicate-transfer} we obtain $\countermodel \models \varphi_i[\casterof[w \cdot (\pattern, \_, \_)]{\bar b}]$ that is $\countermodel \models \varphi_i[\bar a]$ as desired.
	
	Let $1 \leq i \leq n_\exists$. 
	We prove $\countermodel$ satisfies $\forall \bar x \, (\beta_i \rightarrow
	\exists \bar y \, ( \gamma_i \wedge \psi_i))$.
	Assume there is $\bar a \in R^\countermodel$, where $\beta_i = R(\bar z)$.
	By definition of $\bar a$, there exists $w \cdot (\pattern, \_, \_)$ and $\bar b \in R^{\interpof}$ s.t.\ $\bar a = \casterof[w \cdot (\pattern, \_, \_)]{\bar b}$.
	By definition of the mosaic tree, we get $w \cdot (M, \_, \_) \cdot (\patchoice{i}{\bar b}, i, \bar b) \in \relevantpat$.
	We let $M' := \patchoice{i}{\bar b}$ for readability.
	By definition of $M'$ and from Point~2 of the definition of good mosaics, we obtain $\Bmf_{M'} \models \exists \bar y(\gamma_i \wedge \psi_i)[\bar b]$.
	Thus there exists an extension $\bar d \in S^{\Bmf_{M'}}$ of $\bar b$ such that $\Bmf_{M'} \models \psi_i[\bar d]$, where $\gamma_i = S(\bar z')$.
	By Lemma~\ref{lemma:predicate-transfer}, we obtain $\countermodel \models \psi_i[\casterof[w \cdot (M, \_, \_) \cdot (M', i, \bar b)]{\bar d}]$.
	By definition of $\caster[w \cdot (M, \_, \_) \cdot (M', i, \bar b)]$, we have $\casterof[w \cdot (M, \_, \_) \cdot (M', i, \bar b)]{\bar b} = \bar a$.
	Thus the tuple $\casterof[w \cdot (M, \_, \_) \cdot (M', i, \bar b)]{\bar d}$ is an extension of $\bar a$ (that extends $\bar a$ as $\bar d$ extends $\bar b$) that satisfies $\gamma_i \wedge \psi_i$ in $\countermodel$ as desired.
\end{proof}

This proves $\Amf_0$ can indeed by extended into a model of $\phi$.
Notice that Condition~(a) is clearly satisfied due to Point~2 in the definition of mosaics.
The $\subseteq$ direction of Condition~(b) is clearly satisfied from Point~3 in the definition of mosaics joint with Lemma~\ref{lemma:concepts-in-patterns}.
Its $\supseteq$ direction of Condition~(b) is satisfied from the basic conditions on the pair $(\Amf_{0},\typescandidate)$, that is $\Amf_0$ already contains at least one witness for each type from $\typescandidate$.
It remains to verify Condition~(c), that is for all $p \in q$, there is no homomorphism from $p$ in $\countermodel$.

By contradiction, assume there exists a CQ $p \in q$ and a homomorphism $\match : p \rightarrow \countermodel$.
Therefore, for each atom $\alpha := R(\bar z)$ in $p$, it follows from the definition of $R^\countermodel$ that there exists a node $w_\alpha := w \cdot (M, \_, \_) \in \relevantpat$ and $\bar a \in R^{\Bmf_M}$ s.t.\ $\casterof[w_\alpha]{\bar a} = \match(\bar z)$.
We assume chosen such a node $w_\alpha$ for each $\alpha \in p$.
On $\relevantpat$, we consider the order given by the prefix relation, that is for all $w_1, w_2 \in \relevantpat$, we denote $w_1 \leq w_2$ iff $w_1$ is a prefix of $w_2$.
We define a set of words $W_p$ as containing all words from $\relevantpat$ that are prefixes of a word $w_\alpha$ for some $\alpha \in p$.
We let $p_0 := \{ \alpha \in p \mid \Amf_0 \models \match(\alpha) \}$ the subquery of $p$ that is mapped by $\match$ in $\Amf_0$ and $h_0 := \match|_{\mn{var}(p_0)}$ the corresponding mapping of variables.
For a word $w \in \relevantpat$, we set $p_w := p_0 \cup \{ \alpha \in p \setminus p_0 \mid w \leq w_\alpha \}$ and $h_w$ a mapping defined on every $v \in \mn{var}(p_u)$ s.t.\ $\match(v) \in \casterof[u]{B_M}$ by $h(v) = (\caster[u])^{-1}(\match(v))$.
Note that $h_w$ always extends $h_0$ since all mosaics agree on the interpretation of $\Amf_0$ (Point~2 in the definition of mosaics).
%
%

\begin{lemma}
	For all $w \cdot (M, \_, \_) \in W_p$, $(p, p_w, h_w) \in S_M$. 
\end{lemma}

In particular, since $(\candidatepattern, \emptyset, \emptyset) \in W_p$ and $p_{(M_0, \emptyset, \emptyset)} = p$, this guarantees $(p, p, h_{(M_0, \emptyset, \emptyset)}) \in S_{(M_0, \emptyset, \emptyset)}$, which contradicts $M_0$ being a mosaic (it contains a complete match triple!) and thus concludes the proof of Condition~(c).

We now prove the above lemma.

\begin{proof}
	We proceed by induction on elements of $W_p$, starting from its maximal elements w.r.t.\ $\leq$.
	Notice there are at most $\sizeof{p}$ such maximal elements.
	
	{\bf Base case.}
	Assume $u = w \cdot (\pattern, \_, \_) \in W_p$ is maximal for $\leq$.
	Therefore, $p_u = p_0 \cup \{ \alpha \in p \setminus p_0 \mid w_\alpha = u \}$.
	By definition of each $w_\alpha$, $h_u$ defines a homomorphism from $p_u$ to $\Bmf_M$. 
	Since $S_M$ is saturated, it follows from the first saturation condition that $(p, p_u, h_u) \in S_M$.
	
	
	{\bf Induction case.}
	Consider $u = w \cdot (\pattern, \_, \_) \in W_p$ and assume the property holds for all $u' \in W_p$ with $u \leq u'$.
	We refine the subset $p_u \subseteq p$ of atoms by a smaller subset $p_{=u} := p_0 \cup \{ \alpha \in p_u \setminus p_0 \mid w_\alpha = u \}$.
	As in the base case, the definition of words $w_\alpha$ guarantees there is a homomorphism from $p_{=u}$ to $\Bmf_M$ given by $h_{=u} := \caster[u]^{-1} \circ \match|_{D_u}$ where $D_u = \match^{-1}(\casterof[u]{B_M})$.
	From $S_M$ being saturated, we obtain $(p, p_{=u}, h_{=u}) \in S_M$.
	For the other atoms of $p_u$, that is $\alpha \in p_u \setminus p_{=u}$, we have $u < w_\alpha$ thus there exists $u_\alpha = u \cdot (M_\alpha, i_\alpha, \bar a_\alpha) \in W_p$ with $u_\alpha \leq w_\alpha$.
	We denote $u_1, \dots, u_n$ the distinct elements from $\{ u_\alpha \mid \alpha \in p_u \setminus p_{=u} \}$ and if $u_k = u_\alpha$, we let $(M_k, i_k, \bar a_k) := (M_\alpha, i_\alpha, \bar a_\alpha)$. 
	Note that $p_u$ can now be partitioned as $p_{=u} \uplus \biguplus_{k = 1}^n (p_{u_k} \setminus p_0)$.
	For each $1 \leq k \leq n$, we apply the induction hypothesis to obtain $(p, p_{u_k}, h_{u_k}) \in S_{M_k}$.
	From $w \cdot (M, \_, \_) \cdot (M_k, i_k, \bar a_k) \in \relevantpat$ and Point~3 of in the definition of a good mosaic, we obtain $(p, p_{u_k}, g_{u_k}) \in S_{M}$ where $g_{u_k}$ is the restriction of $h_{u_k}$ to range $A_0 \cup \bar a_k$.
	Recall that $S_M$ is saturated, thus using the second condition from definition of saturation, we can derive $(p, p_u, h_{=u} \cup \bigcup_{k = 1}^n g_{u_k}) \in S_M$ if we prove that partial mappings $h_{=u}, g_{u_1}, \dots, g_{u_n}$ are pairwise compatible (in the sense specified in the definition of saturation).
	
	We first treat the case of $g_{u_1}$ and $g_{u_2}$ (any other combination of $g_{u_k}$ and $g_{u_\ell}$ with $k \neq \ell$ being treated similarly).
	Assume $x \in \mn{var}(p_{u_1}) \cap \mn{var}(p_{u_2})$, we have to prove that both $g_{u_1}$ and $g_{u_2}$ are defined on $x$ and that $g_{u_1}(x) = g_{u_2}(x)$.
	If $x \in \mn{var}(p_0)$, this is clear.
	Otherwise, $\match(x) = d^{w'}$ for some $d \in U^+$ and $w' \in \relevantpat$.
	By definition of $h_{u_1}$ and $h_{u_2}$, we get $d^{w'} \in \casterof[u_1]{B_{M_1}}$ and $d^{w'} \in \casterof[u_1]{B_{M_2}}$ respectively.
	Thus $w'$ must be a prefix of both $u_1$ and $u_2$, by definition of the functions $\caster[u_1]$ and $\caster[u_2]$.
	The longest common prefix of $u_1$ and $u_2$ being $u$, we obtain that $w'$ is a prefix of $u$.
	In particular, $d^{w'} \in \casterof[u]{B_M}$ and $h_{u_1}(x) \in \bar a_1$ and $h_{u_2}(x) \in \bar a_2$.
	Thus $g_{u_1}$ and $g_{u_2}$ are both defined on $x$ and equal to $\caster[u]^{-1}(\match(x))$, that is $h_u(x)$.
	
	The argument for the case of $h_{=u}$ and $g_{u_k}$ is similar, observing that if $x \in (\mn{var}(p_{=u}) \cap \mn{var}(p_{u_k})) \setminus \mn{var}(p_0)$, we have $h_{u_k} \in \bar a_k$ thus $g_{u_k}$ being defined on $x$ and equal to $\caster[u]^{-1}(\match(x))$, that is $h_u(x)$.
	
	We thus obtain $(p, p_u, h_{=u} \cup \bigcup_{k = 1}^n g_{u_k}) \in S_M$ and it was already verified within the above cases that $h_u = h_{=u} \cup \bigcup_{k = 1}^n g_{u_k}$, thus $(p, p_u, h_u) \in S_M$ as desired.

\end{proof}

This concludes the proof of  the `$\Leftarrow$' direction in Lemma~\ref{lemma-mosaic}.

\newcommand{\universe}{U}

We now turn to the `$\Rightarrow$' direction.
Assume that $\Amf_0$ can be extended to a model \Amf of $\phi$ that
satisfies Conditions~(a) to~(c).
We extract a set of good mosaics from $\Amf$.
To this end, we define the notion of paths in $\Amf$.

A \emph{path} in \Amf is a sequence
$\bar a_1,\dots,\bar a_n$ of maximally guarded tuples from
$\mn{mgt}(\Amf)$ that contain at least one element of
\mbox{$A \setminus \Delta$}.
%
For each path $p = \bar a_1, \dots, \bar a_n$ in $\Amf$, we now inductively define a mosaic $M_p$ as follows.

If $p = \bar a$ is of length $1$, we let $L_p := \{ i \mid a_i \in \bar a \setminus A_0 \}$ and pick for each $i \in L_p$ a fresh element $u_i \in \universe$.
We denote $\lambda_p : a_i \mapsto u_i$ the corresponding bijection between elements of $\bar a$ that are not in $A_0$ and the chosen fresh elements.
We extend $\lambda_p$ to $A_0 \cup \bar a$ by defining $\lambda_p$ to be the identity on elements from $A_0$.
We now define $M_p$ by $\Bmf_{M_p} = \lambda_p(\Amf|_{A_0 \cup \bar a})$ and $S_{M_p}$ as the set of all triples $(p, \widehat{p}, \lambda_p \circ h|_{A_0 \cup \bar a} )$ s.t.\ $p \in q$, $\widehat{p} \subseteq p$ and $h : \widehat{p} \rightarrow \Amf$ is a homomorphism.
Note that $M_p$ is a well-defined mosaic: regarding $\Bmf_{M_p}$, Point~1 is clear by construction, Points~2 and 3 follow respectively from Conditions~(a) and (b) being satisfied by $\Amf$, Point~4 from $\Amf$ being a model of $\phi$; regarding $S_{M_p}$, it is saturated by construction and does not contain a complete match triple from Condition~(c) being satisfied by $\Amf$.

If $p = \bar a_1, \dots, \bar a_{n}, \bar a_{n+1}$ has length $> 1$, we let $L_p := \{ i \mid a_i \in \bar a_{n+1} \setminus (A_0 \cup \bar a_n) \}$ and pick for each $i \in L_p$ a fresh element $u_i \in \universe$.
Note that this is always possible as $\universe$ has size twice the maximal arity of a predicate.
We denote $\lambda_p : a_i \mapsto u_i$ the corresponding bijection elements of $\bar a_{n+1}$ that are neither in $A_0$ nor in $\bar a_n$, and the chosen fresh elements.
We extend $\lambda_p$ to $A_0 \cup \bar a_{n+1}$ by defining $\lambda_p$ to be the identity on elements from $A_0$ and to be $\lambda_{\bar a_1, \dots, \bar a_n}$ on elements from $\bar a_n \cap \bar a_{n+1}$.
We now define $M_p$ by $\Bmf_{M_p} = \lambda_p(\Amf|_{A_0 \cup \bar a_{n+1}})$ and $S_{M_p}$ as the set of all triples $(q', \widehat{q'}, \lambda_p \circ h|_{A_0 \cup \bar a_{n+1}} )$ s.t.\ $q' \in q$, $\widehat{q'} \subseteq q'$ and $h : \widehat{q'} \rightarrow \Amf$ is a homomorphism.
As in the base case, it is easily verified that $M_p$ is a well-defined mosaic.

We now let $\Mmc := \{ M_p \mid p \textrm{ is a path in } \Amf \}$ and claim all mosaics in $\Mmc$ are good in $\Mmc$.
Let $M_p \in \Mmc$ for some path $p = \bar a_1, \dots, \bar a_n$ in $\Amf$.
Assume there exists $\bar a \in R^{\Bmf_{M_p}}$ where $R(\bar z) = \beta_i$ for some $1 \leq i \leq n_\exists$.
Since $\Amf$ is a model of $\phi$, there exists a tuple $\bar a' \in S^\Amf$ where $S(\bar z') = \gamma_i$ s.t.\ $\bar a'$ witnesses $\Amf \models \exists y (\gamma_i \wedge \psi_i)[\bar a]$.
Consequently, $p' = \bar a_1, \dots, \bar a_n, \bar a'$ is a path in $\Amf$.
We prove $M_{p'}$ is a mosaic satisfying Conditions~1 to 3 in the definition of good mosaics.
Condition~1 is clear by construction of $M_{p'}$, which is inductively defined from $M_p$ and thus their `overlap' is consistent as expected.
Condition~2 is also guaranteed from the very definition of $\bar a'$.
For Condition~3, recall that each match triple $(q', \widehat{q'}, h) \in S_{M_{p'}}$ is such that $\lambda_{p'}{^-1} \circ h$ is the restriction to $A_0 \cup \bar a'$ of a complete homomorphism from $\widehat{q'}$ to $\Amf$.
By considering instead the restriction of $\widehat{q'}$ to $A_0 \cup \bar a$ we obtain the desired match triple in $S_{M_p}$.

\section{Proofs for Section~\ref{sect:lowerbounds}}

To fix notation, we recall that an ATM $\Mmc$ is specified by a
$6$-tuple $\Mmc = (Q, \Sigma, \Gamma, \delta, q_0, g)$ where:
\begin{itemize}
\item $Q$ is the finite set of states;
  \item $\Sigma$ is the finite input alphabet;
	\item $\Gamma \supseteq \Sigma$ is the finite tape alphabet
          with a special blank symbol $\textvisiblespace \in \Gamma
          \setminus \Sigma$;
	\item $\delta : Q \times \Gamma \rightarrow (Q \times \Gamma \times \{ \triangleleft,  \triangleright \})^2$ is the transition function;
	\item $q_0 \in Q$ is the initial state;
	\item $g : Q \rightarrow \{ \universalstate, \existentialstate, \mn{accept}, \mn{reject} \}$ specifies the type of each state.
\end{itemize}
Note that without loss of generality, we consider TMs having the following properties:
\begin{itemize}
	\item for every universal or existential configuration, there exist exactly two applicable transitions;
	\item the machine directly accepts any configuration whose state $s$ is such that $g(s) = \mn{accept}$;
	\item the ATM never tries to go to the left of the initial position.
\end{itemize} 
We say that $\Mmc$ is \emph{polynomially space-bounded} if there
exists a polynomial $p$ such that on input $x$, $\Mmc$ visits only the
first $p(\sizeof{x})$ tape cells. \Mmc being
\emph{$\lol$-exponentially space-bounded}, for some $\lol \geq 1$, is
defined accordingly, allowing the ATM to visit
$\towerof{\lol}{p(\sizeof{x})}$ tape cells. We consider only ATMs that
terminate on every input. It is known that there is a fixed
polynomially space-bounded ATM
$\Mmc = (Q, \Sigma, \Gamma, \delta, q_0, g)$ whose word problem is
\Exp-hard \cite{chandra-alternation}.

\subsection{Proof of Theorem~\ref{theorem:exp-hardness}}

\begin{figure*}
\begin{align}
	\config_q(x,x_h) 
	& \rightarrow \nonconfig_{\triangleleft}(x,x_h),
   \nonconfig_{\triangleright}(x,x_h)
   \label{ruleexp:one}
	\\
	\nonconfig_{\triangleleft}(x,x_p),\ \nonstart(x_p) 
	& \rightarrow \exists y_p \, \nextsymb(y_p,x_p),
	\nonconfig_{q'}(x,y_p), \nonconfig_{\triangleleft}(x,y_p)
   \label{ruleexp:two}
	\\
	\nonconfig_{\triangleright}(x,x_p),\ \nonend(x_p) & \rightarrow
	\exists y_p \, \nextsymb(x_p,y_p), \nonconfig_{q'}(x,y_p),
                                                            \nonconfig_{\triangleright}(x,y_p).
   \label{ruleexp:three}
	\\
	\nextconfig_{i}(x,y,x_h) & \rightarrow 
                                   \copylefttape(x,y,x_h), \copyrighttape(x,y,x_h)
   \label{ruleexp:four}
  \\
	\copylefttape(x,y,x_p), \nonstart(x_p) & \rightarrow 
	\exists y_p \,
	\nextsymb(y_p,x_p), 
                                                 \copylefttape(x,y,y_p), \copytape(x,y,y_p)
                                                    \label{ruleexp:five}
	\\
	\copyrighttape(x,y,x_p), \nonend(x_p) &\rightarrow 
	\exists y_p \,
	\nextsymb(x_p,y_p), 
	\copyrighttape(x,y,y_p), \copytape(x,y,y_p)
                                                    \label{ruleexp:six}
	\\
	\copytape(x,y,x_p), \mn{tape}_s(x,x_p) &\rightarrow 
                                                 \mn{tape}_s(y,x_p)
                                                    \label{ruleexp:seven}
\end{align}
\caption{Additional rules used in the proof of Theorem~\ref{theorem:exp-hardness}, for all $i \in \{ 1, 2\}$.}
\label{figure:rules-exp-hardness}
\end{figure*}


We exhibit a set of existential rules $\Omc$ and a circumscription
pattern $\CP$ such that, given a word $e = e_1 \cdots e_n \in \Sigma^*$, we can
construct in polynomial time a database $D$ such that $\Mmc$ accepts
$e$ iff $\Circ(\Omc, D) \models \target(a)$, where $\target$ is a
unary predicate and $a$ a dedicated constant symbol.  To give a better
intuition of the involved mechanisms, we describe the
constructions of $\Omc$ and~$D$ together.  It can, however, be verified that
$\Omc$ is independent of~$e$.

We use all tape positions $i$ with $1 \leq i \leq p(n)$ as constants
in the database $D$. In addition, we use the constant $a$ that occurs
inside the query above. In $D$, we mark the constant symbols that
represent positions by a
$$
\pos(i)  \textrm{ for } 1 \leq i \leq p(n).
$$
To ensure that, in every minimal model, the positions are picked
from the intended constants, we require the predicate $\pos$ to be minimized.
%

We introduce binary predicates $\mn{config}_q$ for every $q \in Q$ and
$\mn{tape}_s$ for every $s \in \Gamma$. Intuitively,
$\mn{config}_q(x,x_h)$ says that $x$ is a configuration of \Mmc where
\Mmc is in state $q$ and the head is on position $x_h$ of the
tape. Likewise, $\mn{tape}_s(x,x_p)$ says that in configuration $x$,
the symbol $s$ is on position $x_p$ of the tape. For all
$q \in Q$ and $s \in \Gamma$, we require that
$$
\begin{array}{rcl}
  \config_q(x,x_h) &\rightarrow& \pos(x_h) \\[1mm]
\tape_s(x, x_p) &\rightarrow&\pos(x_p).
\end{array}
$$
The initial configuration is now encoded by
\begin{align*}
	\config_{q_0}(a,1) 
	\\
	\tape_{e_i}(a, i) & \textrm{ for } 1 \leq i \leq n
	\\
	\tape_{\textvisiblespace}(a, i) & \textrm{ for } n < i \leq p(n).
\end{align*}
To generate the other configurations of the ATM computation on input
$e$, we use existential quantifiers in rules. We need some
preparation. To start with, we add auxiliary facts on the
constants
$1,\dots,n$ that pertain to their order:
\begin{align*}
	\nonnextpos(i,j) & \textrm{ for } 1 \leq i,j \leq p(n) \text{
                           with } i \neq j \\
  	\nonstart(i) & \textrm{ for } 1 < i \leq p(n) \\
  	\nonend(i) & \textrm{ for } 1 \leq i < p(n).
\end{align*}
We will later also use a `positive' version \mn{next} of the
$\nonnextpos$ predicate. We want these predicates to be disjoint
in models that falsify $\target(a)$.
This is achieved by adding the rules
\begin{align*}
	\nextsymb(x, y), \nonnextsymb(x, y) & \rightarrow \error(x) \\
  \catcherror(x, y), \error(x) & \rightarrow \target(x).
\end{align*}
  and facts
\[
\catcherror(a, i)  \textrm{ for } 1 \leq i \leq p(n).
\]
This mechanism will be used in the following to detect also other
violations of the intended encoding. We also want $\nextsymb$ to be
restricted to positions:
$$
\nextsymb(x,y) \rightarrow \pos(x), \pos(y).
$$

We next make sure that the state, head position, and tape content
is unique for every configuration. Regarding the tape content, for
all distinct $s,s' \in \Gamma$ we put
$$
  \mn{tape}_s(x,x_p),\mn{tape}_{s'}(x,x_p) \rightarrow \mn{error}(x_p).
$$
To ensure uniqueness of the state and head position, for all $q,q' \in
Q$ we add Rules~(\ref{ruleexp:one}) to~(\ref{ruleexp:three}) shown in Figure~\ref{figure:rules-exp-hardness}.
%
Moreover, for distinct $q,q' \in Q$, we add
%
%
\begin{align*}
  \config_q(x,x_h),\config_{q'}(x,x_h) & \rightarrow \mn{error}(x_h)
  \\
  	\config_q(x,x_p),\nonconfig_{q}(x,x_p) & \rightarrow \mn{error}(x_p).
\end{align*}
We now generate the additional configurations.  For all $q \in Q$ with
$g(q) \notin \{ \mn{accept}, \mn{reject} \}$, all $s \in \Gamma$, 
if $\delta(q,s)=(q_1,s_1,m_1,q_2,s_2,m_2)$ then
put for all $i \in \{ 1, 2\}$,
\[
\begin{array}{l@{}c@{}r}
	\config_q(x,x_h), \mn{tape}_s(x,x_h) & \rightarrow
	& \exists y\, \exists y_h \, \nextconfig_{i}(x,y,x_h),
	\\ & & \config_{q_i}(y,y_h), \mn{tape}_{s_i}(y,x_h), 
	\\ & & \mn{MOV}
\end{array}
\]
%
where $\mn{MOV}$ is $\mn{next}(x_h,y_h)$ if $m_i= \triangleright$ and
$\mn{next}(y_h,x_h)$ if $m_i= \triangleleft$. The above rule sets up
correctly the state and head position of the new configuration, as
well as
the symbol on the tape at the position where the head was located
previously.
We still need to say that the other symbols didn't change. For
all $q \in Q$, $i \in \{1,2\}$, and $s \in \Gamma$, add Rules~(\ref{ruleexp:four}) to~(\ref{ruleexp:seven}) 
shown in Figure~\ref{figure:rules-exp-hardness}.

Note that the constant $a$ is the root of a tree of configurations
that are connected by the (projection to the first two components
of the) predicates $\nextconfig_{i}$, 
$i \in \{1,2\}$. We now
propagate acceptance information upwards in that tree. For all
states $q \in Q$ with $g(q)=\mn{accept}$, add
$$
	\config_q(x,x_h) \rightarrow \acc(x).
$$
Further add, for all $q \in Q$ and $i \in \{1,2\}$,
$$
\nextconfig_{i}(x,y,x_h), \acc(y) \rightarrow \mn{acc}_i(x).
$$
Now, for all $q \in Q$ with  $g(q)=\mn{univ}$, all $q' \in Q$ with
$g(q')=\mn{exist}$, and all $i \in \{1,2\}$, add
\begin{align*}
  \config_q(x,x_h), \acc_1(x), \acc_2(x) & \rightarrow \acc(x) \\
  \config_{q'}(x,x_h), \acc_i(x) & \rightarrow \acc(x).
\end{align*}
At the root of the configuration tree, which is $a$, acceptance
makes the query true:
\begin{align*}
	\acc(x) & \rightarrow \target(x).
\end{align*}
We use $\Omc$ to denote the obtained set of rules and $D$ for the obtained database.
The circumscription pattern $\CP$ has a minimized predicate $\pos$ and all other predicates are varying.
In particular, $\CP$ has no fixed predicates. It is now
straightforward to prove the following.
\begin{lemma}
$\Mmc$ accepts on input $e$ iff $\Circ(\Omc, D) \models \target(a)$. 
\end{lemma}
We omit the details, but remark that, as a crucial point, in every
minimal model $\Amf$ of $\kb = (\Omc,D)$, we must have
$\pos^\Amf = \{1,\dots,p(n)\}$.  This can be seen as follows.  First,
$1,\dots,p(n)$ have been asserted to satisfy $\pos$ in $D$ and must be
distinct due to the semantics of databases. And second, for any model
\Bmf of \kb in which $\pos^\Bmf$ is a strict superset of
$\{1,\dots,p(n)\}$,
we can create a model \Amf of \kb with $\Amf <_\CP \Bmf$ in
the following way:
\begin{itemize}

\item set $\pos^\Amf = \{1,\dots,p(n)\}$;

\item set $\mn{next}^\Amf = \pos^\Amf \times \pos^\Amf$;
  
\item set $R^\Amf = A \times \{1,\dots,p(n)\}$ for all relations $R$
  of the form $\mn{conf}_q$ and $\mn{tape}_s$;

\item interpret all other
  predicates as total. 

\end{itemize}

\medskip

We now consider the second part of Theorem~\ref{theorem:exp-hardness},
claiming that when switching from AQ-querying to UCQ-querying we can
make do with predicates of arity at most two. Note that the only
predicates in the above reduction that are of higher arity are
$\nextconfig_i$, $\copytape$, $\copylefttape$, and
$\copyrighttape$. These predicates are used to ensure that the content
of tape cells that are not under the head remain the same when the
ATM makes a transition.

We modify the reduction by making $\nextconfig_i$ binary, dropping the
third position that records the head position of the configuration
that is in the first position. The (important) rule that generates
additional configurations otherwise remains unchanged.  We drop the
relations $\copytape$, $\copylefttape$, and $\copyrighttape$ as well
as Rules~(4) to~(7) in Figure~\ref{figure:rules-exp-hardness} in which
they are used.

To compensate for this, we add additional disjuncts to the query. In
fact, we use one disjunct for any two distinct $s,s' \in \Gamma$
and all $i \in \{1,2\}$:
$$
\begin{array}{l}
  \nextconfig_i(x,y) \wedge \mn{tape}_s(x,p) \wedge
  \mn{tape}_{s'}(y,p) \; \wedge \\[1mm]
  \bigwedge_{q \in Q} \nonconfig_q(x,p).
\end{array}
$$

\subsection{Proof of Theorem~\ref{theorem:tower-hardness}}

It remains to describe how the computation of the ATM
$\Mmc = (Q, \Sigma, \Gamma, \delta, q_0, g)$ is represented.

\smallskip

We introduce a unary predicate $\state_q$ for every $q \in Q$ and binary predicates
$\head$ and $\mn{tape}_s$ for every $s \in \Gamma$. Intuitively,
$\state_q(x)$ and $\head(x, x_h)$ say that $x$ is a configuration of \Mmc where
\Mmc is in state $q$ and the head is on position $x_h$ of the
tape. Likewise, $\mn{tape}_s(x,x_p)$ says that in configuration $x$,
the symbol $s$ is written on position $x_p$ of the tape. For all $s \in \Gamma$, we require that
$$
\begin{array}{rcl}
	\head(x,x_h) & \rightarrow & \cell_{\lol}(x_h) 
	\\
	\tape_s(x, x_p) & \rightarrow & \cell_{\lol}(x_p).
\end{array}
$$

\begin{figure*}
	\begin{align}
%
		\nonhead_{\triangleleft}(x,y),\ \nonstart(y) & \rightarrow \exists y' \, \nonhead_{\triangleleft}(x,y'), \nextcell_{\lol}(y',y), \nonhead(x,y')
		\label{rule:head-left}
		\\
		\nonhead_{\triangleright}(x,y),\ \nonend(y) & \rightarrow \exists y' \, \nonhead_{\triangleright}(x,y'), \nextcell_{\lol}(y,y'), \nonhead(x,y')
		\label{rule:head-right}
		\\
		\blanks_\triangleright(x, y), \nonend(y) & \rightarrow \exists y'\; \blanks_\triangleright(x, y'), \nextcell_{\lol}(y, y'), \tape_{\textvisiblespace}(x, y')
                                                           \label{rule:blanks-right} \\
                                                           		\zeros_\triangleright(x, y), \nonend(y), \cell_{k-1}(y) & \rightarrow \exists y'\; \zeros_\triangleright(x, y'), \nextcell_{k-1}(y, y'), \bit_{k, 0}(x, y')
		\label{rule:zeros-right}
		\\
		\copytape_\triangleleft(x, x', y), \nonstart(y), \cell_{k-1}(y) & \rightarrow \exists y'\; \copytape_\triangleleft(x, x', y'), \nextcell_{k-1}(y, y'), \copytape(x, x', y')
		\label{rule:copy-left}
	\end{align}
	\caption{Additional rules used in the proof of Theorem~\ref{theorem:tower-hardness} for every $k \in \{ 2, \dots, \lol\}$.}
	\label{figure:rules-tower-two}
\end{figure*}
 We next make sure that the state, head position, and tape content
are unique for every configuration. Regarding the state and tape content, for
all distinct $s,s' \in \Gamma$ and all distinct $q,q' \in
Q$ we put:
\begin{align*}
	\state_q(x),\state_{q'}(x) & \rightarrow \exists y\; \error_{\lol + 1}(y)
	\\
	\tape_s(x,x_p), \tape_{s'}(x,x_p) & \rightarrow \exists y\; \error_{\lol + 1}(y).
\end{align*}
%
To ensure uniqueness of the head position, we scan the tape in both directions, which is initialized by the rule
\begin{align*}
	\head(x,x_h) & \rightarrow \nonhead_{\triangleleft}(x,x_h),	\nonhead_{\triangleright}(x,x_h).
\end{align*}
Predicates $\nonhead_{\triangleleft}$ and $\nonhead_{\triangleright}$
follow the same propagation mechanism as $\zeros_\triangleleft$, $\ones_\triangleleft$ and $\copytape_\triangleright$, see Rules~\ref{rule:head-left} and \ref{rule:head-right} in Figure~\ref{figure:rules-tower-two}.
To prevent multiple head positions, we add the rule
\begin{align*}
	\head(x,x_p),\nonhead(x,x_p) & \rightarrow \exists y\; \error_{\lol + 1}(y).
\end{align*}
%
%
%
%

The initial configuration is represented by the constant $a$ and encoded with the following facts in $D$:
\begin{align*}
	\state_{q_0}(a)
	\\
	\head(a,c^{\lol}_0) 
	\\
	\tape_{e_i}(a, c^{\lol}_{i-1}) & \textrm{ for } 1 \leq i \leq n
	\\
	\tape_{\textvisiblespace}(a, c^{\lol}_n), \blanks_\triangleright(a, c^{\lol}_n)
\end{align*}
The predicate $\blanks_\triangleright(a, c^{\lol}_n)$ initiates
another propagation, see Rule~\ref{rule:blanks-right} in
Figure~\ref{figure:rules-tower-two}.  For the above to work, we want each
constant $c^{\lol}_i$ to represent the $i^\text{th}$ element of the
$\lol^\text{th}$ order, for $1 \leq i \leq n$. To achieve this, we
actually need constants $c^{k}_i$ for \emph{all}
$k \in \{ 1,\dots,\lol\}$ and all $i \in \{1,\dots,n\}$. This is
because we need to properly setup the bit values for the constants
$c^{\lol}_n$ for which we need the first elements of the
$(\lol-1)^\text{st}$, and so on, recursively.  For $i \in \{1,\dots,n\}$,
we use $b^{(i)}_0\dots b^{(i)}_n$ to denote the binary encoding of $i$
with $n+1$ bits. Notice that $b^{(i)}_n$ is always $0$ since we use $n+1$
bits to encode an integer that is bounded by
$n$. 
For every $k \in \{ 2, \dots, \lol \}$, every
$i, j \in \{ 0, \dots, n \}$, we now add facts:
\begin{align*}
	\bit_{k, b^{(i)}_j}(c^k_i, c^{k-1}_j) 
	\\
	\zeros_\triangleright(c^k_i, c^{k-1}_{n}).
\end{align*}
Also here, a predicate $\zeros_\triangleright$ is used to complete the encoding on long tapes, via Rule~\ref{rule:zeros-right} in Figure~\ref{figure:rules-tower-two}.

To generate the other configurations of the ATM computation on input
$e$, we use existential quantifiers in rules.
For all $q \in Q$ with
$g(q) \notin \{ \mn{accept}, \mn{reject} \}$, all $s \in \Gamma$, 
if $\delta(q,s)=(q_1,s_1,m_1,q_2,s_2,m_2)$ then
put for all $i \in \{ 1, 2\}$:
\[
\begin{array}{l@{}r}
	\state_q(x), \head(x,x_h), \tape_s(x,x_h) \ \rightarrow \hspace*{-3cm}
	\smallskip
	\\ & \exists y\, \exists y_h \, \nextconfig_{i}(x,y,x_h), \state_{q_i}(y)
	\smallskip
	\\ & \config(y,y_h), \mn{tape}_{s_i}(y,x_h), \mn{MOV}.
\end{array}
\]
%
where $\mn{MOV}$ is $\nextcell_{\lol} (x_h,y_h)$ if $m_i= \triangleright$ and
$\nextcell_{\lol}(y_h,x_h)$ if $m_i= \triangleleft$. The above rule sets up
correctly the state and head position of the new configuration, as
well as
the symbol on the tape at the position where the head was located
previously.
We still need to ensure that the other symbols didn't change. 
To this end, we reuse the predicate $\copytape$ used to copy bit values when moving along each $\nextcell_k$.
First, we add the following rule:
\begin{align*}
\nextconfig_{i}(x,y,x_h) & \rightarrow 
\copylefttape(x,y,x_h), \copyrighttape(x,y,x_h),
\end{align*}
along with copies of Rule~\ref{rule:copy-left} from Figure~\ref{figure:rules-tower-two}
and Rule~\ref{rule:copy-right} from Figure~\ref{figure:rules-tower} for $k = \lol$.
To guarantee that the symbols are copied, we add, for every $s \in \Gamma$, the rule
\begin{align*}
  \copytape(x,y,x_p), \tape_s(x,x_p) \rightarrow \tape_s(y,x_p).
\end{align*}


Note that the constant $a$ is the root of a tree of configurations
that are connected by the (projection to the first two components
of the) predicates $\nextconfig_{i}$, 
$i \in \{1,2\}$. We now
propagate acceptance information upwards in that tree. For all
states $q \in Q$ with $g(q)=\mn{accept}$, add
$$
\state_q(x) \rightarrow \acc(x).
$$
Further add, for all $q \in Q$ and $i \in \{1,2\}$,
$$
\nextconfig_{i}(x,y,x_h), \acc(y) \rightarrow \mn{acc}_i(x).
$$
Now, for all $q \in Q$ with  $g(q)=\mn{univ}$, all $q' \in Q$ with
$g(q')=\mn{exist}$, and all $i \in \{1,2\}$, add
\begin{align*}
	\state_q(x), \acc_1(x), \acc_2(x) & \rightarrow \acc(x) \\
	\state_{q'}(x), \acc_i(x) & \rightarrow \acc(x).
\end{align*}
At the root of the configuration tree, which is $a$, acceptance
makes the query true:
\begin{align*}
	\acc(x) & \rightarrow \target(x).
\end{align*}
We denote $\Omc$ the obtained set of rules and $D$ the obtained database.
We can now prove the following.
\begin{lemma}
	\label{lemma-main-claim-tower-hardness}
	$\Mmc$ accepts on input $e$ iff $\Omc, D \models_\CP \target(a)$. 
\end{lemma}


First, notice every $\CP$-minimal model $\structure{A}$ of $(\Omc, D)$ satisfies $\catch^\structure{A} = \{ a^\structure{A} \}$.
It is clear, from $D$, that $a^\structure{A} \in \catch^\structure{A}$.
Now assume by contradiction that $\{ a^\structure{A} \} \subsetneq \catch^\structure{A}$.
Then we easily construct a model $\structure{B} <_\CP \structure{A}$ by setting $\domain{B} = \domain{A}$ and interpreting every constant by $a^\structure{A}$ and every predicate of arity $r \geq 1$ as $\{ a^\structure{A} \}^r$.

From there and the rules stating each $\error_k$ is subsumed by $\catch$ and $\target$, it follows that in every $\CP$-minimal model $\structure{A}$ of $(\Omc, D)$ and for every $k \in \{ 1, \dots, \lol + 1 \}$, if $\error_k^\structure{A} \neq \emptyset$, then $\error_k^\structure{A} = \{ a^\structure{A} \}$ and $\structure{A} \models \target(a)$.

For the $(\Rightarrow)$ direction of Lemma~\ref{lemma-main-claim-tower-hardness}, it thus remains to treat the case of $\CP$-minimal models $\structure{A}$ of $(\Omc, D)$ such that $\error_k^\structure{A} = \emptyset$ for every $k \in \{ 1, \dots, \lol + 1 \}$.
We prove that in such a model $\structure{A}$, we have $\sizeof{\cell_1^\structure{A}} = p(n)$ and for every $k \in \{2, \dots, \lol \}$: $\sizeof{\cell_k^\structure{A}} = 2^{\sizeof{\cell_{k-1}^\structure{A}}}$.

For $\cell_1$, recall that $D$ already specifies $p(n)$ instances of this predicate, namely the constants $c^1_0, \dots, c^1_{p(n)-1}$. 
Furthermore, with the predicate $\diff$, all these constants must be interpreted as distinct elements, as otherwise $\error_1 \neq \emptyset$.
Thus $\sizeof{\cell_1^\structure{A}} \geq p(n)$.
Now, assume by contradiction that there exists another instance of $\cell_1$ somewhere in $\domain{A}$.
We can then construct a model $\structure{B} <_\CP \structure{A}$, yielding a contradiction, by interpreting $\cell_1$ as specified in $D$ and by collapsing every further constants and predicates on $a^\structure{A}$.
Notice that $a^\structure{A}$ cannot be among interpretations of constants $c^1_0, \dots, c^1_{p(n)-1}$ as it would have triggered $\error_1$.
Note that the obtained model $\structure{B}$ satisfies $\error_k^\structure{B} = \{ a^\structure{A} \} \supsetneq \error_k^\structure{A}$ for $k \in \{ 2, \dots, \lol + 1\}$ but that $\structure{B} <_\CP \structure{A}$ still holds due to the preference order on minimized predicates, prioritizing $\cell_1$ over those further predicates.

We then proceed by induction on $k$.
The argument is sensibly the same: due to $\error_k^\structure{A} = \emptyset$ and the $(k-1)^\textrm{th}$ tape having the intended size, it is readily verified that each constant $c^k_0$ starts a tape, following predicate $\nextcell_k$, of length at least $2^{\sizeof{\cell_{k-1}^\structure{A}}}$.
The existence of an extra instance beyond those $2^{\sizeof{\cell_{k-1}^\structure{A}}}$ is denied in the same manner, that is by constructing a $\structure{B} <_\CP \structure{A}$ collapsing all predicates concerning further tapes on $a^\structure{A}$ and crucially relying on the preference order.

We thus proved that all $\CP$-minimal models of $(\Omc, D)$ either have a non-empty $\error_k$ for some $k \in \{ 1, \dots, \lol + 1 \}$, in which case they satisfy the query, or produce a $\lol^\textrm{th}$ tape with exactly the correct size for our ATM to run on.
In this latter case, $\error_{\lol +1 }$ being empty guarantees that the model $\structure{A}$ encodes a valid computation of $\Mmc$ on input $e$.
Therefore, if $\Mmc$ accepts on input $e$, then the $\acc$ predicate is carried back to the root and yields $\structure{A} \models \target(a)$, which concludes the $(\Rightarrow)$ direction of Lemma~\ref{lemma-main-claim-tower-hardness}.
For the $(\Leftarrow)$ direction, we assume $\Mmc$ does not accept on input $e$, and we construct $\structure{A}$ representing the computation of $\Mmc$ on $e$ exactly as intended.
$\CP$-minimality of $\structure{A}$ follows from the intermediate claims we already proved.



\end{document}